\documentclass{article}

\usepackage[pdftex]{graphicx, color}
\usepackage{amsthm,amsmath,amsfonts,amssymb}
\usepackage{bm}
\usepackage[title]{appendix}
\usepackage[colorlinks=true, linkcolor=blue, filecolor=blue, urlcolor=blue, citecolor=red]{hyperref}
\usepackage{authblk}
\usepackage{indentfirst}
\usepackage{geometry}
\newtheorem{theorem}{Theorem}[section]
\newtheorem{lemma}{Lemma}[section]
\allowdisplaybreaks[2]
\makeatletter
\long\def\@makecaption#1#2{%
  \normalsize%% add this line
  \vskip\abovecaptionskip
  \sbox\@tempboxa{#1: #2}%
  \ifdim \wd\@tempboxa >\hsize
    #1: #2\par
  \else
    \global \@minipagefalse
    \hb@xt@\hsize{\hfil\box\@tempboxa\hfil}%
  \fi
  \vskip\belowcaptionskip}
\makeatother

\title{Goodness-of-fit Test for Latent Block Models}
\author[1]{Chihiro Watanabe\thanks{watanabe-chihiro763@g.ecc.u-tokyo.ac.jp}} %%%
\author[1,2]{Taiji Suzuki\thanks{taiji@mist.i.u-tokyo.ac.jp}}
\affil[1]{{\normalsize Graduate School of Information Science Technology, The University of Tokyo, Tokyo, Japan}}
\affil[2]{{\normalsize Center for Advanced Intelligence Project (AIP), RIKEN, Tokyo, Japan}}
\date{}

\begin{document}
\maketitle

\begin{abstract}%   <- trailing '%' for backward compatibility of .sty file
Latent block models are used for probabilistic biclustering, which is shown to be an effective method for analyzing various relational data sets. However, there has been no statistical test method for determining the row and column cluster numbers of latent block models. Recent studies have constructed statistical-test-based methods for stochastic block models, which assume that the observed matrix is a square symmetric matrix and that the cluster assignments are the same for rows and columns. In this study, we developed a new goodness-of-fit test for latent block models to test whether an observed data matrix fits a given set of row and column cluster numbers, or it consists of more clusters in at least one direction of the row and the column. To construct the test method, we used a result from the random matrix theory for a sample covariance matrix. We experimentally demonstrated the effectiveness of the proposed method by showing the asymptotic behavior of the test statistic and measuring the test accuracy. 
\end{abstract}

\section{Introduction}

Block modeling \cite{Hartigan1972, Arabie1978} is known to be effective in representing various relational data sets, such as the data sets of movie ratings \cite{Shan2008}, customer-product transactions \cite{Shan2008}, congressional voting \cite{Keribin2015}, document-word relationships \cite{Dhillon2001}, and gene expressions \cite{Pontes2015}. 
Latent block models or LBMs \cite{Govaert2003} are used for probabilistic biclustering of such relational data matrices, where rows and columns represent different objects. 
For instance, suppose that a matrix $A = (A_{ij})_{1\leq i \leq n, 1 \leq j \leq p} \in \mathbb{R}^{n \times p}$ represents the relationship between users and movies, where entry $A_{ij}$ is the rating of the $j$th movie by the $i$th user. 
In LBMs, we assume a regular-grid block structure behind the observed matrix $A$; i.e., both rows (users) and columns (movies) of matrix $A$ are simultaneously decomposed into latent clusters. A block is defined as a combination of row and column clusters, and entries of the same block in matrix $A$ are supposed to be i.i.d. random variables. 

An open problem in using LBMs is that there has been no statistical criterion for determining the numbers of row and column clusters. Recently, statistical-test-based approaches \cite{Bickel2016, Lei2016, Hu2020} have been proposed for estimating the cluster number of stochastic block models (SBMs) \cite{Holland1983}. 
SBMs are similar to LBMs in the sense that they assume a block structure behind an observed matrix; however, they are based on different assumptions from LBMs that an observed matrix is a square symmetric matrix and that the cluster assignments are the same for rows and columns \cite{Mariadassou2015}. 
In regard to the LBM setting, no statistical method has been constructed to determine row and column cluster numbers. 

Aside from the test-based methods, several model selection approaches have been proposed based on cross-validation \cite{Chen2018} or an information criterion \cite{Keribin2012, Peixoto2013, Keribin2015}. However, these approaches have several limitations. (1) First, they cannot provide knowledge about the reliability of the result besides the finally estimated cluster numbers. Rather than minimizing the generalization error, in some cases, it is more appropriate to provide a probabilistic guarantee in reliability for the purpose of knowledge discovery. 
(2) Second, both the cross-validation-based and information-criterion-based methods depend on the clustering algorithm used. For instance, we can employ the Bayesian information criterion (BIC) for estimating the marginal likelihood only if the Fisher information matrix of the model is regular, which is not the case for block models. 
Constructing an information criterion that estimates the expectation of the generalization error for a wider class of models is generally difficult. 
(3) Finally, the above methods require relatively large computational complexity. Computation of an information criterion requires the process of approximating the posterior distribution by the Markov chain Monte Carlo (MCMC) method, and cross-validation requires the iterative calculation of the test error with different sets of partitions of the training and test data sets. 

In this study, we proposed a new statistical test method for LBMs. To construct a hypothesis test with a theoretical guarantee, we used a result from random matrix theory. Recent studies on random matrix theory have revealed the asymptotic behavior of singular values of an $n \times p$ random matrix \cite{Geman1980, Silverstein1985, Yin1988, Bai1993, Johansson2000, Johnstone2001, Soshnikov2002, Peche2009, Pillai2014, Bao2015, Ding2018}. Here, we assume that each entry $Z_{ij}$ of matrix $Z$, which is given by $Z_{ij} = (A_{ij} - P_{ij})/\sigma_{ij}$ (which is computed by the original matrix $A$, its block-wise mean $P$ and standard deviation $\sigma$) follows a distribution with a sub-exponential decay. From the result in \cite{Pillai2014}, the normalized maximum eigenvalue of $Z^{\top} Z$ converges in law to the Tracy-Widom distribution with index $1$, under the above sub-exponential condition. 
Based on this result, we constructed a goodness-of-fit test for a given set of row and column cluster numbers of an LBM, using the maximum singular value of matrix $\hat{Z}$, which is an estimator of the matrix $Z$. 
We proved that under the null hypothesis (i.e., observed matrix $A$ consists of a given set of row and column cluster numbers), the proposed test statistic $T$ converges in law to the Tracy-Widom distribution with index $1$ (Theorem \ref{th:realizable}). We also showed that under the alternative hypothesis, test statistic $T$ increases in proportion to $m^{\frac{5}{3}}$ with a high probability, where $m$ is a number proportional to the matrix size (Theorems \ref{th:unrealizable_lower} and \ref{th:unrealizable_upper}). 

The proposed method solves the limitations of other model selection approaches. (1) Our statistical test method enables us to obtain knowledge about the reliability of the test results. When testing a given set of row and column cluster numbers, we can explicitly set the probability of Type I error (or false positive) as a significance level $\alpha$. 
(2) Unlike the other model selection methods, the proposed method does not depend on the clustering algorithm as long as it satisfies the consistency condition (Section \ref{sec:method}). It only uses the output of a clustering algorithm to test a given set of cluster numbers; there is no need to modify the test method according to the clustering algorithm. 
(3) The proposed test method requires relatively small computational complexity. It does not require the MCMC procedure or partitioning into the training and test data sets. 
For these reasons, the proposed test-based method can be widely used for the purpose of knowledge discovery. 

The next sections consist of the detailed explanation of the proposed test method for LBMs. In Section \ref{sec:method}, we describe the proposed goodness-of-fit test and its theoretical guarantee with the assumptions required for the problem setting. 
Next, we briefly review the related works and their differences from the proposed method in Section \ref{sec:related}. 
The main results are presented in Section \ref{sec:statistic}, where we prove the asymptotic properties of the proposed test statistic. 
In Section \ref{sec:experiments}, we experimentally demonstrate the effectiveness of the proposed test method by showing the asymptotic behavior of the test statistic and calculating the test accuracy. We discuss the results and limitations of the proposed method in Section \ref{sec:discussion} and conclude the paper in Section \ref{sec:conclusion}. 

%%%%%%%%%%%%%%%%%%%%%%%%%%%%%%%%%%%%%%%%%%%%%%%%%%%%%%%

\section{Problem setting and statistical model for goodness-of-fit test for latent block models}
\label{sec:method}

Let $A \in \mathbb{R}^{n \times p}$ be an $n \times p$ observed matrix. We assume that each entry of matrix $A$ is independently generated, given its row and column clusters. Let $(K, H)$ be the null set of cluster numbers for rows and columns of an observed matrix $A$, which is unknown in advance. We denote the cluster indices of the $i$th row and the $j$th column of matrix $A$ as $g^{(1)}_i \in \{1, \dots, K\}$ and $g^{(2)}_j \in \{1, \dots, H\}$, respectively. We assume that each entry of matrix $A$ is independently subject to a distribution with block-wise mean $P$ and block-wise standard deviation $\sigma$: 
\begin{align}
\label{eq:LBM}
&P = (P_{ij})_{1\leq i \leq n, 1 \leq j \leq p}, \ \ \ \ \ 
P_{ij} = B_{g^{(1)}_i g^{(2)}_j}.  \nonumber \\
&\sigma = (\sigma_{ij})_{1\leq i \leq n, 1 \leq j \leq p}, \ \ \ \ \ 
\sigma_{ij} = S_{g^{(1)}_i g^{(2)}_j}.  \nonumber \\
&A = (A_{ij})_{1\leq i \leq n, 1 \leq j \leq p}, \ \ \ \ \ 
\mathbb{E} [A_{ij}] = P_{ij}, \ \ \ \ \ 
\mathbb{E} [(A_{ij} - P_{ij})^2] = \sigma_{ij}^2, 
\end{align}
where $B_{k h}$ and $S_{k h} > 0$, respectively, are the mean and the positive standard deviation of entries in the $(k, h)$th null block under the null hypothesis. 

In this paper, we propose a goodness-of-fit test for selecting the cluster numbers $(K, H)$ from observed matrix $A$. In such a test, we test whether $(K, H)$ is equal to a given set of cluster numbers $(K_0, H_0)$ or at least one of the given row and column cluster numbers $K_0$ or $H_0$ is smaller than the null cluster numbers $K$ or $H$. In other words, the null (N) and alternative (A) hypotheses are given by 
\begin{eqnarray}
\label{eq:hypothesis}
\mathrm{(N) :}\ (K, H) = (K_0, H_0), \ \ \ \ \ \ \ \ \ \ 
\mathrm{(A) :}\ K > K_0\ \mathrm{or}\ H > H_0. 
\end{eqnarray}
By sequentially testing the cluster numbers in the following order (Figure \ref{fig:order_test}), we can select the cluster numbers of a given observed matrix $A$. 
\begin{enumerate}
\item Test $(K_0, H_0) = (1, 1)$. 
\item Test $(K_0, H_0) = (1, 2), (2, 1)$. 
\item Test $(K_0, H_0) = (1, 3), (2, 2), (3, 1)$. 
\item $\cdots$
\item Test $(K_0, H_0) = (1, L), (2, L-1), \dots, (L, 1)$. 
Let $(\hat{K}, \hat{H})$ be the row and column cluster numbers where the null hypothesis is accepted and $\hat{K} + \hat{H} = L+1$ holds. The selected set of cluster numbers is $(\hat{K}, \hat{H})$. 
\end{enumerate}
Based on the above sequentially ordered test, selection of the cluster numbers requires $(K+H)(K+H-1)/2$ tests at most. 
\begin{figure}[t]
  \centering
  \includegraphics[width=0.5\hsize]{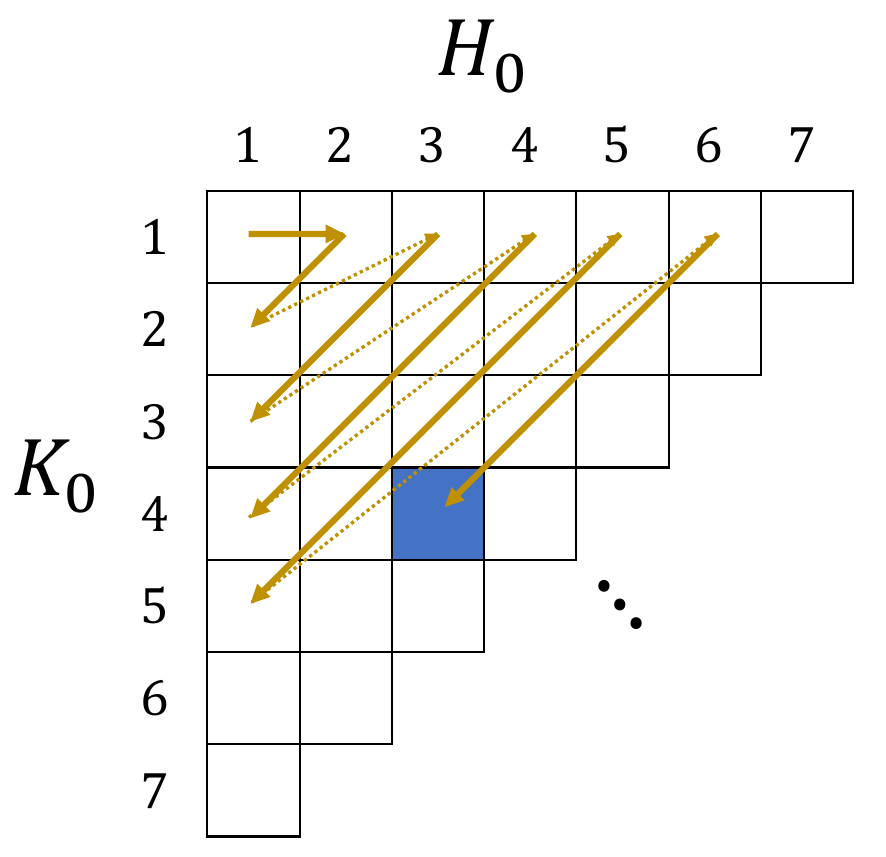}
  \caption{The sequential order for testing row and column cluster numbers. For example, let the blue entry $(4, 3)$ be the null cluster numbers $(K, H)$. Based on this sequentially ordered test, the given cluster numbers $(K_0, H_0)$ are always unrealizable (that is, at least one of $K>K_0$ or $H>H_0$ holds), until it reaches to $(K, H)$. } 
  \label{fig:order_test}
\end{figure}

\paragraph{Assumptions.} 
Throughout this paper, we make the following assumptions to derive the test statistics: 
\begin{enumerate}
\renewcommand{\theenumi}{(\roman{enumi})}
\item We assume that a distribution of $Z_{ij}$, which is given by $Z_{ij} = (A_{ij} - P_{ij})/\sigma_{ij}$ as in (\ref{eq:Z_true}) later, has a sub-exponential decay. That is, there exists some $\vartheta >0$ such that for $x>1$, $\mathrm{Pr} \left( \left| Z_{ij} \right| > x \right) \leq \vartheta^{-1}\exp (-x^{\vartheta})$. From this assumption, note that for any $n^{\mathrm{M}} \in \mathbb{N}$, the $n^{\mathrm{M}}$th moment of a random variable $Z_{ij}$ is finite (i.e., $\mathbb{E} [Z_{ij}^{n^{\mathrm{M}}}] < \infty$). 
\item We denote the number of rows and columns of matrix $A$ as $n$ and $p$, respectively. We assume that both $n$ and $p$ increase in proportion to some sufficiently large number $m$ (i.e., $n, p \propto m$). 
\item Let $K$ and $H$, respectively, be the minimum row and column cluster numbers to represent the block structure of observed matrix $A$ under the null hypothesis. We assume that both $K$ and $H$ are finite constants that do not increase with the matrix sizes $n$ and $p$. 
We also assume that the minimum row and column sizes of a block in the null block structure, which we denote as $n_{\mathrm{min}}$ and $p_{\mathrm{min}}$, respectively, satisfy $n_{\mathrm{min}} = \Omega_p (m)$ and $p_{\mathrm{min}} = \Omega_p (m)$, where we used the following notation: 
\begin{align}
&X = \Omega_p \left( f(m) \right). \nonumber \\
&\Leftrightarrow \forall \epsilon>0, \ \exists C>0, M>0, \ \forall m \geq M, \ \mathrm{Pr} \left(C f(m) \leq X \right) \geq 1-\epsilon. 
\end{align}
In other words, we assume that with high probability, there is no ``too small'' block in matrix $A$. 
\label{asmp:block_size}
\item If the given set of cluster numbers $(K_0, H_0)$ is equal to the null cluster numbers $(K, H)$, then we call it a \textit{realizable} case. Otherwise, we call it an \textit{unrealizable} case ($K>K_0$ or $H>H_0$). 
% [Revision 2020/9] <---
In Section \ref{sec:statistic}, we see that Theorems \ref{th:unrealizable_lower} and \ref{th:unrealizable_upper} guarantee the behavior of the test statistic $T$ in unrealizable cases. 
For now, there is no way to detect the cases where $(K < K_0) \cap (H \leq H_0)$ or $(K \leq K_0) \cap (H < H_0)$ holds, and to cope with such settings is beyond the scope of this paper. 
% [Revision 2020/9] --->
\item In the realizable case, we assume that a clustering algorithm is \textit{consistent}, that is, the probability that it outputs the correct block structure converges to $1$, in the limit of $m \to \infty$. By using this assumption, the proposed method does not depend on a specific clustering algorithm. Several clustering algorithms including \cite{Flynn2020, Ames2014, Brault2016} have been proven to be consistent. 
\label{asmp:consistency}
\end{enumerate}

%%%%%%%%%%%%%%%%%%%%%%%%%%%%%%%%%%%%%%%%%%%%%%%%%%%%%%%

\section{Relation to existing works}
\label{sec:related}

In this section, we briefly review the related works and explain the differences between them and the proposed method. 

\subsection{Model selection for block models}

%==============================================================================
\paragraph{Statistical-test-based methods (for SBM)} 
Recently, several methods have been proposed for testing the properties of a given observed matrix in relation to SBMs \cite{Bickel2016, Lei2016, Karwa2016, Hu2020, Yuan2018}. 
Particularly, the methods proposed in \cite{Bickel2016, Lei2016, Hu2020} have enabled us to estimate the number of blocks for SBMs. However, these methods differ from ours in the problem setting; they can be applied only to an SBM setting, where an observed matrix is a square symmetric matrix, and the cluster assignments are the same for rows and columns. There has been no method for estimating the block number for LBMs, where rows and columns (not necessarily square) of an observed matrix are simultaneously decomposed into clusters. 

%==============================================================================
\paragraph{Cross-validation-based methods} 
Cross-validation is a widely used method for model selection, where a data set is first split into training and test data sets, and then the best model with the minimum test error is determined. Recently, cross-validation methods for matrix data have been proposed \cite{Dabbs2016, Li2020, Kawamoto2017, Chen2018} to determine the number of clusters in network data. 
Although the purpose of these methods and our method is similar, these methods differ from ours in that their target is the network data, where the observed matrix is square and its rows and columns represent the same node sets. Thus, the block structure is symmetric regardless of whether the network itself is directed or undirected). 
Moreover, unlike a statistical test, these methods cannot provide quantitative knowledge about the reliability of the selected model. Furthermore, the computational cost of cross-validation is generally high because it requires the iterative calculation of the test error with different data set partitions. 

%==============================================================================
\paragraph{Information-criterion-based methods} 
Another approach for determining the number of blocks in a matrix is to estimate the generalization error or marginal likelihood by some information criterion for given sets of block numbers. By using such information criteria, we can select a model in a statistically meaningful (non-heuristic) way. In regard to block models, many variants of BIC have been proposed \cite{Keribin2012, Peixoto2013, Keribin2015, Hu2019, Saldana2017}. 
Unlike our test-based method, which only requires a clustering algorithm to satisfy the consistency condition (Section \ref{sec:method}), an information criterion for a theoretical guarantee should be carefully chosen according to the given clustering algorithm. 
For instance, BIC can be employed for estimating the marginal likelihood only if the Fisher information matrix of the model is regular, which is not the case for block models. 

To solve this problem, as an alternative criterion to BIC, the integrated completed likelihood (ICL) criterion has been used in many studies for estimating the number of blocks in LBMs \cite{Lomet2012, Wyse2017, Corneli2015}. In ICL, we first derive a marginal likelihood for a given set of an observed matrix and block assignments and then substitute the set of estimated block assignments to approximate the marginal likelihood. 
However, since ICL is computed based on a single estimator of block assignments, there is no guarantee for the goodness of the approximation of marginal likelihood. 

Similar to cross-validation-based methods, information-criterion-based methods cannot provide a probabilistic guarantee for the reliability of the selected model, which is a disadvantage for the purpose of knowledge discovery. The computational cost also becomes a problem because the computation of an information criterion requires the process of approximating the posterior distribution by MCMC. 

%==============================================================================
\paragraph{Other model selection methods}
Aside from the information criteria, several studies have proposed to determine the number of blocks in LBMs based on the co-clustering adjusted rand index \cite{Robert2017}, the extended modularity for biclustering \cite{Labiod2011}, or the expected posterior loss for a given loss function \cite{Rastelli2018}. 
Another approach is to define the posterior distribution not only on cluster assignments of rows and columns but also on row and column cluster numbers \cite{Wyse2012, Passino2020}. 
Unlike the model selection approaches, such nonparametric Bayesian methods can estimate the distribution of the block numbers. The best-fitted number of the blocks can be determined based on the posterior distribution (e.g., we can choose a MAP estimator \cite{Passino2020}). However, in this case, the computational cost of MCMC is higher than that of the information-criterion-based methods because it requires a large number of iterations to approximate the posterior distribution both on the block assignments and the number of blocks. 

%%%%%%%%%%%%%%%%%%%%%%%%%%%%%%%%%%%%%%%%%%%%%%%%%%%%%%%

\section{Test statistic for determining the set of cluster numbers}
\label{sec:statistic}

To derive the test statistic for the proposed goodness-of-fit test, we first normalize each entry $A_{ij}$ of an observed matrix $A$ by subtracting $P_{ij}$ and dividing it by $\sigma_{ij}$, where $P$ and $\sigma$, respectively, are the block-wise mean and standard deviation in (\ref{eq:LBM}): 
\begin{eqnarray}
\label{eq:Z_true}
Z = (Z_{ij})_{1\leq i \leq n, 1 \leq j \leq p}, \ \ \ \ \ \ \ \ \ \ 
Z_{ij} = \frac{A_{ij} - P_{ij}}{\sigma_{ij}}. 
\end{eqnarray}
By definition, each entry $Z_{ij}$ of matrix $Z$ in (\ref{eq:Z_true}) independently follows a distribution with zero mean and standard deviation of one. 
Therefore, according to the result in \cite{Pillai2014}, if $n = n(p)$ and $n/p \to \gamma \neq 0, \infty$ in the limit of $p \to \infty$, the scaled maximum eigenvalue $T^*$ of matrix $Z^\top Z$ converges in law to the Tracy-Widom distribution with index $1$ ($TW_1$) in the limit of $p \to \infty$: 
\begin{eqnarray}
\label{eq:T_true}
T^* = \frac{\lambda_1 - a}{b}, \ \ \ \ \ \ \ \ \ \ 
T^* \rightsquigarrow TW_1 \ \mathrm{(Convergence\ in\ law)}, 
\end{eqnarray}
where $\lambda_1$ is the maximum eigenvalue of matrix $Z^\top Z$ and 
\begin{eqnarray}
\label{eq:ab}
a = (\sqrt{n} + \sqrt{p})^2, \ \ \ \ \ \ \ \ \ \ 
b = (\sqrt{n} + \sqrt{p}) \left( \frac{1}{\sqrt{n}} + \frac{1}{\sqrt{p}} \right)^{\frac{1}{3}}. 
\end{eqnarray}

In most cases, the null cluster numbers $(K, H)$ and the null cluster assignments $g^{(1)}$ and $g^{(2)}$ are unknown in advance. Therefore, we can only estimate the block structure based on the observed matrix $A$ and the given cluster numbers. Let $(K_0, H_0)$ be the given set of row and column cluster numbers, and $\hat{g}^{(1)}$ and $\hat{g}^{(2)}$, respectively, be the estimated cluster assignments for rows and columns. Based on such an estimated block structure $(\hat{g}^{(1)}, \hat{g}^{(2)})$, we estimate the block-wise mean and standard deviation by
\begin{align}
\label{eq:BPS_hat}
&\hat{B} = (\hat{B}_{k h})_{1 \leq k \leq K_0, 1 \leq h \leq H_0}, \ \ \ \ \ \ \ \ \ \ 
\hat{B}_{k h} = \frac{1}{|I_{k}| |J_{h}|} \sum_{i \in I_{k}, j \in J_{h}} A_{ij}, \nonumber \\
&\hat{P} = (\hat{P}_{ij})_{1\leq i \leq n, 1 \leq j \leq p}, \ \ \ \ \ \ \ \ \ \ 
\hat{P}_{ij} = \hat{B}_{\hat{g}^{(1)}_i \hat{g}^{(2)}_j}, \nonumber \\
&\hat{S} = (\hat{S}_{k h})_{1 \leq k \leq K_0, 1 \leq h \leq H_0}, \ \ \ \ \ \ \ \ \ \ 
\hat{S}_{k h} = \sqrt{ \frac{1}{|I_{k}| |J_{h}|} \sum_{i \in I_{k}, j \in J_{h}} \left( A_{ij} - \hat{P}_{ij} \right)^2}, \nonumber \\
&\hat{\sigma} = (\hat{\sigma}_{ij})_{1\leq i \leq n, 1 \leq j \leq p}, \ \ \ \ \ \ \ \ \ \ 
\hat{\sigma}_{ij} = \hat{S}_{\hat{g}^{(1)}_i \hat{g}^{(2)}_j}, 
\end{align}
where $I_k$ is the set of row indices of matrix $A$ that are assigned to the $k$th cluster, and $J_h$ is the set of column indices of matrix $A$ that are assigned to the $h$th cluster: 
\begin{eqnarray}
\label{eq:I_k_J_h}
I_k = \left\{ i: \hat{g}^{(1)}_i = k \right\}, \ \ \ \ \ \ \ \ \ \ 
J_h = \left\{ j: \hat{g}^{(2)}_j = h \right\}. 
\end{eqnarray}
The consistency assumption \ref{asmp:consistency} guarantees that if $(K_0, H_0) = (K, H)$, the probability that the cluster assignments $(I_k)_{1 \leq k \leq K_0}$ and $(J_h)_{1 \leq h \leq H_0}$ are correct converges to $1$ in the limit of $m \to \infty$. 

We define an estimator of normalized matrix $Z$ in (\ref{eq:Z_true}) based on the estimated block-wise mean $\hat{P}$ and standard deviation $\hat{\sigma}$ in (\ref{eq:BPS_hat}): 
\begin{eqnarray}
\label{eq:Z_hat}
\hat{Z} = (\hat{Z}_{ij})_{1\leq i \leq n, 1 \leq j \leq p}, \ \ \ \ \ \ \ \ \ \ 
\hat{Z}_{ij} = \frac{A_{ij} - \hat{P}_{ij}}{\hat{\sigma}_{ij}}. 
\end{eqnarray}

The test statistic $T$ for the proposed goodness-of-fit test is given by the scaled maximum eigenvalue of matrix $\hat{Z}^\top \hat{Z}$: 
\begin{eqnarray}
\label{eq:T_statistic}
T = \frac{\hat{\lambda}_1 - a}{b}, 
\end{eqnarray}
where $\hat{\lambda}_1$ is the maximum eigenvalue of matrix $\hat{Z}^\top \hat{Z}$, and $a$ and $b$ are given by (\ref{eq:ab}). 

Based on the following results in Theorems \ref{th:realizable}, \ref{th:unrealizable_lower} and \ref{th:unrealizable_upper}, we propose a one-sided goodness-of-fit test for a given set of cluster numbers $(K_0, H_0)$ at the significance level of $\alpha$ by using the test statistic $T$: 
\begin{eqnarray}
\label{eq:rejection}
\mathrm{Reject\ null\ hypothesis}\ ((K, H) = (K_0, H_0)),\ \ \ \ \ \mathrm{if}\ T \geq t(\alpha), 
\end{eqnarray}
where $t(\alpha)$ is the $\alpha$ upper quantile of the Tracy-Widom distribution with index $1$. 
By applying the sequentially ordered test that we explained in Section \ref{sec:method} based on the above rejection rule (\ref{eq:rejection}), we can select a set of row and column cluster numbers $(\hat{K}, \hat{H})$ for a given observed matrix $A$. 

In the proof of Theorem \ref{th:realizable}, we also use the following notations. Let $\tilde{B}_{kh}$ and $\tilde{S}_{kh}$, respectively, be the \textbf{sample} mean and standard deviation of all the entries in the $(k, h)$th \textbf{null} block in matrix $A$. Based on such notations, we define the sample mean matrix $\tilde{P}$ and standard deviation matrix $\tilde{\sigma}$ for the correct block structure, and matrix $\tilde{Z}$ by: 
\begin{align}
\label{eq:tilde_Z}
&\tilde{P} = (\tilde{P}_{ij})_{1\leq i \leq n, 1 \leq j \leq p}, \ \ \ \ \ \ \ \ \ \ \tilde{P}_{ij} = \tilde{B}_{g^{(1)}_i g^{(2)}_j}, \nonumber \\
&\tilde{\sigma} = (\tilde{\sigma}_{ij})_{1\leq i \leq n, 1 \leq j \leq p}, \ \ \ \ \ \ \ \ \ \ \tilde{\sigma}_{ij} = \tilde{S}_{g^{(1)}_i g^{(2)}_j}, \nonumber \\
&\tilde{Z} = (\tilde{Z}_{ij})_{1\leq i \leq n, 1 \leq j \leq p}, \ \ \ \ \ \ \ \ \ \ \tilde{Z}_{ij} = \frac{A - \tilde{P}_{ij}}{\tilde{\sigma}_{ij}}. 
\end{align}

%====================================================================
\begin{theorem}[Realizable case]
We assume that the following condition holds: $n = n(p)$, $n/p \to \gamma  \neq 0, \infty$ in the limit of $p \to \infty$. Under the consistency assumption \ref{asmp:consistency} for the clustering algorithm, if $(K_0, H_0) = (K, H)$, 
\begin{eqnarray}
T \rightsquigarrow TW_1 \ \mathrm{(Convergence\ in\ law)},
\end{eqnarray}
in the limit of $p \to \infty$, where $T$ is defined as in (\ref{eq:T_statistic}). 
\label{th:realizable}
\end{theorem}
\begin{proof}
We denote the operator norm by $\| \cdot \|_{\mathrm{op}}$, 
\begin{eqnarray}
\| A \|_{\mathrm{op}} = \sup_{\bm{u} \in \mathbb{R}^p} \frac{\| A \bm{u} \|}{\| \bm{u} \|}. 
\end{eqnarray}

First of all, we derive the difference between $B_{kh}$ ($S_{kh}$) and $\tilde{B}_{kh}$ ($\tilde{S}_{kh}$), which have been defined in (\ref{eq:LBM}) and (\ref{eq:tilde_Z}). Since the number of entries in the block is proportional to $m^2$ by the assumption \ref{asmp:block_size}, $\sqrt{m^2} \left( B_{kh} - \tilde{B}_{kh} \right)$ converges to $\mathcal{N} (0, S_{kh}^2)$ from the central limit theorem. Therefore, from Prokhorov's theorem \cite{Vaart1998}, we have
\begin{eqnarray}
\label{eq:BBtildediff}
\tilde{B}_{kh} = B_{kh} + O_p \left( \frac{1}{m} \right). 
\end{eqnarray}
Also, the following equation holds (The proof  is given in Appendix \ref{sec:ap_sigma_tilde}):  
\begin{eqnarray}
\label{eq:SStildediff}
\tilde{S}_{kh} = S_{kh} + O_p \left( \frac{1}{m} \right). 
\end{eqnarray}

From here, we derive the difference between the maximum eigenvalue $\tilde{\lambda}_1$ of matrix $\tilde{Z}^{\top} \tilde{Z}$ and the maximum eigenvalue $\lambda_1$ of matrix $Z^{\top} Z$, where the definitions of matrices $Z$ and $\tilde{Z}$ have been given in (\ref{eq:Z_true}) and (\ref{eq:tilde_Z}), respectively. 
From (\ref{eq:T_true}), we have $\lambda_1 = O_p (m)$. Therefore, the largest singular value of matrix $Z$, which is equal to $\| Z \|_{\mathrm{op}}$, is in the order of $O_p (\sqrt{m})$. 

By the subadditivity of the operator norm, we have 
\begin{eqnarray}
\label{eq:ZZtildediff}
\left| \| Z \|_{\mathrm{op}} - \| \tilde{Z} \|_{\mathrm{op}} \right| \leq \| Z -\tilde{Z} \|_{\mathrm{op}}. 
\end{eqnarray}

Let $A^{(k, h)}$, $P^{(k, h)}$, $\tilde{P}^{(k, h)}$, $Z^{(k, h)}$ and $\tilde{Z}^{(k, h)}$, respectively, be the $(k, h)$th \textbf{null} blocks of matrices $A$, $P$, $\tilde{P}$, $Z$ and $\tilde{Z}$. We also denote the row and column sizes of the $(k, h)$th \textbf{null} block as $n_k$ and $p_h$, respectively. From the definitions in (\ref{eq:Z_true}) and (\ref{eq:tilde_Z}), we have 
\begin{eqnarray}
Z^{(k, h)} = \frac{A^{(k, h)} - P^{(k, h)}}{S_{kh}}, \ \ \ \ \ \ \ \ \ \ 
\tilde{Z}^{(k, h)} = \frac{A^{(k, h)} - \tilde{P}^{(k, h)}}{\tilde{S}_{kh}}. 
\end{eqnarray}
Combining this with (\ref{eq:BBtildediff}), (\ref{eq:SStildediff}) and the fact that the Frobenius norm upper bounds the operator norm, we have 
\begin{align}
\label{eq:ZZtilde_F}
&\| Z^{(k, h)} - \tilde{Z}^{(k, h)} \|_{\mathrm{op}} 
= \left\| \frac{A^{(k, h)} - P^{(k, h)}}{S_{kh}} - \frac{A^{(k, h)} - \tilde{P}^{(k, h)}}{\tilde{S}_{kh}} \right\|_{\mathrm{op}} \nonumber \\
&= \left\| \frac{A^{(k, h)} - P^{(k, h)}}{S_{kh}} - \frac{A^{(k, h)} - P^{(k, h)}}{\tilde{S}_{kh}} + \frac{A^{(k, h)} - P^{(k, h)}}{\tilde{S}_{kh}} - \frac{A^{(k, h)} - \tilde{P}^{(k, h)}}{\tilde{S}_{kh}} \right\|_{\mathrm{op}} \nonumber \\
&\leq \left\| \frac{A^{(k, h)} - P^{(k, h)}}{S_{kh}} - \frac{A^{(k, h)} - P^{(k, h)}}{\tilde{S}_{kh}} \right\|_{\mathrm{op}} + \left\| \frac{A^{(k, h)} - P^{(k, h)}}{\tilde{S}_{kh}} - \frac{A^{(k, h)} - \tilde{P}^{(k, h)}}{\tilde{S}_{kh}} \right\|_{\mathrm{op}} \nonumber \\
% [Revision 2020/9] <---
&= \left| \frac{\tilde{S}_{kh} - S_{kh}}{S_{kh} \tilde{S}_{kh}} \right| \| A^{(k, h)} - P^{(k, h)} \|_{\mathrm{op}} + \frac{1}{\tilde{S}_{kh}} \| P^{(k, h)} - \tilde{P}^{(k, h)} \|_{\mathrm{op}} \nonumber \\
&\leq \left| \frac{\tilde{S}_{kh} - S_{kh}}{S_{kh} \tilde{S}_{kh}} \right| \| A^{(k, h)} - P^{(k, h)} \|_{\mathrm{op}} + \frac{1}{\tilde{S}_{kh}} \| P^{(k, h)} - \tilde{P}^{(k, h)} \|_{\mathrm{F}} \nonumber \\
&= \left| \frac{\tilde{S}_{kh} - S_{kh}}{\tilde{S}_{kh}} \right| \| Z^{(k, h)} \|_{\mathrm{op}} + \frac{1}{\tilde{S}_{kh}} \sqrt{n_k p_h} \left| B_{kh} - \tilde{B}_{kh} \right| \nonumber \\
% [Revision 2020/9] --->
&= \frac{O_p (1/m)}{S_{kh} + O_p (1/m)} \| Z^{(k, h)} \|_{\mathrm{op}} + \frac{O_p (1/m)}{S_{kh} + O_p (1/m)} \sqrt{n_k p_h} \ \ \ \left(\because (\ref{eq:BBtildediff}), (\ref{eq:SStildediff}) \right) \nonumber \\
&= \frac{O_p (1/m)}{S_{kh} + O_p (1/m)} O_p (\sqrt{m}) + \frac{O_p (1/m)}{S_{kh} + O_p (1/m)} \sqrt{n_k p_h} \ \ \ \left(\because (\ref{eq:T_true}) \right) \nonumber \\
&= O_p \left( \frac{1}{\sqrt{m}} \right) + O_p (1) = O_p (1). 
\end{align}

Therefore, since the operator norm of a matrix is not larger than the sum of the operator norms of all of its blocks and the number of blocks are finite constants, we have
\begin{eqnarray}
% [Revision 2020/9] <---
\| Z - \tilde{Z} \|_{\mathrm{op}} &\leq& \sum_{k=1}^K \sum_{h=1}^H \| Z^{(k, h)} - \tilde{Z}^{(k, h)} \|_{\mathrm{op}} = O_p (1). 
% [Revision 2020/9] --->
\end{eqnarray}
By combining this with (\ref{eq:ZZtildediff}), we obtain 
\begin{eqnarray}
\label{eq:Ztilde_op}
% [Revision 2020/9] <---
\left| \| Z \|_{\mathrm{op}} - \| \tilde{Z} \|_{\mathrm{op}} \right| = O_p (1). 
% [Revision 2020/9] --->
\end{eqnarray}

Next, we consider the joint probability of the event $\mathcal{F}_m$ that the clustering algorithm outputs the correct block structure (i.e., $\tilde{Z} = \hat{Z}$) and the event $\mathcal{G}_{m, C}$ that $\left| \| Z \|_{\mathrm{op}} - \| \tilde{Z} \|_{\mathrm{op}} \right| \leq C$ holds. Such a joint probability satisfies the following inequality: 
\begin{eqnarray}
\label{eq:jointp}
\mathrm{Pr} \left( \mathcal{F}_m \cap \mathcal{G}_{m, C} \right) \geq 1 - \mathrm{Pr} \left( \mathcal{F}^{\mathrm{C}}_m \right) - \mathrm{Pr} \left( \mathcal{G}^{\mathrm{C}}_{m, C} \right), 
\end{eqnarray}
where $\mathcal{A}^{\mathrm{C}}$ is the complement of event $\mathcal{A}$. 
The consistency assumption \ref{asmp:consistency} guarantees that if $(K_0, H_0) = (K, H)$, $\mathrm{Pr} \left( \mathcal{F}^{\mathrm{C}}_m \right)$ converges to $0$ in the limit of $m \to \infty$. By combining this fact with (\ref{eq:Ztilde_op}), we obtain 
\begin{eqnarray}
\label{eq:jointp2}
\forall \epsilon>0, \ \exists C>0, M>0, \ \forall m \geq M, \ 
\mathrm{Pr} \left( \mathcal{F}_m \cap \mathcal{G}_{m, C} \right) \geq 1 - \epsilon, 
\end{eqnarray}
which results in 
\begin{eqnarray}
\label{eq:Zhat_op}
% [Revision 2020/9] <---
\left| \| Z \|_{\mathrm{op}} - \| \hat{Z} \|_{\mathrm{op}} \right| = O_p (1). 
\end{eqnarray}
By using the above results, we can prove that the following equation holds for all $\epsilon \in \left( 0, \frac{2}{7} \right)$ (The proof is given in Appendix \ref{sec:ap_lambda_diff}): 
\begin{eqnarray}
\label{eq:Zhat2_op}
\frac{| \lambda_1 - \hat{\lambda}_1 |}{b} = O_p \left( m^{-\frac{1}{21} + \epsilon} \right), 
\end{eqnarray}

From (\ref{eq:T_true}), (\ref{eq:Zhat2_op}), and Slutsky's theorem, by setting $\epsilon < \frac{1}{21}$, 
% [Revision 2020/9] --->
\begin{eqnarray}
\frac{\hat{\lambda}_1 - a}{b} = T^* + \frac{\hat{\lambda}_1 - \lambda_1}{b} \rightsquigarrow TW_1 \ \mathrm{(Convergence\ in\ law)}. 
\end{eqnarray}
This is equivalent to the statement of Theorem \ref{th:realizable}. 
\end{proof}

%=======================================================================
\begin{theorem}[Unrealizable case, lower bound]
Suppose $K_0 < K$ or $H_0 < H$. 
\begin{eqnarray}
T = \Omega_p \left( \frac{m^{\frac{5}{3}}}{K^2 H^2} \right), 
\end{eqnarray}
where $T$ is defined as in (\ref{eq:T_statistic}). 
\label{th:unrealizable_lower}
\end{theorem}
\begin{proof}
Let $\bar{P}$ be a matrix that consists of the \textbf{estimated} block structure and whose entries are the population block-wise means, which can be calculated using $P$ (see also Figure \ref{fig:unrealizable}). 
\begin{figure}[t]
  \centering
  % [Revision 2020/9] <---
  \includegraphics[width=\hsize]{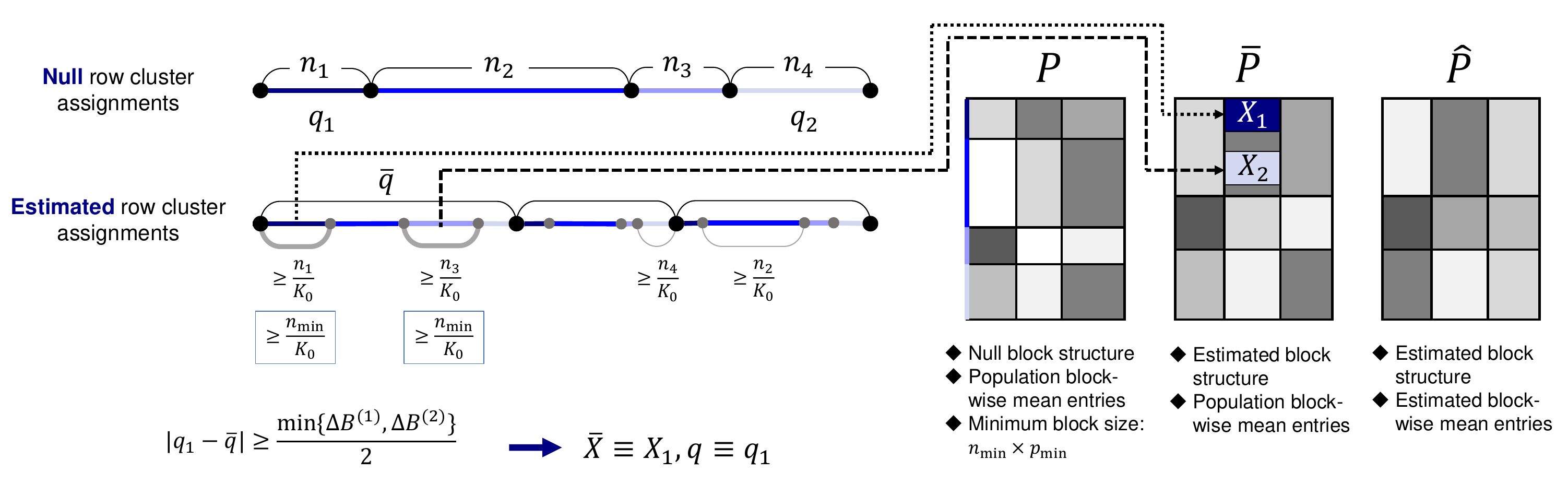} 
  % [Revision 2020/9] --->
  \caption{Difference between matrices $P$, $\bar{P}$, and $\hat{P}$ in an unrealizable case. }
  \label{fig:unrealizable}
\end{figure}
To derive the difference between matrices $P$ and $\hat{P}$, we first  focus on the relationship between matrices $P$ and $\bar{P}$. In the unrealizable case (i.e., $K_0 < K$ or $H_0 < H$), we can assume $K_0 < K$ without loss of generality. 

Let $n_k$ be the number of rows in the $k$th \textbf{null} row cluster. For all the \textbf{null} row cluster indices $k \in \{1, \dots, K\}$, at least one \textbf{estimated} row cluster contains $n_k/K_0$ or more rows that are assigned to the $k$th row cluster in the \textbf{null} block structure (otherwise, the total number of rows in the $k$th \textbf{null} row cluster is smaller than $n_k$). Since $K_0 < K$, at least one estimated block contains two or more sets of rows whose \textbf{null} row clusters are mutually different, and both of which have the row sizes of at least $n_{\mathrm{min}}/K_0$, where $n_{\mathrm{min}}$ is the minimum row size of a block in the \textbf{null} block structure. By the same reasoning, for all the \textbf{null} column cluster indices $h \in \{1, \dots, H\}$, at least one \textbf{estimated} column cluster contains $p_h/H_0$ or more columns that are assigned to the $h$th column cluster in the \textbf{null} block structure, where $p_h$ is the number of rows in the $h$th \textbf{null} column cluster. By combining these facts, there exists at least one \textbf{estimated} block that contains two or more submatrices, both of which have the sizes of at least $(n_{\mathrm{min}}/K_0) \times (p_{\mathrm{min}}/H_0)$ and whose \textbf{null} blocks are mutually different. 

Let $X_1$ and $X_2$ be such submatrices, whose \textbf{null} block-wise mean are $q_1$ and $q_2$, respectively. We can assume $q_1 > q_2$ without the loss of generality. In matrix $\bar{P}$, which has the \textbf{estimated} block structure, both of $X_1$ and $X_2$ have the same values $\bar{q}$. 
% [Revision 2020/9] <---
Here, $|q_2 - \bar{q}| \geq \frac{|q_1 - q_2|}{2}$ holds if $\bar{q} \geq \frac{q_1 + q_2}{2}$, and $|q_1 - \bar{q}| \geq \frac{|q_1 - q_2|}{2}$ otherwise. 
% [Revision 2020/9] --->
Therefore, for any $\bar{q}$, there exists at least one submatrix $\bar{X}$ (which is either $X_1$ or $X_2$) with a size of at least $(n_{\mathrm{min}}/K_0) \times (p_{\mathrm{min}}/H_0)$, where all the entries are $q$ (which is either $q_1$ or $q_2$) in matrix $P$ and 
% [Revision 2020/9] <---
\begin{align}
\label{eq:q_bar_q_bound}
|q - \bar{q}| &\geq \frac{\min \{ \Delta B^{(1)}, \Delta B^{(2)} \}}{2}, \nonumber \\
\Delta B^{(1)} &\equiv \min_{k, k' \in \{ 1, \dots, K \}, h \in \{ 1, \dots, H \}} |B_{kh} - B_{k'h}|, \nonumber \\
\Delta B^{(2)} &\equiv \min_{k \in \{ 1, \dots, K \}, h, h' \in \{ 1, \dots, H \}} |B_{kh} - B_{kh'}|. 
\end{align}
% [Revision 2020/9] --->

Let $(k_1, h_1)$ be the row and column cluster indices of the \textbf{estimated} block which contains the above submatrix $\bar{X}$. We denote the row and column sizes of the $(k_1, h_1)$th \textbf{estimated} block as $\underline{n}_1$ and $\underline{p}_1$, respectively. 
% [Revision 2020/9] <---
Let $\underline{A}^{(k_1, h_1)}$, $\underline{P}^{(k_1, h_1)}$, $\underline{\bar{P}}^{(k_1, h_1)}$, and $\underline{\hat{P}}^{(k_1, h_1)}$, respectively, be the $(k_1, h_1)$th \textbf{estimated} block of $A$, $P$, $\bar{P}$, and $\hat{P}$. We define $\hat{q} \equiv \hat{B}_{k_1 h_1}$. 
% [Revision 2020/9] --->
In regard to the difference between matrices $\bar{P}$ and $\hat{P}$ (both of which have the \textbf{estimated} block structure), we have 
\begin{align}
\label{eq:o_block_un}
% [Revision 2020/9] <---
&|\hat{q} - \bar{q}| = \frac{1}{\underline{n}_1 \underline{p}_1} \left| \sum_{i=1}^{\underline{n}_1} \sum_{j=1}^{\underline{p}_1} \left( \underline{\hat{P}}^{(k_1, h_1)}_{ij} - \underline{\bar{P}}^{(k_1, h_1)}_{ij} \right) \right| \nonumber \\
&= \frac{1}{\underline{n}_1 \underline{p}_1} \left| \sum_{i=1}^{\underline{n}_1} \sum_{j=1}^{\underline{p}_1} \left( \underline{A}^{(k_1, h_1)}_{ij} - \underline{P}^{(k_1, h_1)}_{ij} \right) \right| 
= \frac{1}{\underline{n}_1 \underline{p}_1} \left| \left\langle \bm{u}_1, \left( \underline{A}^{(k_1, h_1)} - \underline{P}^{(k_1, h_1)} \right) \bm{u}_2 \right\rangle \right| \nonumber \\
&\leq \frac{1}{\underline{n}_1 \underline{p}_1} \| \bm{u}_1 \| \| \bm{u}_2 \| \| \underline{A}^{(k_1, h_1)} - \underline{P}^{(k_1, h_1)} \|_{\mathrm{op}} 
= \frac{1}{\sqrt{\underline{n}_1 \underline{p}_1}} \| \underline{A}^{(k_1, h_1)} - \underline{P}^{(k_1, h_1)} \|_{\mathrm{op}} \nonumber \\
&\leq \sqrt{\frac{K_0 H_0}{n_{\mathrm{min}} p_{\mathrm{min}}}} \| \underline{A}^{(k_1, h_1)} - \underline{P}^{(k_1, h_1)} \|_{\mathrm{op}} 
\leq \sqrt{\frac{K_0 H_0}{n_{\mathrm{min}} p_{\mathrm{min}}}} \| A - P \|_{\mathrm{op}} \nonumber \\
&\leq \sqrt{\frac{K_0 H_0}{n_{\mathrm{min}} p_{\mathrm{min}}}} \sum_{k=1}^K \sum_{h=1}^H \| A^{(k, h)} - P^{(k, h)} \|_{\mathrm{op}} 
= \sqrt{\frac{K_0 H_0}{n_{\mathrm{min}} p_{\mathrm{min}}}} \sum_{k=1}^K \sum_{h=1}^H S_{kh} \| Z^{(k, h)} \|_{\mathrm{op}} \nonumber \\
&\leq \sqrt{\frac{K_0 H_0}{n_{\mathrm{min}} p_{\mathrm{min}}}} KH \max_{k = 1, \dots, K, h = 1, \dots, H} S_{kh} \| Z \|_{\mathrm{op}} 
= O_p \left( \frac{KH}{\sqrt{m}} \right). 
\end{align}
where $A^{(k, h)}$, $P^{(k, h)}$, and $Z^{(k, h)}$, respectively, are the $(k, h)$th \textbf{null} blocks of matrices $A$, $P$, and $Z$, and $\bm{u}_1 = [1, 1, \dots, 1]^{\top} \in \mathbb{R}^{\underline{n}_1}$ and $\bm{u}_2 = [1, 1, \dots, 1]^{\top} \in \mathbb{R}^{\underline{p}_1}$. To derive the final equation in (\ref{eq:o_block_un}), we used the assumption that $n_{\mathrm{min}}, p_{\mathrm{min}} = \Omega_p (m)$ and the fact that $\| Z \|_{\mathrm{op}}$ is equal to the largest singular value of $Z$, which is $O_p (\sqrt{m})$ from (\ref{eq:T_true}). 
% [Revision 2020/9] --->

Let $\mathcal{E}_{m, C}$ be the event that $|q - \bar{q}| - CKH/\sqrt{m} \leq |q - \hat{q}|$ holds. For all $q$, $\bar{q}$, and $\hat{q}$, the following inequality holds: 
\begin{eqnarray}
\label{eq:q_qhat_qbar}
\Bigl| |q - \bar{q}| - |q - \hat{q}| \Bigr| \leq |\hat{q} - \bar{q}|. 
\end{eqnarray}
By combining (\ref{eq:o_block_un}) and (\ref{eq:q_qhat_qbar}), we obtain 
\begin{eqnarray}
\label{eq:q_qhat_qbar_prob}
\forall \epsilon>0, \ \exists C>0, M>0, \ \forall m \geq M, \ \mathrm{Pr} (\mathcal{E}_{m, C}) \geq 1-\epsilon. 
\end{eqnarray}

% [Revision 2020/9] <---
From now on, we denote the row and column sizes of submatrix $\bar{X}$, respectively, by $\bar{n}_1$ and $\bar{p}_1$. Let $A^*$, $P^*$, $\bar{P}^*$, $\hat{P}^*$, $Z^*$, and $\hat{Z}^*$, respectively, be the submatrices of matrices $A$, $P$, $\bar{P}$, $\hat{P}$, $Z$, and $\hat{Z}$ with the same row and column indices as submatrix $\bar{X}$. 
% [Revision 2020/9] --->
We also denote the constant entries of the submatrices of $\sigma$ and $\hat{\sigma}$ with the same row and column indices as submatrix $\bar{X}$, respectively, as $\sigma^*$ and $\hat{\sigma}^*$. 
From the definition (\ref{eq:Z_hat}) and since the operator norm of a submatrix is not larger than that of the original matrix, we have 
\begin{align}
\label{eq:z_hat_op_partial}
\| \hat{Z} \|_{\mathrm{op}} 
\geq&\ \| \hat{Z}^* \|_{\mathrm{op}} 
= \frac{1}{\hat{\sigma}^*} \| A^* - \hat{P}^* \|_{\mathrm{op}} 
= \frac{1}{\hat{\sigma}^*} \| ( A^* - P^* ) + ( P^* - \hat{P}^* ) \|_{\mathrm{op}} \nonumber \\
\geq&\ \frac{1}{\hat{\sigma}^*} \left| \| A^* - P^* \|_{\mathrm{op}} - \| P^* - \hat{P}^* \|_{\mathrm{op}} \right| \nonumber \\
=&\ \frac{1}{\hat{\sigma}^*} \left| \sigma^* \| Z^* \|_{\mathrm{op}} - \| P^* - \hat{P}^* \|_{\mathrm{op}} \right|. 
\end{align}

First, the order of the estimated standard deviation $\hat{\sigma}^*$ is given by 
\begin{eqnarray}
\label{eq:hat_sigma_star_o}
\hat{\sigma}^*  = O_p (KH). 
\end{eqnarray}
The proof of (\ref{eq:hat_sigma_star_o}) is in Appendix \ref{ap_sigma_star}. 

% [Revision 2020/9] <---
The only non-zero (and thus, the largest) singular value of matrix $\left( P^* - \hat{P}^* \right)$ is $\sqrt{\bar{n}_1 \bar{p}_1} \left| q - \hat{q} \right|$. Since the largest singular value of a matrix is equal to its operator norm, we have 
\begin{eqnarray}
\| P^* - \hat{P}^* \|_{\mathrm{op}} = \sqrt{\bar{n}_1 \bar{p}_1} \left| q - \hat{q} \right| 
\geq \sqrt{\frac{n_{\mathrm{min}}}{K_0} \frac{p_{\mathrm{min}}}{H_0}} \left| q - \hat{q} \right|. 
\end{eqnarray}
% [Revision 2020/9] --->
Therefore, by combining this fact with (\ref{eq:q_bar_q_bound}), if the statement of event $\mathcal{E}_{m, C}$ holds, the following inequality also holds: 
\begin{eqnarray}
\label{eq:p_phat_op}
% [Revision 2020/9] <---
\sqrt{\frac{n_{\mathrm{min}}}{K_0} \frac{p_{\mathrm{min}}}{H_0}} \left( \frac{\min \{ \Delta B^{(1)}, \Delta B^{(2)} \}}{2} - \frac{CKH}{\sqrt{m}} \right) \leq \| P^* - \hat{P}^* \|_{\mathrm{op}}, 
\end{eqnarray}
which results in that $\| P^* - \hat{P}^* \|_{\mathrm{op}} = \Omega_p \left( \delta m \right)$, where $\delta \equiv \min \{ \Delta B^{(1)}, \Delta B^{(2)} \}$. 
% [Revision 2020/9] --->

Also, from (\ref{eq:T_true}), we have $\| Z^* \|_{\mathrm{op}} \leq \| Z \|_{\mathrm{op}} =O_p (\sqrt{m})$. 
% [Revision 2020/9] <---
By substituting this fact, (\ref{eq:hat_sigma_star_o}), and (\ref{eq:p_phat_op}) into (\ref{eq:z_hat_op_partial}), we finally obtain
% [Revision 2020/9] --->
\begin{eqnarray}
\label{eq:apop_prob}
\| \hat{Z} \|_{\mathrm{op}}^2 = \Omega_p \left(\frac{\delta^2 m^2}{K^2 H^2} \right). 
\end{eqnarray}

Here, $\| \hat{Z} \|_{\mathrm{op}}^2$ is equal to the maximum eigenvalue $\hat{\lambda}_1$ of matrix $\hat{Z}^\top \hat{Z}$, and the test statistic is $T = \frac{\hat{\lambda}_1 - a}{b}$. Using the definition (\ref{eq:ab}), we obtain $a = O_p (m)$ and 
\begin{align}
\label{eq:bm}
b =&\ \left( \sqrt{n} + \sqrt{p} \right) \left( \frac{1}{\sqrt{n}} + \frac{1}{\sqrt{p}} \right)^{\frac{1}{3}} 
= \left( \sqrt{\beta_1 m} + \sqrt{\beta_2 m} \right) \left( \frac{1}{\sqrt{\beta_1 m}} + \frac{1}{\sqrt{\beta_2 m}} \right)^{\frac{1}{3}} \nonumber \\
=&\ m^{\frac{1}{3}} \left( \sqrt{\beta_1} + \sqrt{\beta_2} \right) \left( \frac{1}{\sqrt{\beta_1}} + \frac{1}{\sqrt{\beta_2}} \right)^{\frac{1}{3}}, 
\end{align}
where we used the definitions $\beta_1 \equiv n/m$ and $\beta_2 \equiv p/m$. 

By combining these results and (\ref{eq:apop_prob}), we obtain
\begin{eqnarray}
T m^{\frac{1}{3}} = \Omega_p \left(\frac{\delta^2 m^2}{K^2 H^2} \right) 
\iff T = \Omega_p \left( \frac{\delta^2 m^{\frac{5}{3}}}{K^2 H^2} \right), 
\end{eqnarray}
which concludes the proof. 
\end{proof}

%=========================================================================
\begin{theorem}[Unrealizable case, upper bound]
Suppose $K_0 < K$ or $H_0 < H$. Then, 
\begin{eqnarray}
T = O_p \left( m^{\frac{5}{3}} \right), 
\end{eqnarray}
where $T$ is defined as in (\ref{eq:T_statistic}). 
\label{th:unrealizable_upper}
\end{theorem}
\begin{proof}
% [Revision 2020/9] <---
We define $P$, $\bar{P}$, and $\hat{P}$ as in Theorem \ref{th:unrealizable_lower}. Let $\underline{\hat{Z}}^{(k, h)}$, $\underline{A}^{(k, h)}$, and $\underline{\hat{P}}^{(k, h)}$, respectively, be the $(k, h)$th \textbf{estimated} blocks of matrices $\hat{Z}$, $A$, and $\hat{P}$. We denote the row and column sizes of the $(k, h)$th \textbf{estimated} block as $\underline{n}_k$ and $\underline{p}_h$, respectively. Since the operator norm of a matrix is not larger than the sum of the operator norms of all its blocks, we have
\begin{align}
\label{eq:z_hat_op_upper}
\| \hat{Z} \|_{\mathrm{op}} 
&\leq \sum_{k=1}^{K_0} \sum_{h=1}^{H_0} \| \underline{\hat{Z}}^{(k, h)} \|_{\mathrm{op}} 
= \sum_{k=1}^{K_0} \sum_{h=1}^{H_0} \frac{1}{\hat{S}_{kh}} \| \underline{A}^{(k, h)} - \underline{\hat{P}}^{(k, h)} \|_{\mathrm{op}} \nonumber \\
&= \sum_{k=1}^{K_0} \sum_{h=1}^{H_0} \frac{\sqrt{\underline{n}_k \underline{p}_h}}{\| \underline{A}^{(k, h)} - \underline{\hat{P}}^{(k, h)} \|_{\mathrm{F}}} \| \underline{A}^{(k, h)} - \underline{\hat{P}}^{(k, h)} \|_{\mathrm{op}} \nonumber \\
&\leq \sum_{k=1}^{K_0} \sum_{h=1}^{H_0} \frac{\sqrt{\underline{n}_k \underline{p}_h}}{\| \underline{A}^{(k, h)} - \underline{\hat{P}}^{(k, h)} \|_{\mathrm{F}}} \| \underline{A}^{(k, h)} - \underline{\hat{P}}^{(k, h)} \|_{\mathrm{F}} \nonumber \\
&= \sum_{k=1}^{K_0} \sum_{h=1}^{H_0} \sqrt{\underline{n}_k \underline{p}_h} 
\leq K_0 H_0 \sqrt{np} = O_p (m).
% [Revision 2020/9] --->
\end{align}

The test statistic is $T = \frac{\hat{\lambda}_1 - a}{b}$, where $\hat{\lambda}_1 = \| \hat{Z} \|_{\mathrm{op}}^2 = O_p (m^2)$. Based on the same discussion as in Theorem \ref{th:unrealizable_lower}, $a = O_p (m)$ and (\ref{eq:bm}) hold. Consequently, we obtain $T = O_p (m^2 / m^{\frac{1}{3}}) = O_p (m^{\frac{5}{3}})$, which concludes the proof. 
\end{proof}
%=========================================================================================. 

%%%%%%%%%%%%%%%%%%%%%%%%%%%%%%%%%%%%%%%%%%%%%%%%%%%%%%%

\section{Experiments}
\label{sec:experiments}

\subsection{Realizable case: convergence of test statistic $T$ in law to Tracy-Widom distribution}
\label{sec:exp_realizable}

First of all, we checked the convergence of the proposed test statistic $T$ in law to the Tracy-Widom distribution with index $1$, under the \textit{realizable} setting, which has been stated in Theorem \ref{th:realizable}, by using synthetic data that were generated based on three types of distributions: 

\begin{itemize}
\item \textbf{Gaussian Latent Block Model}: The observed matrices were generated whose entries in the $(k, h)$th block follows the normal distribution $\mathcal{N} (B_{kh}, S_{kh})$. In the Gaussian LBM setting, we used the following null model and parameters: 
\begin{eqnarray}
(K, H) = (4, 3), \ 
B = 
\begin{pmatrix}
0.9 & 0.1 & 0.4  \\
0.2 & 0.7 & 0.3  \\
0.3 & 0.2 & 0.8  \\
0.6 & 0.9 & 0.1  \\
\end{pmatrix}, \ 
S = 
\begin{pmatrix}
0.08 & 0.06 & 0.15  \\
0.14 & 0.12 & 0.07  \\
0.09 & 0.1 & 0.11  \\
0.16 & 0.13 & 0.05  \\
\end{pmatrix}. 
\end{eqnarray}
\item \textbf{Bernoulli Latent Block Model} The observed matrices were generated whose entries in the $(k, h)$th block follows the normal distribution $\mathrm{Bernoulli} (B_{kh})$. In the Bernoulli LBM setting, we used the following null model and parameters: 
\begin{eqnarray}
(K, H) = (4, 3), \ \ \ \ \ 
B = 
\begin{pmatrix}
0.9 & 0.1 & 0.4  \\
0.2 & 0.7 & 0.3  \\
0.3 & 0.2 & 0.8  \\
0.6 & 0.9 & 0.1  \\
\end{pmatrix}. 
\end{eqnarray}
\item \textbf{Poisson Latent Block Model} The observed matrices were generated whose entries in the $(k, h)$th block follows the normal distribution $\mathrm{Pois} (B_{kh})$. In the Poisson LBM setting, we used the following null model and parameters: 
\begin{eqnarray}
(K, H) = (4, 3), \ \ \ \ \ 
B = 
\begin{pmatrix}
9.0 & 1.0 & 4.0  \\
2.0 & 7.0 & 3.0  \\
3.0 & 2.0 & 8.0  \\
6.0 & 9.0 & 1.0  \\
\end{pmatrix}. 
\end{eqnarray}
\end{itemize}

Based on the above Latent Block Model, we randomly generated $1000$ observed matrices, estimated their block structures based on the Ward's hierarchical clustering algorithm \cite{Ward1963}, and computed the test statistic $T$. With respect to the matrix size, we tried the following $10$ settings: $(n, p) = (300 \times i, 225 \times i)$, $i = 1, \dots, 10$. When generating an observed matrix, the null cluster of each row was randomly chosen from the discrete uniform distribution on $\{1, 2, 3, 4\}$. Similarly, the null cluster of each column was randomly chosen from the discrete uniform distribution on $\{1, 2, 3\}$. 

Figures \ref{fig:QQ_normal}, \ref{fig:QQ_bernoulli}, and \ref{fig:QQ_poisson}, respectively, show the Q-Q plots of the test statistic $T$ and the $TW_1$ distribution in the settings of Gaussian, Bernoulli, and Poisson settings. Each plotted point corresponds to a sample of test statistic $T$, and the horizontal and vertical lines, respectively, show its theoretical and sample quantiles. These figures show that the test statistic converged in law to the $TW_1$ distribution. 

Figure \ref{fig:preliminaryT} shows the ratios of the trials where $T \geq t(0.01)$, $T \geq t(0.05)$, and $T \geq t(0.1)$ for the above three distributional settings, where $t(\alpha)$ is the $\alpha$ upper quantile of the $TW_1$ distribution. We used the approximated values $t(0.01) \approx 2.02345$, $t(0.05) \approx 0.97931$, and $t(0.1) \approx 0.45014$, according to Table $2$ in \cite{Tracy2009}. From Figure \ref{fig:preliminaryT}, we see that the tail probability of the test statistic $T$ also converged to those of the $TW_1$ distributions for all of the three distributional settings. 

%-----
\begin{figure}
  \centering
  \includegraphics[width=0.192\hsize]{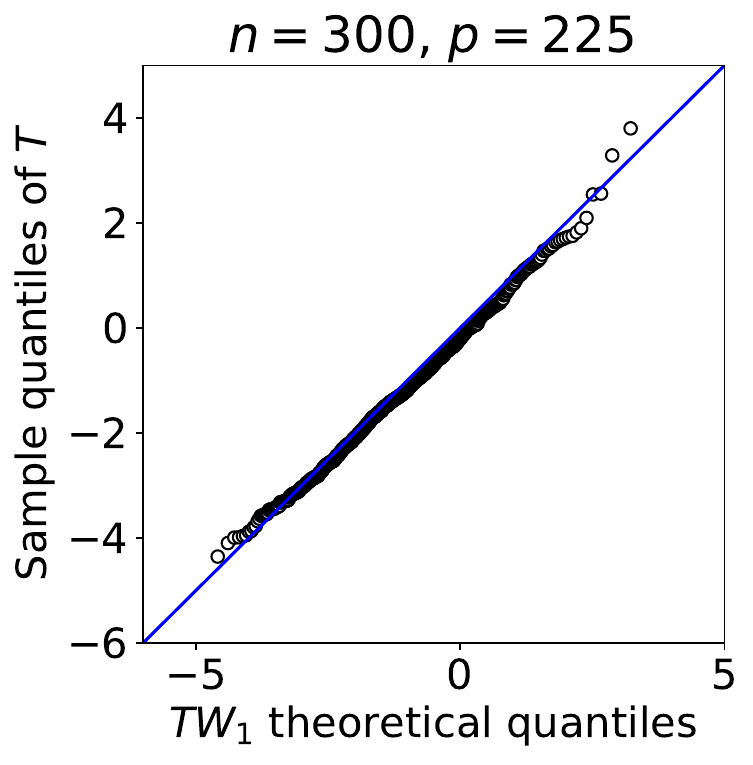}
  \includegraphics[width=0.192\hsize]{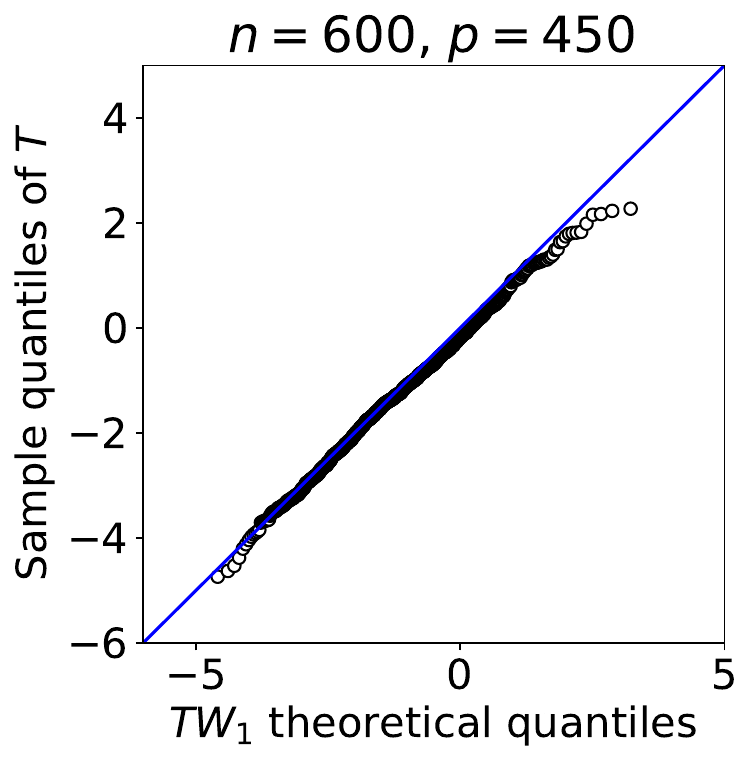}
  \includegraphics[width=0.192\hsize]{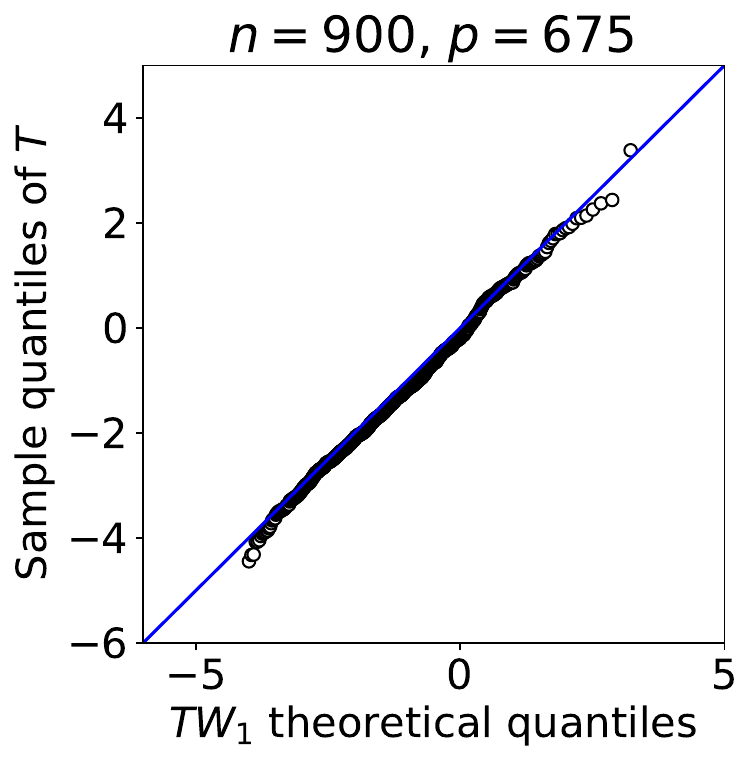}
  \includegraphics[width=0.192\hsize]{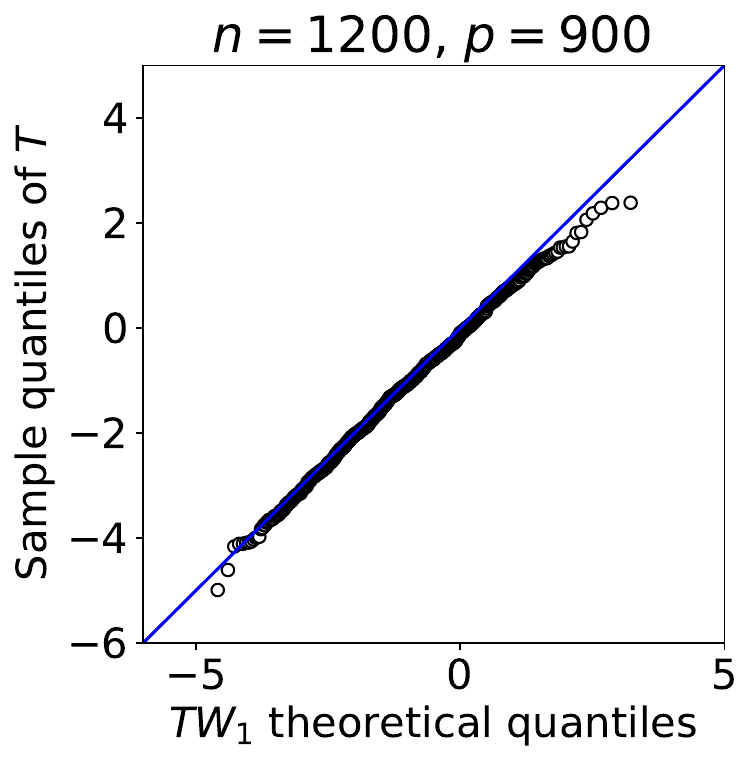}
  \includegraphics[width=0.192\hsize]{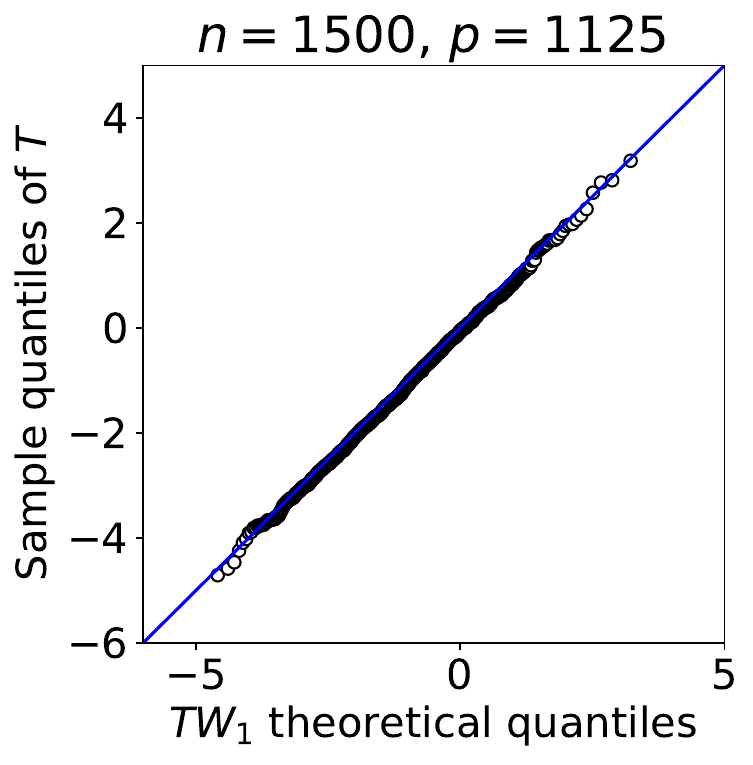}\\
  \includegraphics[width=0.192\hsize]{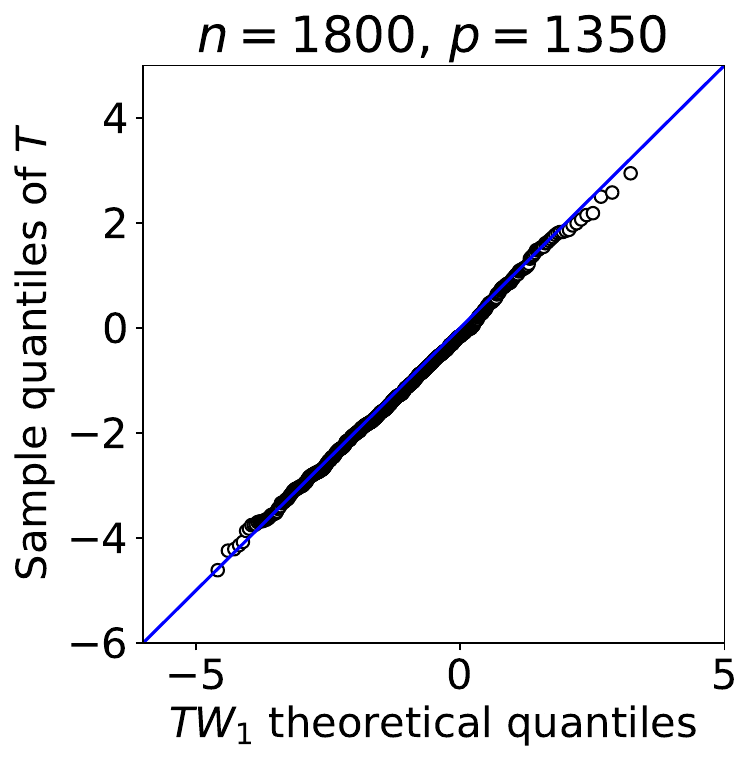}
  \includegraphics[width=0.192\hsize]{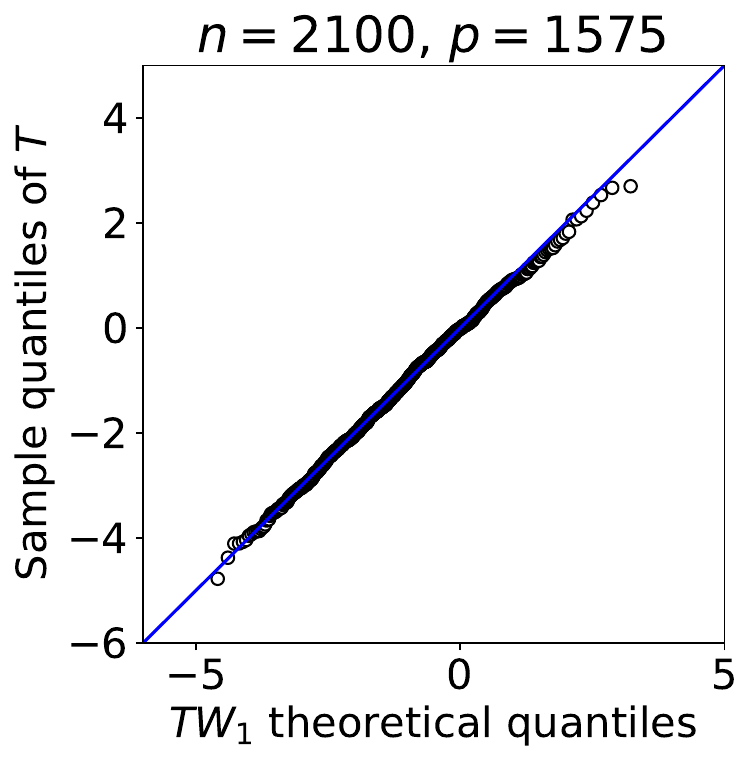}
  \includegraphics[width=0.192\hsize]{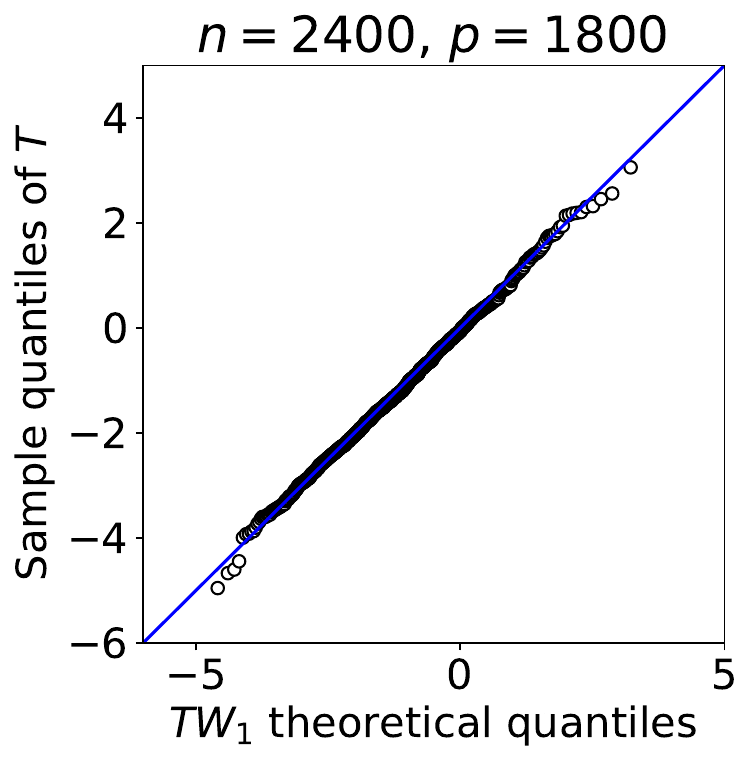}
  \includegraphics[width=0.192\hsize]{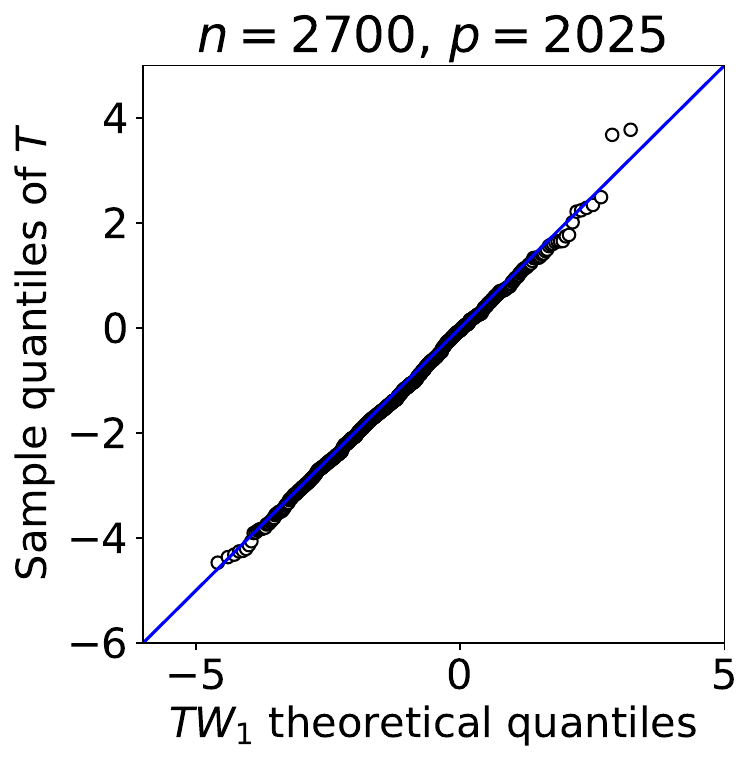}
  \includegraphics[width=0.192\hsize]{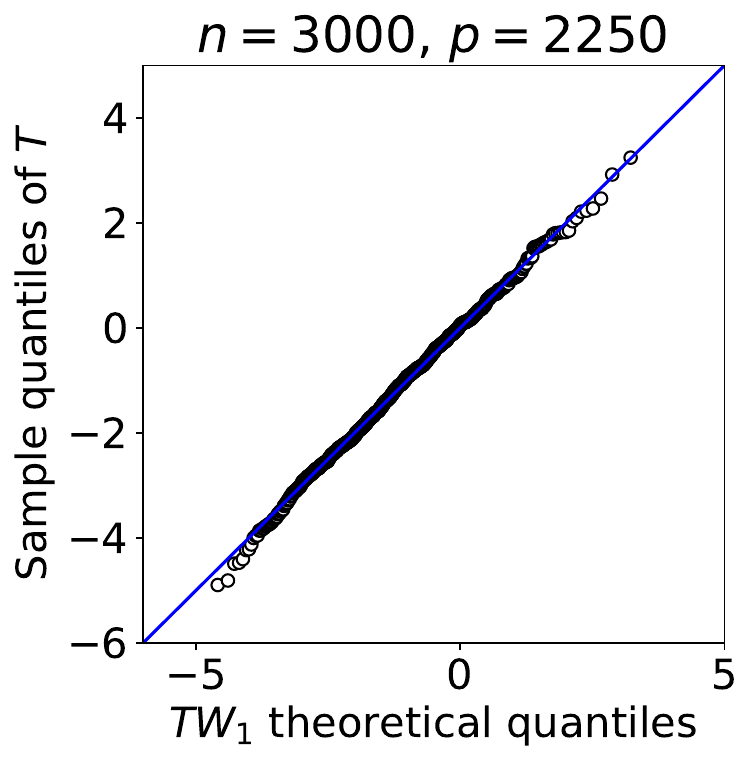}
  \caption{Q-Q plot of test statistic $T$ against the $TW_1$ distribution in the setting of \textbf{Gaussian case}. }\vspace{5mm}
  \label{fig:QQ_normal}
  \includegraphics[width=0.192\hsize]{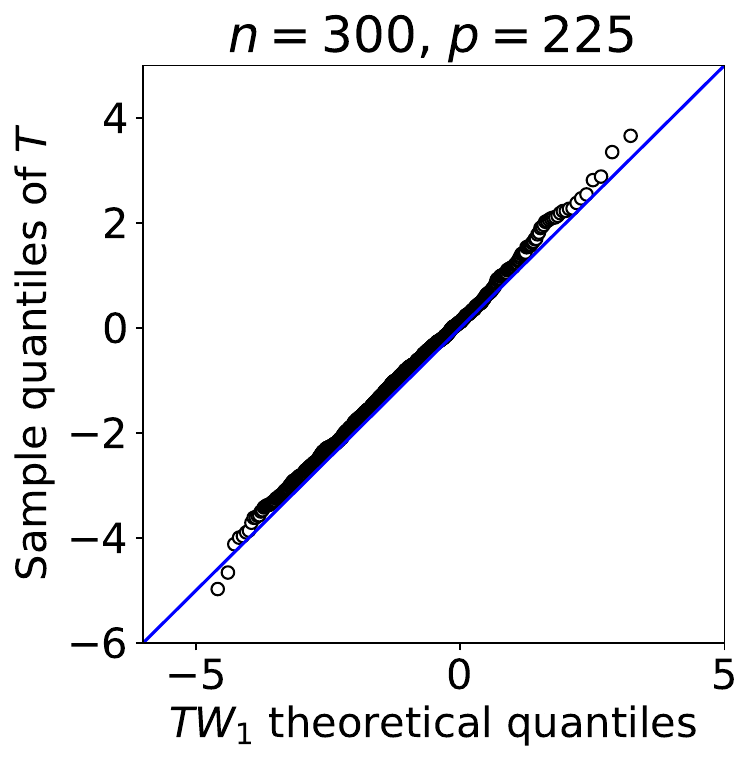}
  \includegraphics[width=0.192\hsize]{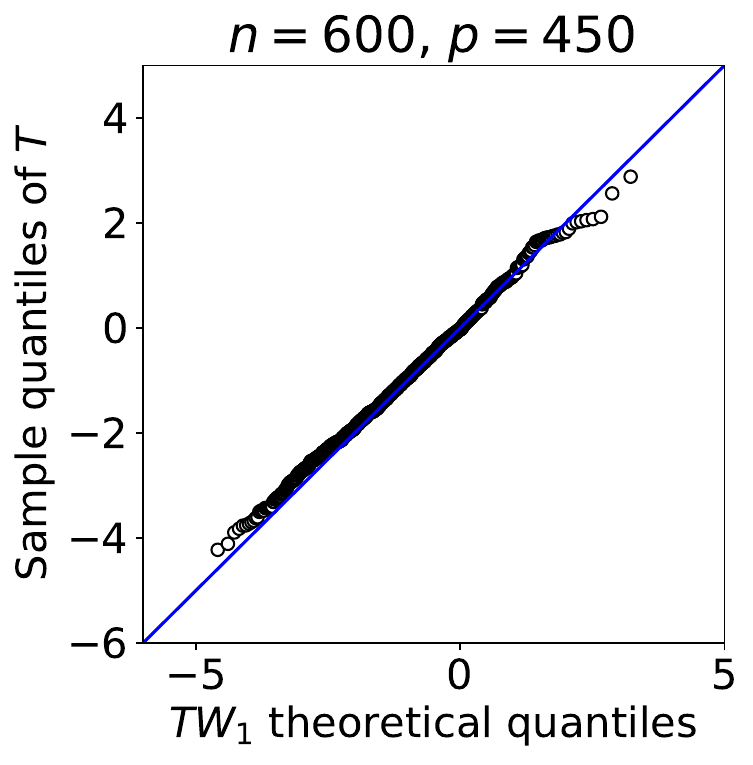}
  \includegraphics[width=0.192\hsize]{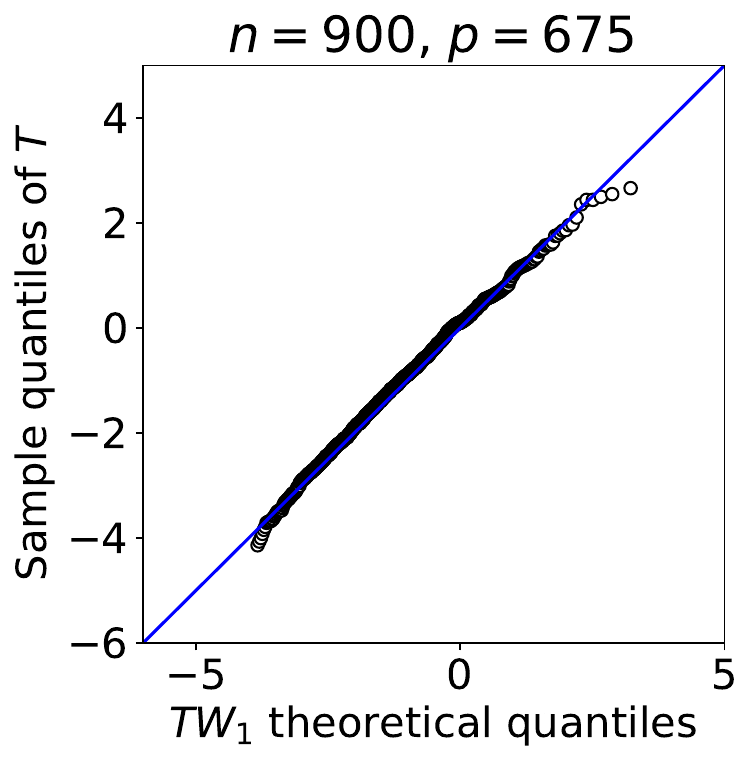}
  \includegraphics[width=0.192\hsize]{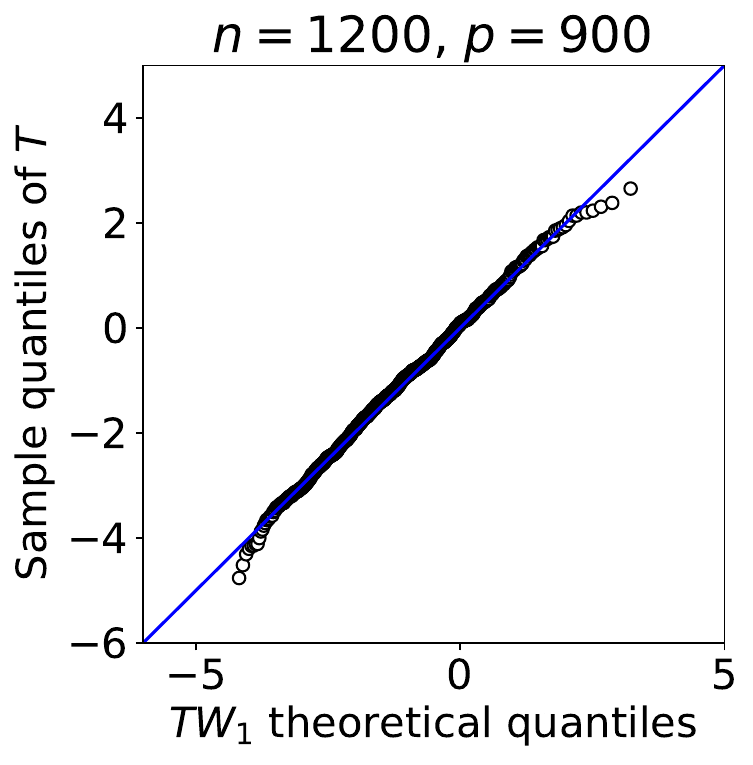}
  \includegraphics[width=0.192\hsize]{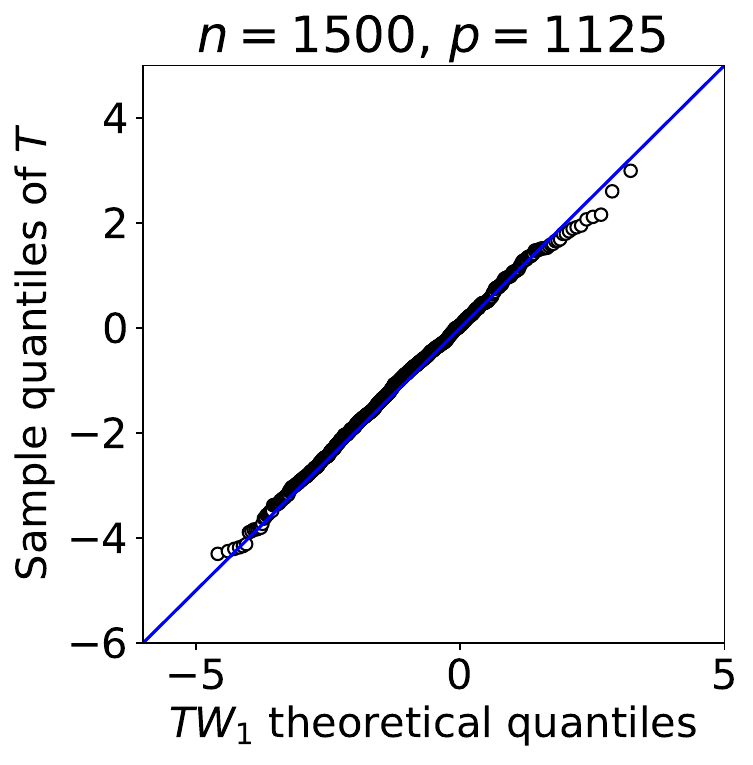}\\
  \includegraphics[width=0.192\hsize]{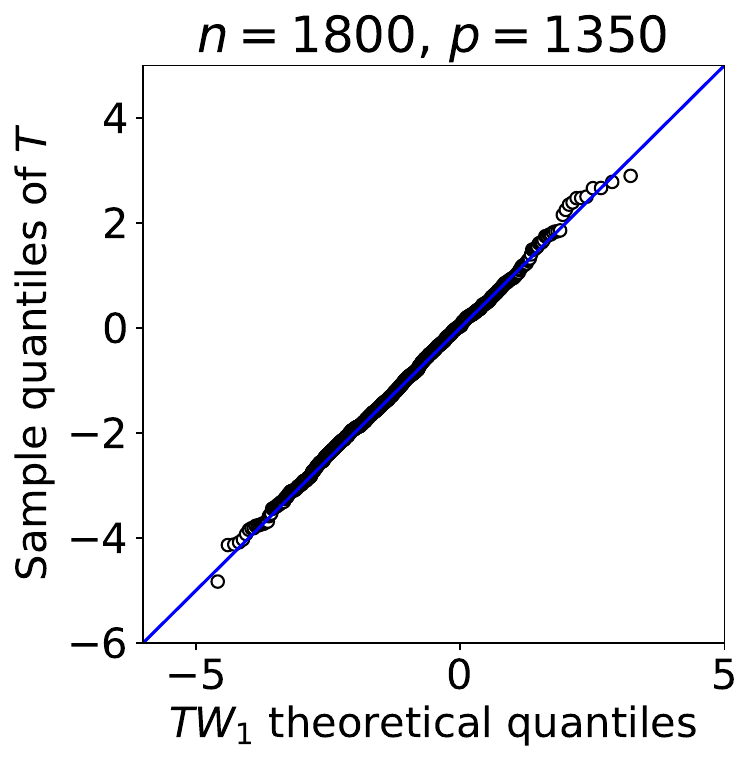}
  \includegraphics[width=0.192\hsize]{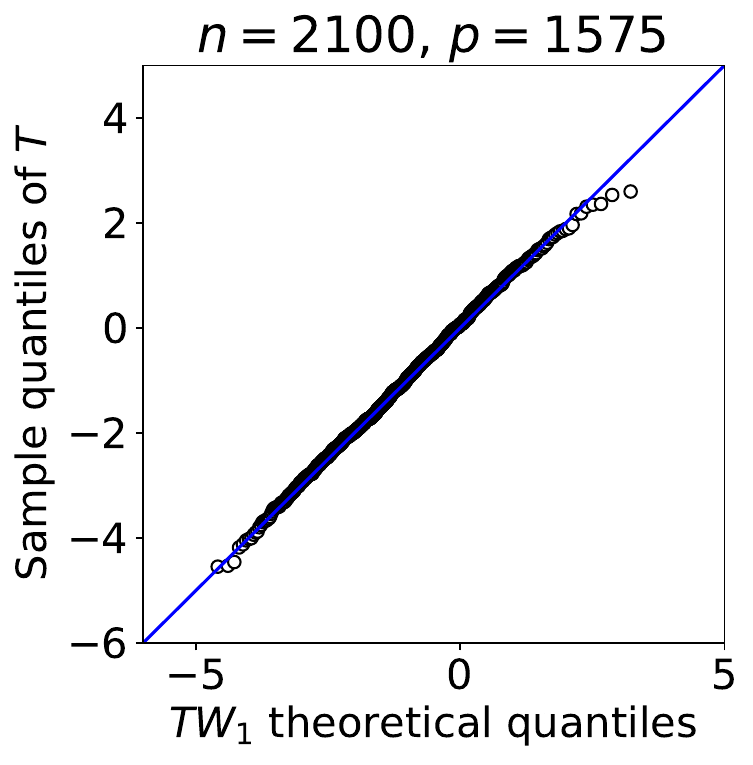}
  \includegraphics[width=0.192\hsize]{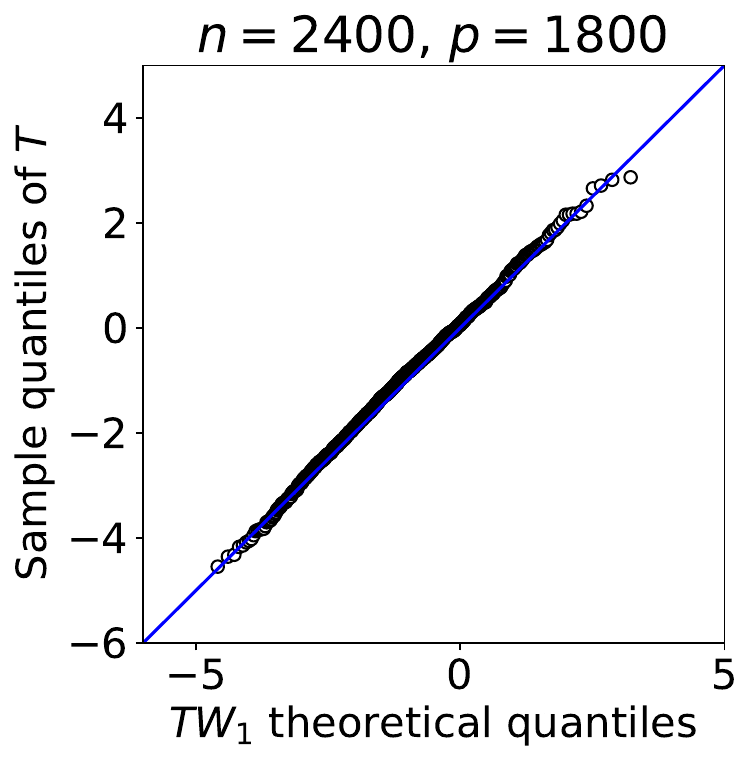}
  \includegraphics[width=0.192\hsize]{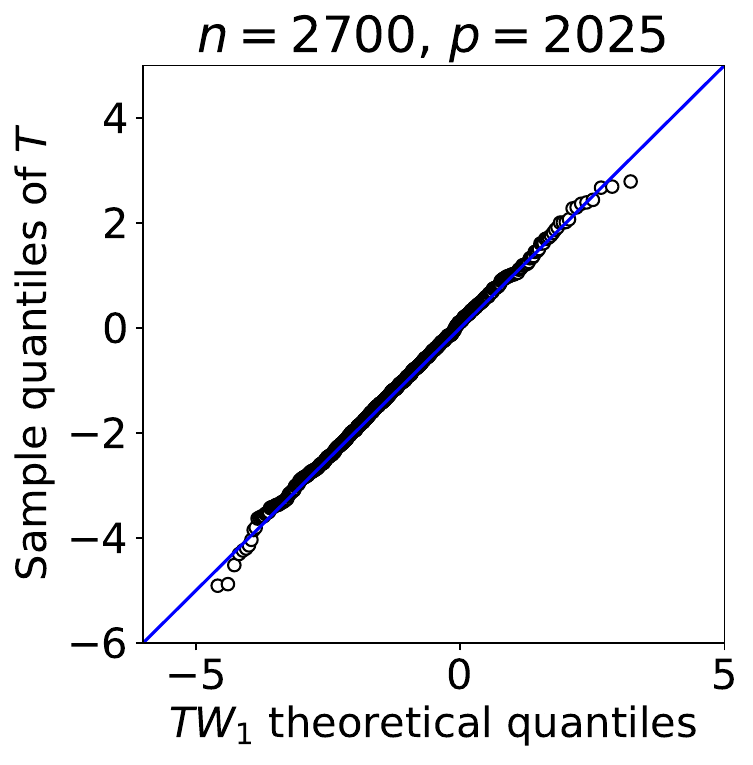}
  \includegraphics[width=0.192\hsize]{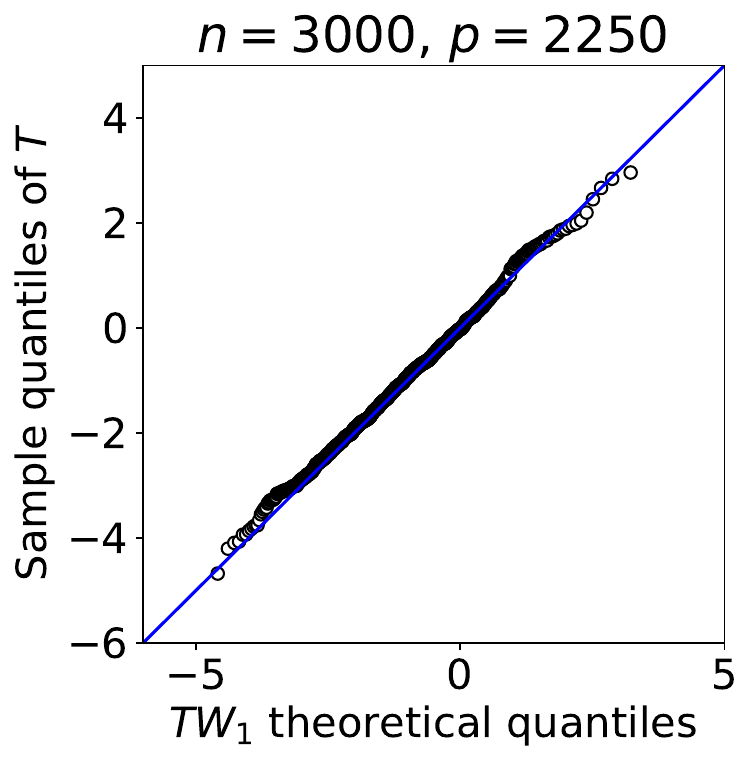}
  \caption{Q-Q plot of test statistic $T$ against the $TW_1$ distribution in the setting of \textbf{Bernoulli case}. }\vspace{5mm}
  \label{fig:QQ_bernoulli}
  \includegraphics[width=0.192\hsize]{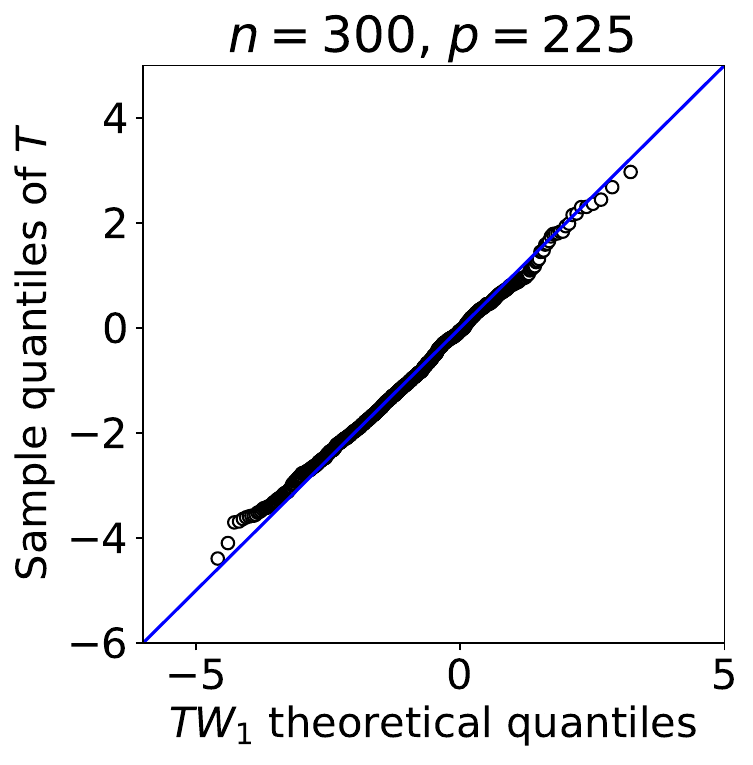}
  \includegraphics[width=0.192\hsize]{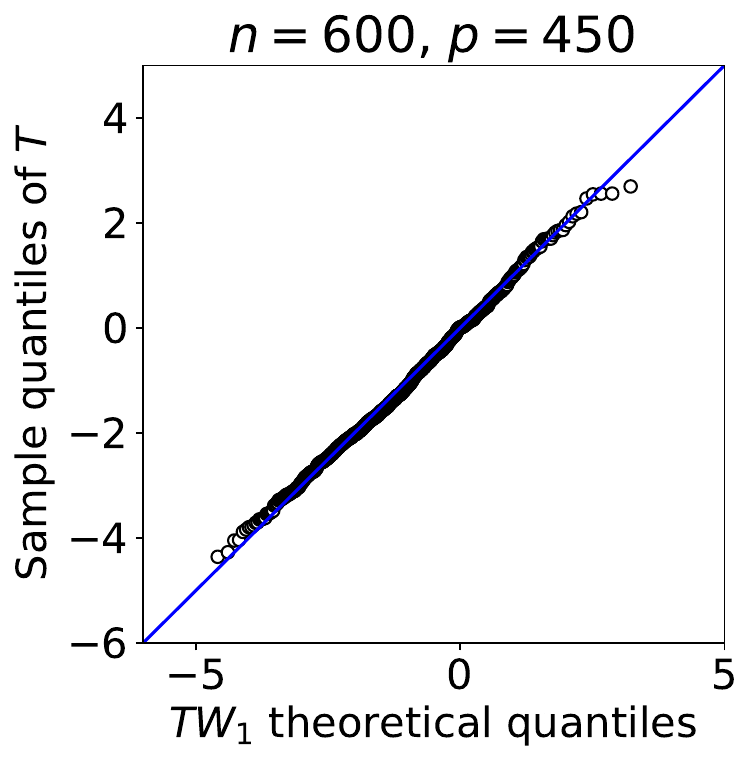}
  \includegraphics[width=0.192\hsize]{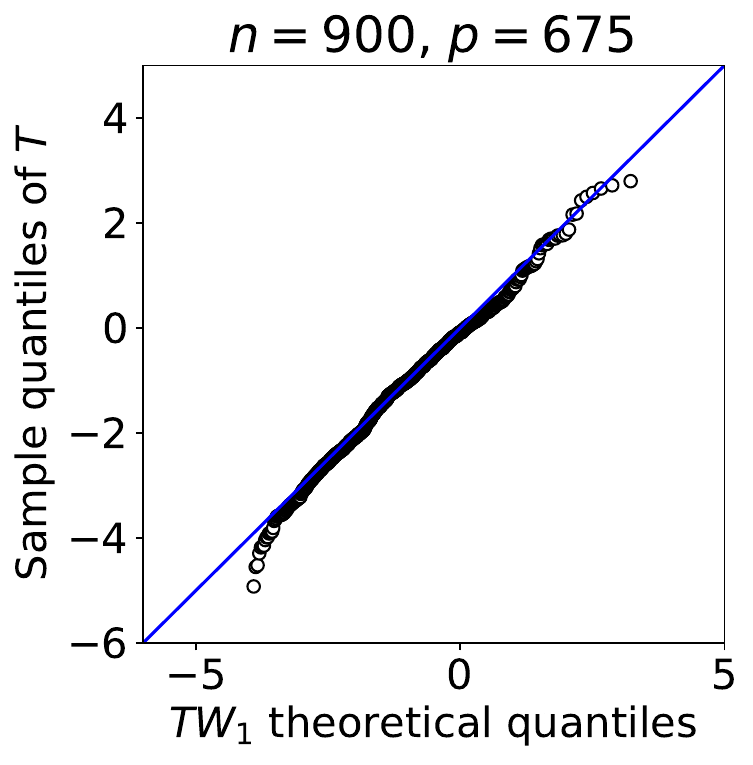}
  \includegraphics[width=0.192\hsize]{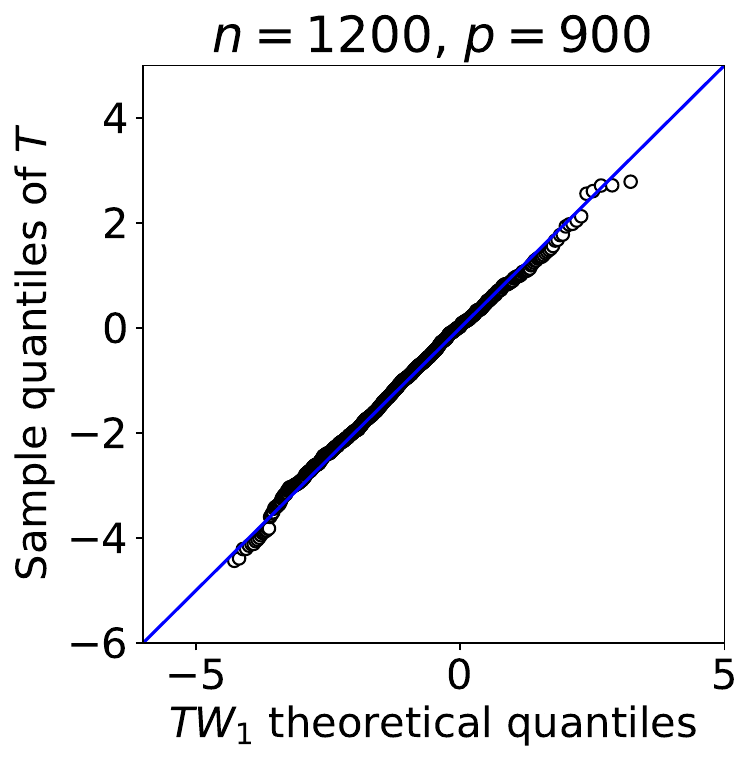}
  \includegraphics[width=0.192\hsize]{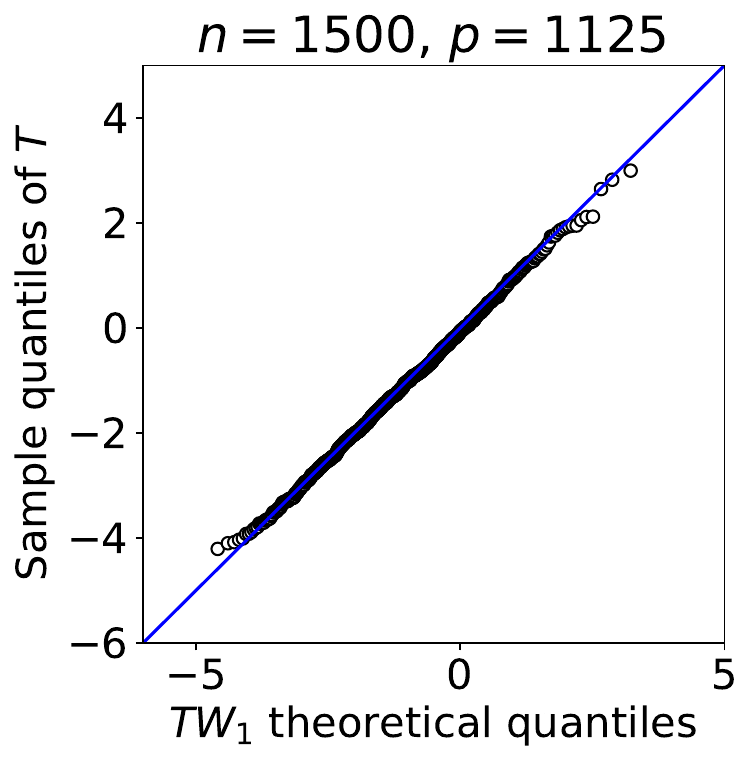}\\
  \includegraphics[width=0.192\hsize]{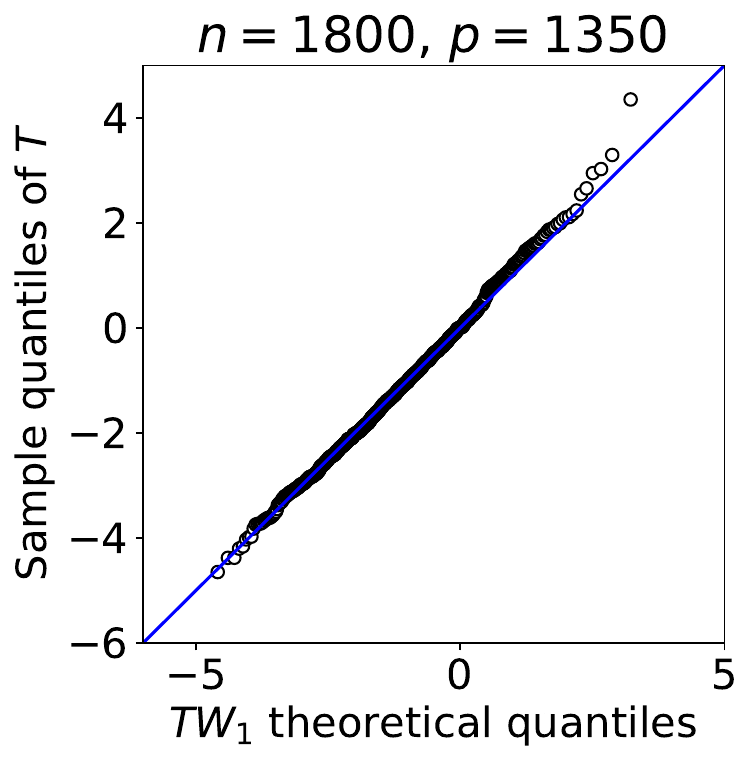}
  \includegraphics[width=0.192\hsize]{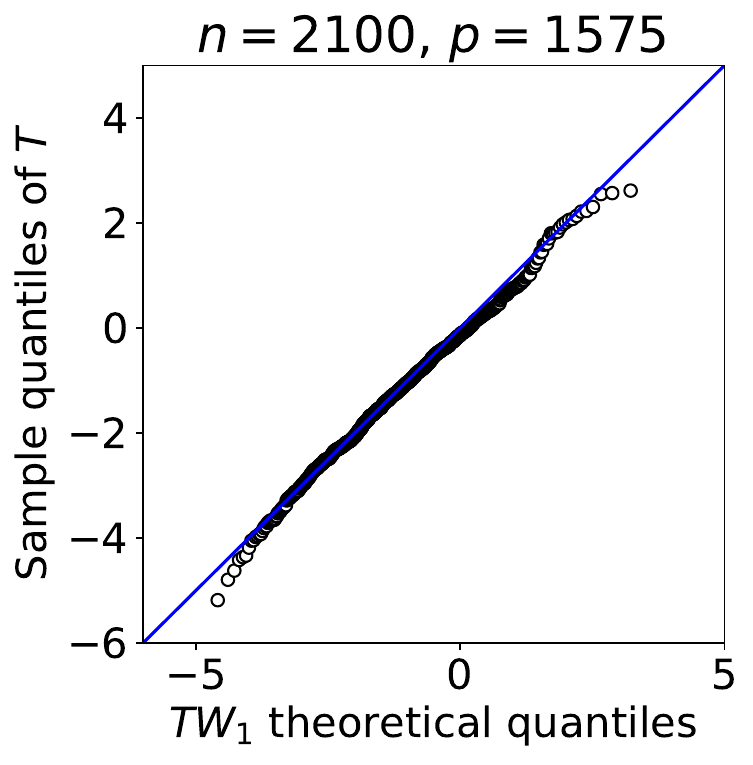}
  \includegraphics[width=0.192\hsize]{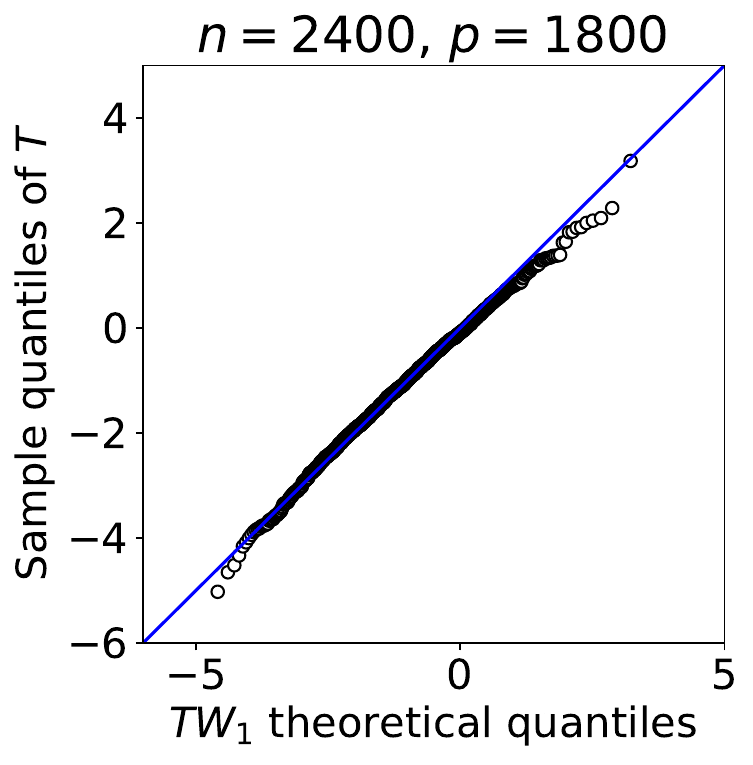}
  \includegraphics[width=0.192\hsize]{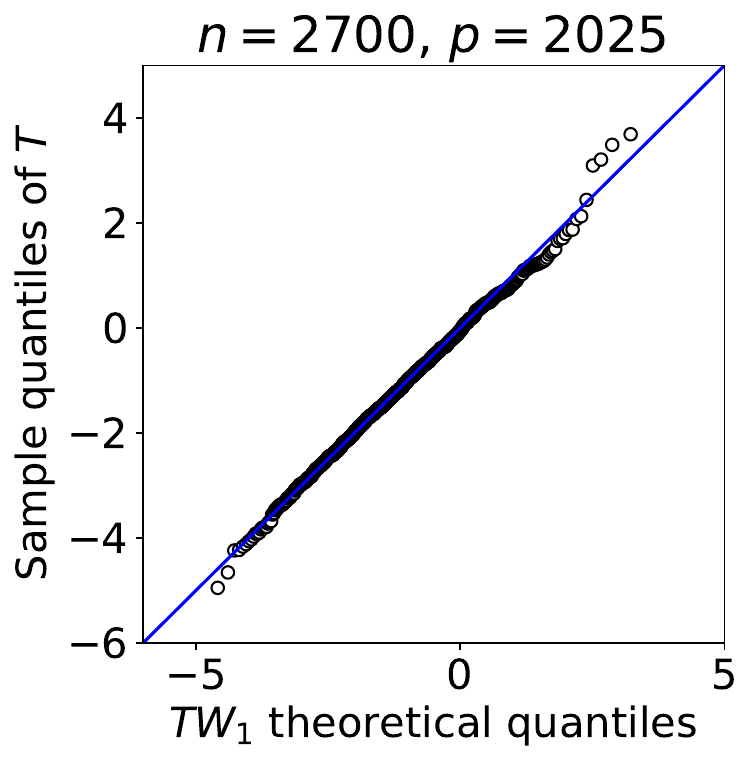}
  \includegraphics[width=0.192\hsize]{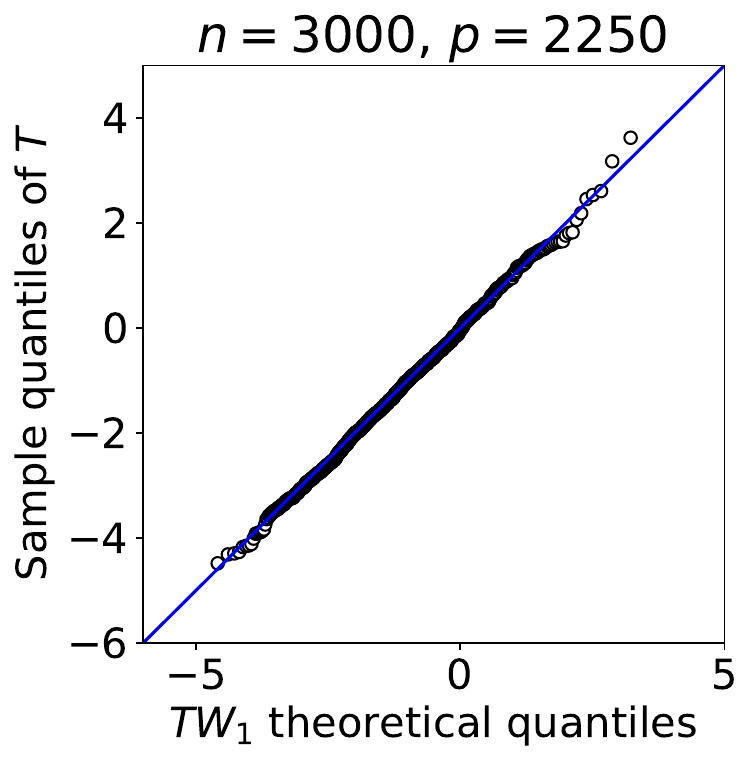}
  \caption{Q-Q plot of test statistic $T$ against the $TW_1$ distribution in the setting of \textbf{Poisson case}. }
  \label{fig:QQ_poisson}
\end{figure}
%-----
\begin{figure}[t]
  \centering
  \includegraphics[width=0.32\hsize]{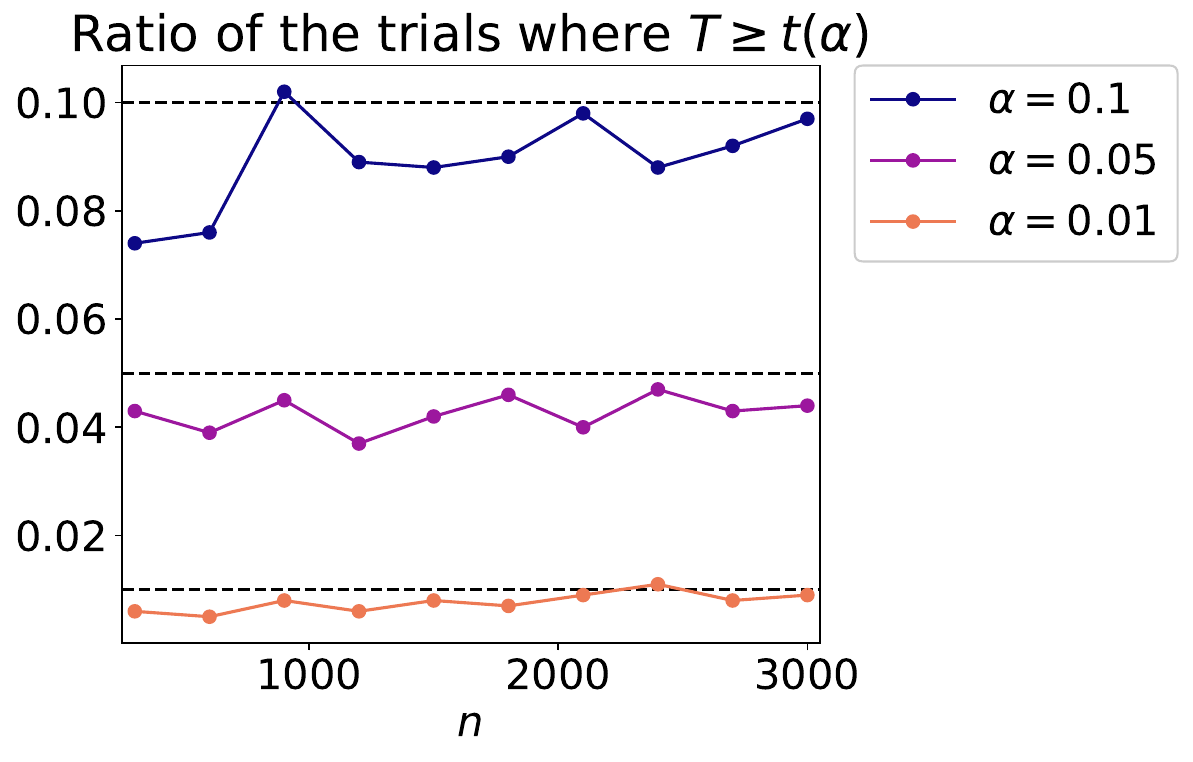}
  \includegraphics[width=0.32\hsize]{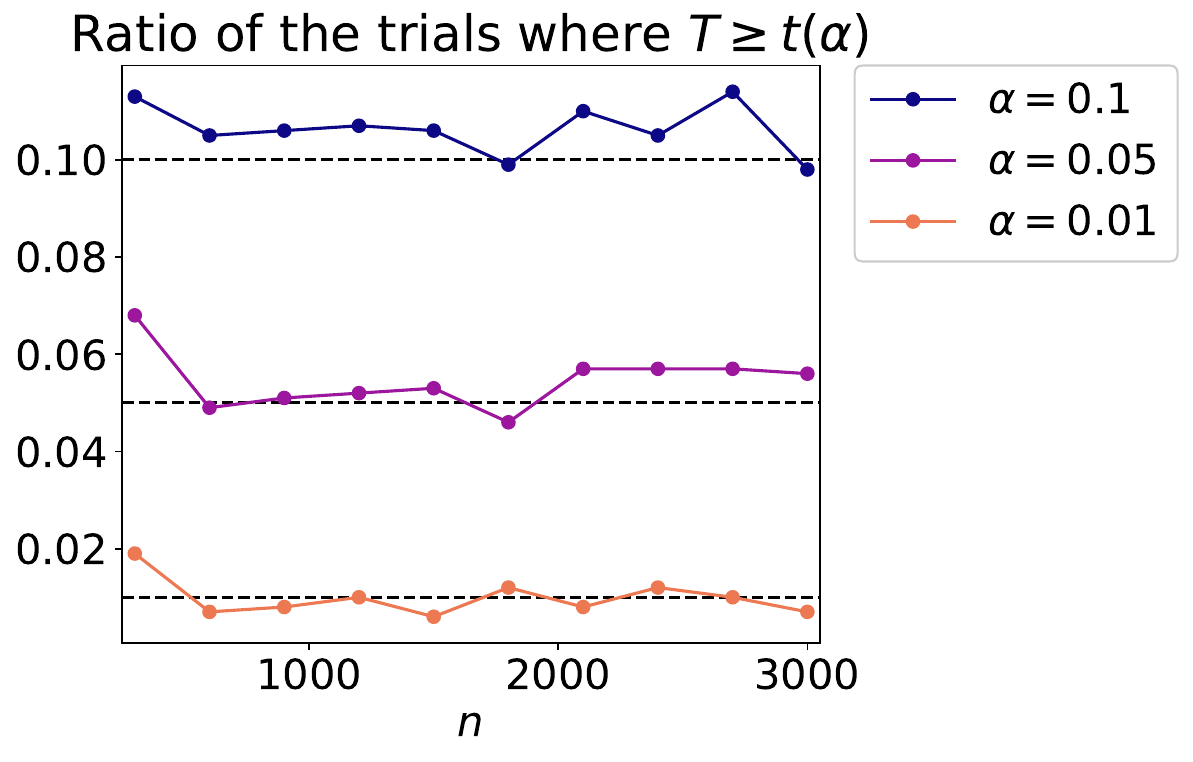}
  \includegraphics[width=0.32\hsize]{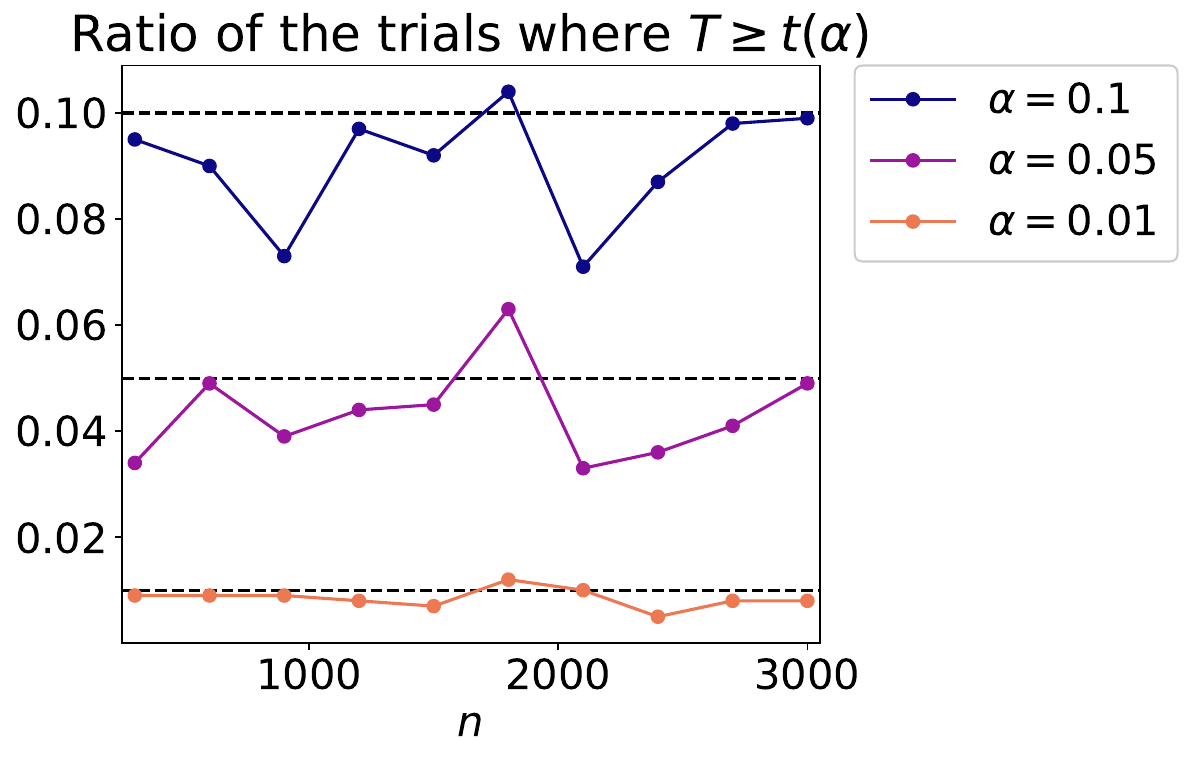}
  \caption{Ratio of the number of trials where test statistic $T \geq t(\alpha)$, where $t(\alpha)$ is the $\alpha$ upper quantile of the $TW_1$ distribution. The left, center, and right figures, respectively, show the results for the settings of Gaussian, Bernoulli, and Poisson distributions. The horizontal line shows the number of rows $n$ in the observed matrix. }\vspace{5mm}
  \label{fig:preliminaryT}
  \centering
  \includegraphics[width=0.32\hsize]{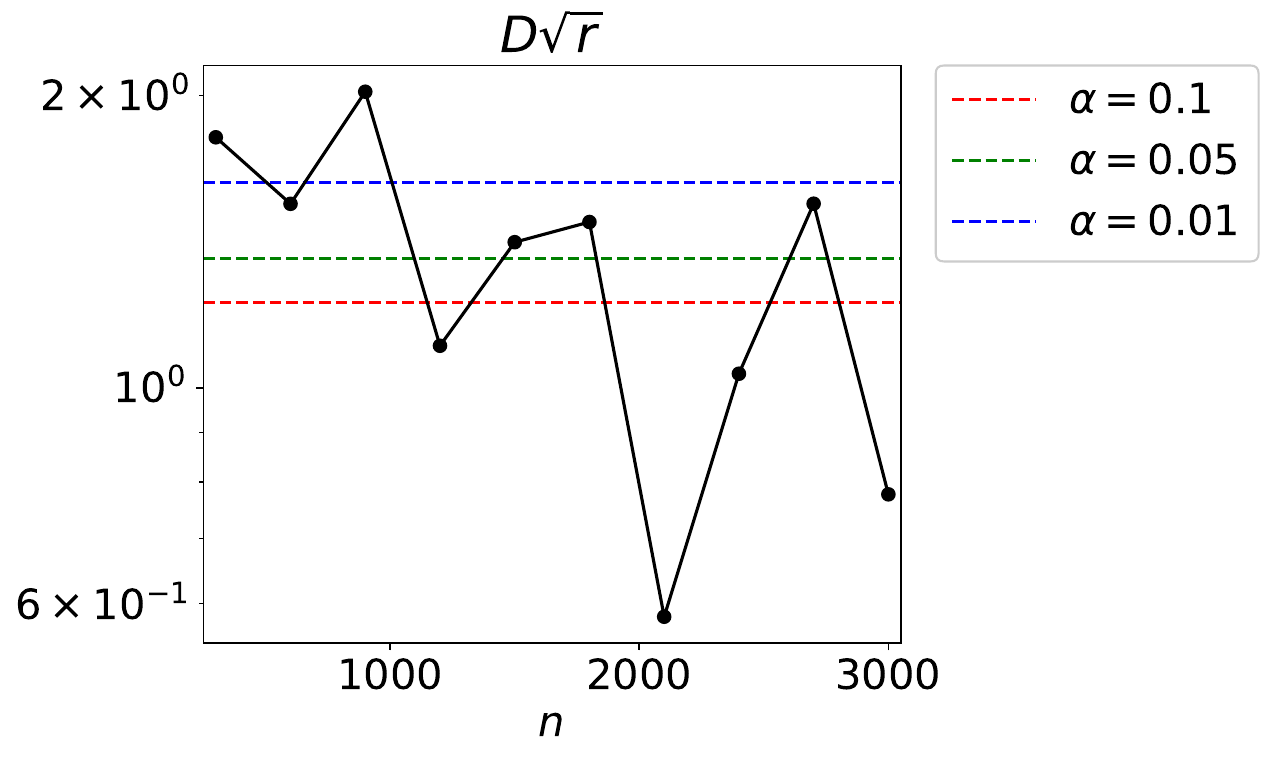}
  \includegraphics[width=0.32\hsize]{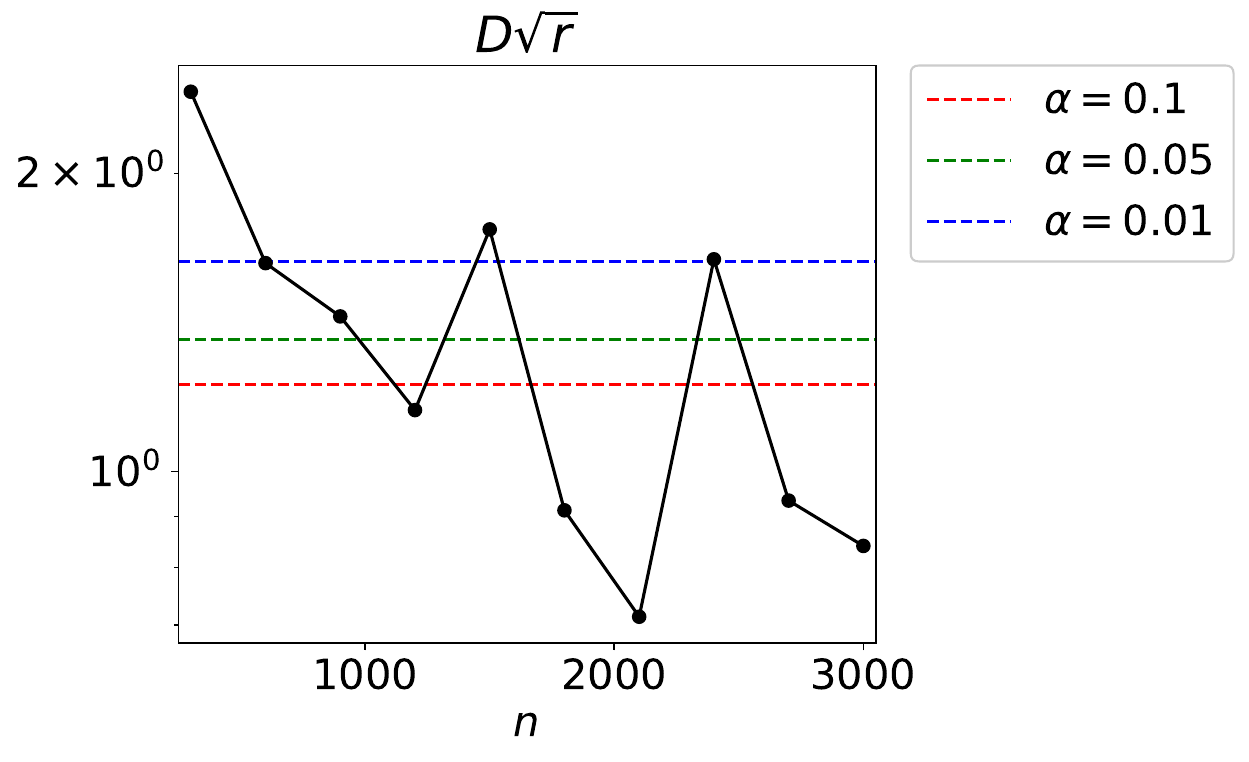}
  \includegraphics[width=0.32\hsize]{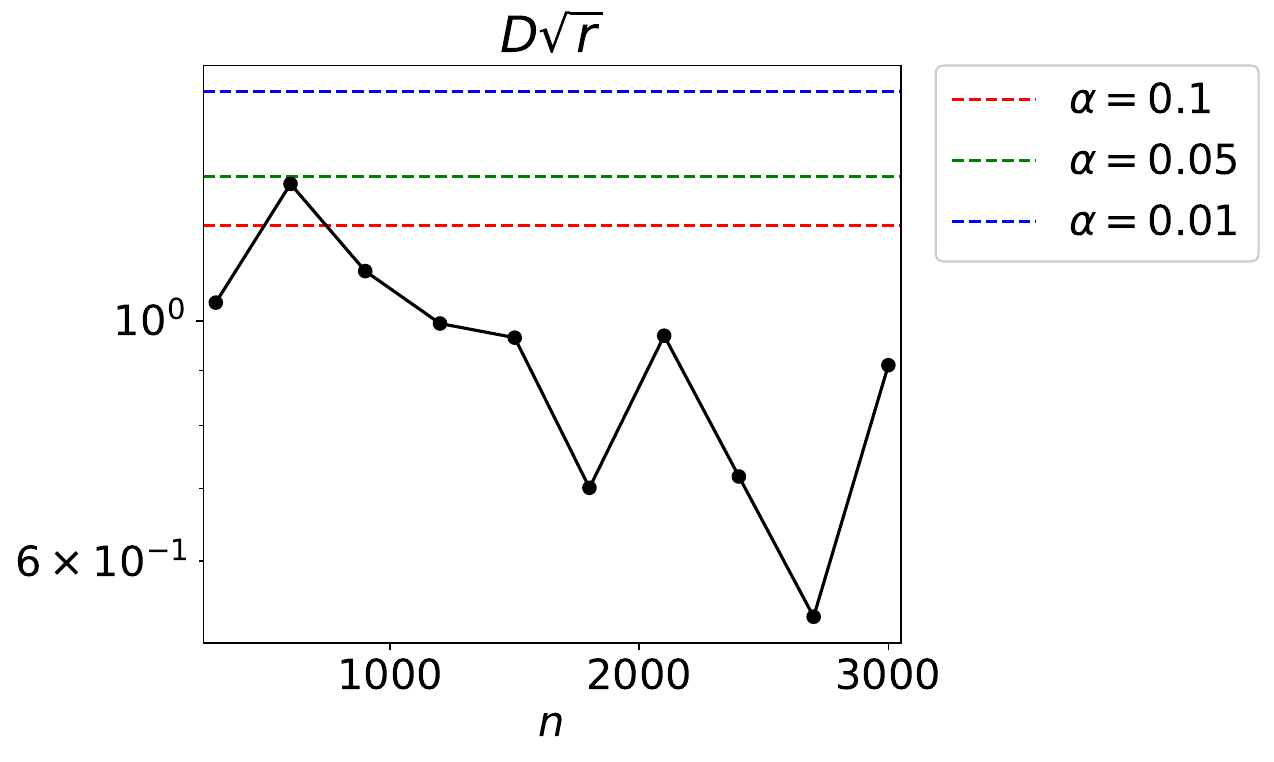}
  \caption{Test statistics $D\sqrt{r}$ of the Kolmogorov-Smirnov test \cite{Conover1999} for the test statistic $T$. The left, center, and right figures, respectively, show the results for the settings of Gaussian, Bernoulli, and Poisson distributions. If the test statistic $D\sqrt{r}$ is larger than the significance level $\alpha$, then the null hypothesis that $T$ follows the $TW_1$ distribution is rejected, and if otherwise, the null hypothesis is accepted. }
  \label{fig:KStest}
\end{figure}
%-----

We also plotted the results of the Kolmogorov-Smirnov test \cite{Conover1999} for the test statistic $T$ in Figure \ref{fig:KStest}. We tested whether the distribution of $T$ is the $TW_1$ distribution or not based on the test statistic $D\sqrt{r}$, where $D$ is the maximum absolute difference between the empirical distribution function of $T$ and the cumulative distribution function of the $TW_1$ distribution, and $r$ is the sample size, which is set at $1000$ in this experiment. 
Figure \ref{fig:KStest} shows the convergence of the proposed test statistic $T$ in law to the $TW_1$ distribution under the realizable setting. 

\subsection{Unrealizable case: asymptotic behavior of test statistic $T$}
\label{sec:exp_unrealizable}

Next, we checked the asymptotic behavior of the proposed test statistic $T$ under the \textit{unrealizable} setting, which has been stated in Theorems \ref{th:unrealizable_lower} and \ref{th:unrealizable_upper}, by using synthetic data that were generated based on the same three types of distributions as in Section \ref{sec:exp_realizable}. By combining Theorems \ref{th:unrealizable_lower} and \ref{th:unrealizable_upper}, we obtain the following theorem: 
\begin{theorem}[Unrealizable case, two-sided bound]
Suppose $K_0 < K$ or $H_0 < H$. Then, 
\begin{eqnarray}
% [Revision 2020/9] <---
\forall \epsilon>0, \ \exists C_1>0, C_2>0, M>0,\ \forall m \geq M, \nonumber \\
\mathrm{Pr} (C_1 m^{\frac{5}{3}} \leq T \leq C_2 m^{\frac{5}{3}}) \geq 1-\epsilon. 
% [Revision 2020/9] --->
\end{eqnarray}
\label{th:unrealizable_two_sided}
\end{theorem}
In other words, with high probability, the proposed test statistic $T$ increases in proportion to $p^{\frac{5}{3}}$ in the limit of $p \to \infty$, since we have assumed that $p \propto m$. 

With respect to the null models and parameters, we used the same settings as in Section \ref{sec:exp_realizable} for all of the three distributional settings (i.e., Gaussian, Bernoulli, and Poisson LBMs). 
Based on such settings, we randomly generated $100$ observed matrices, estimated their block structures based on the Ward's hierarchical clustering algorithm \cite{Ward1963}, and computed the test statistic $T$. With respect to the matrix size, we tried the following $10$ settings: $(n, p) = (200 \times i, 150 \times i)$, $i = 1, \dots, 10$. When generating an observed matrix, the null cluster of each row was randomly chosen from the discrete uniform distribution on $\{1, 2, 3, 4\}$. Similarly, the null cluster of each column was randomly chosen from the discrete uniform distribution on $\{1, 2, 3\}$. 

Figures \ref{fig:unrealizableT} and \ref{fig:unrealizableT2} show the asymptotic behavior of the proposed test statistic $T$ under the unrealizable setting. As shown in Theorem \ref{th:unrealizable_two_sided}, we see that $T$ increases in proportion to $m^{\frac{5}{3}}$, where $n, p \propto m$. 

%-----
\begin{figure}[t]
  \centering
  \includegraphics[width=0.32\hsize]{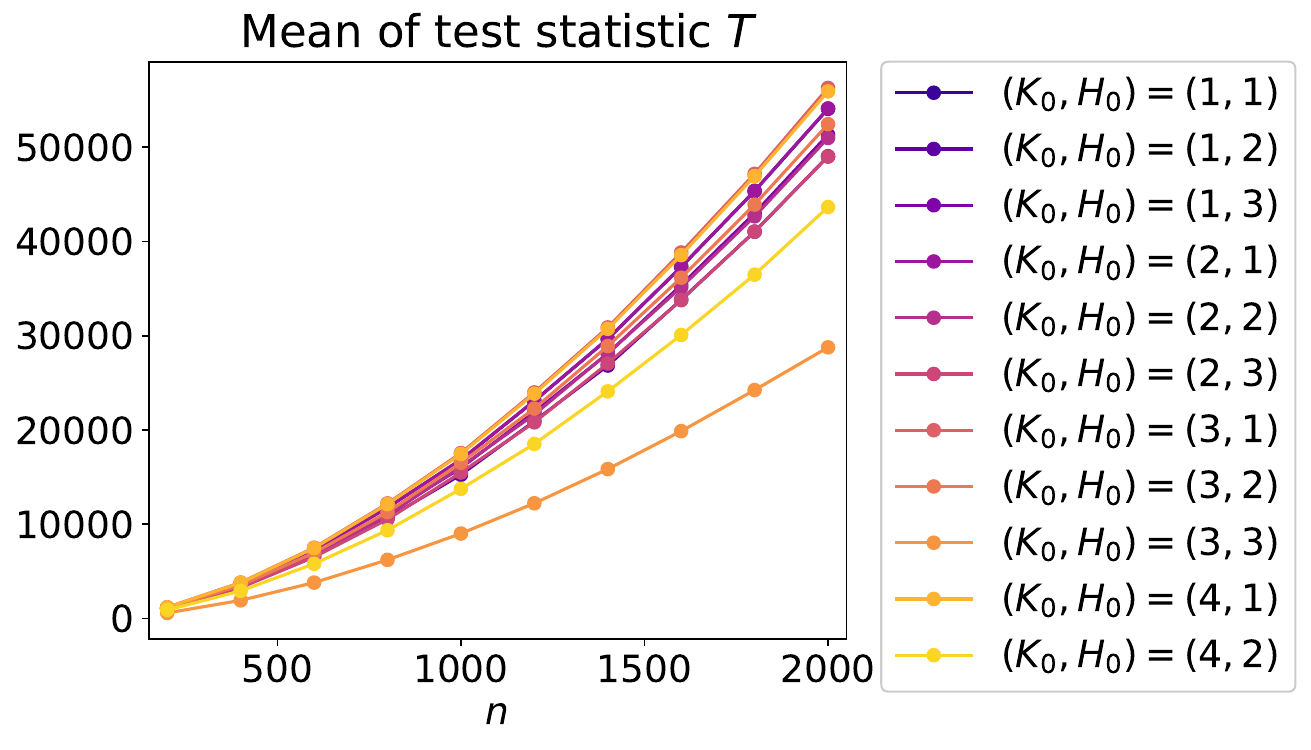}
  \includegraphics[width=0.32\hsize]{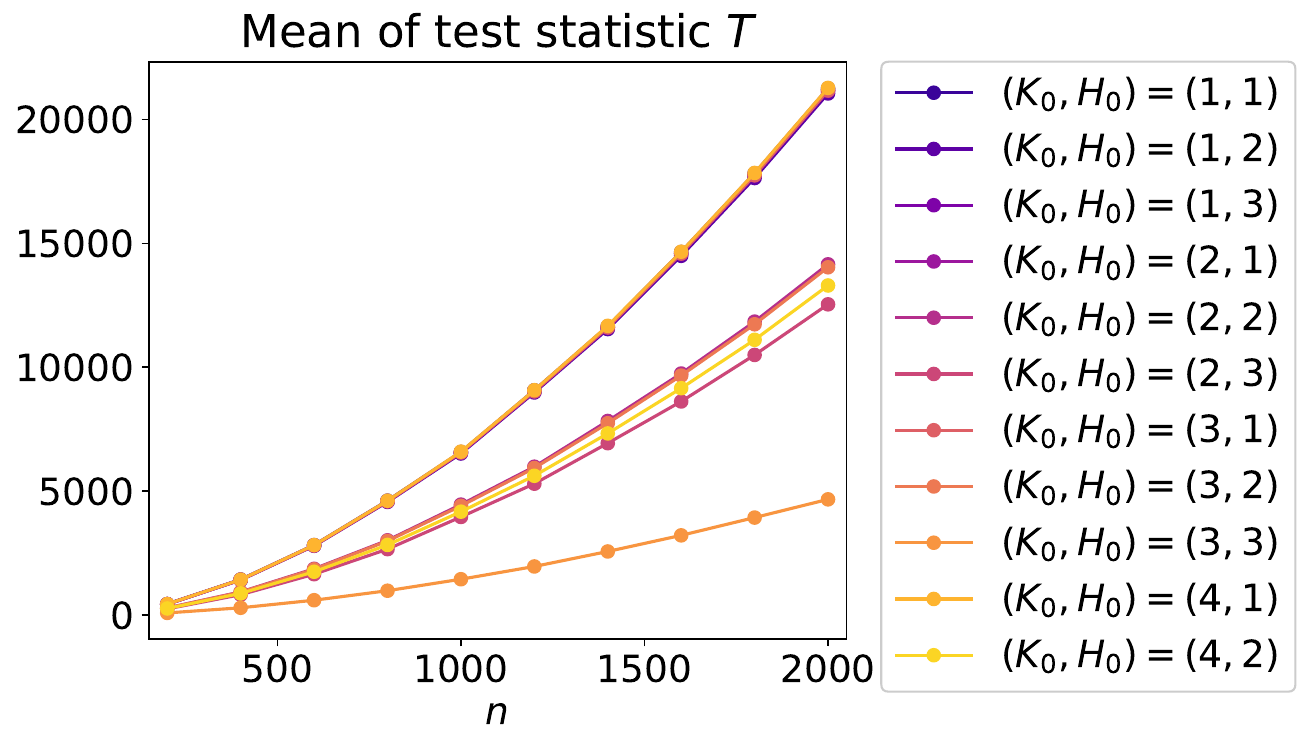}
  \includegraphics[width=0.32\hsize]{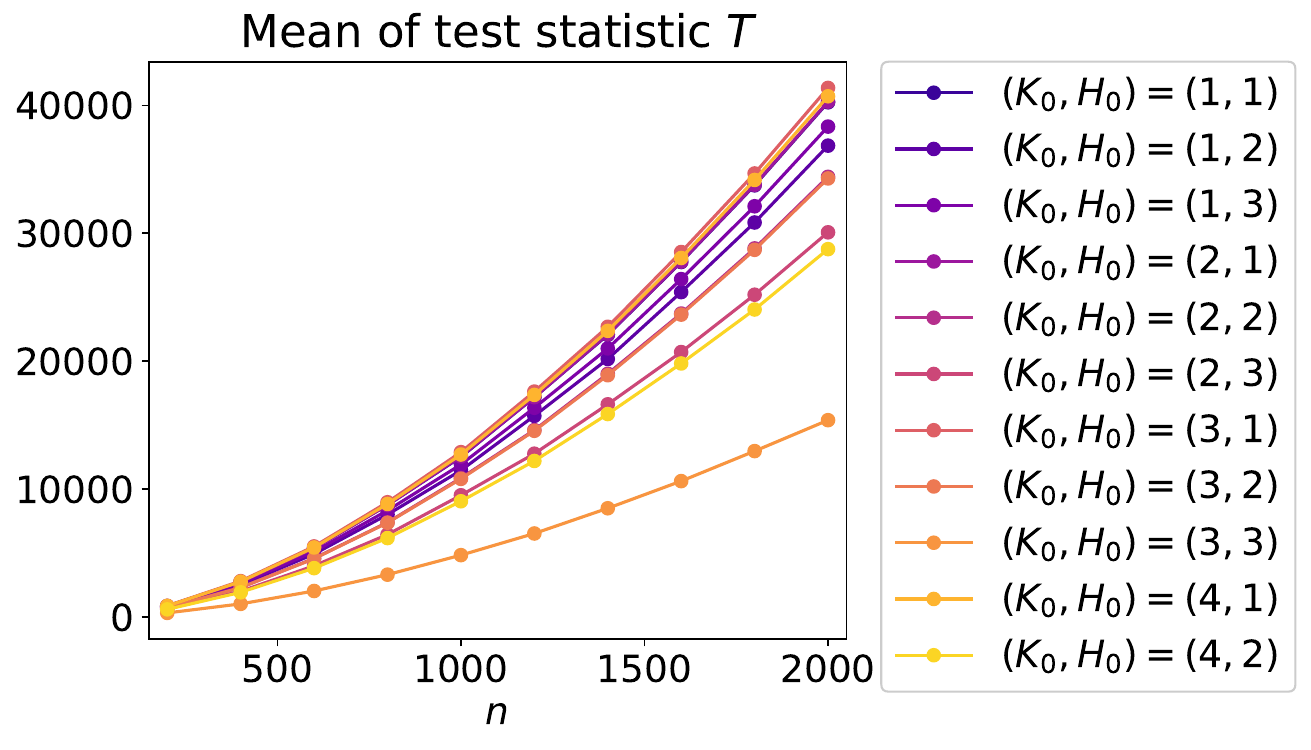}
  \caption{Mean test statistic $T$ in the unrealizable case for $100$ trials. The null row and column cluster numbers are $4$ and $3$, respectively. The left, center, and right figures, respectively, show the results for the settings of Gaussian, Bernoulli, and Poisson distributions. The horizontal line shows the number of rows $n$ in the observed matrix. }\vspace{5mm}
  \label{fig:unrealizableT}
  \includegraphics[width=0.32\hsize]{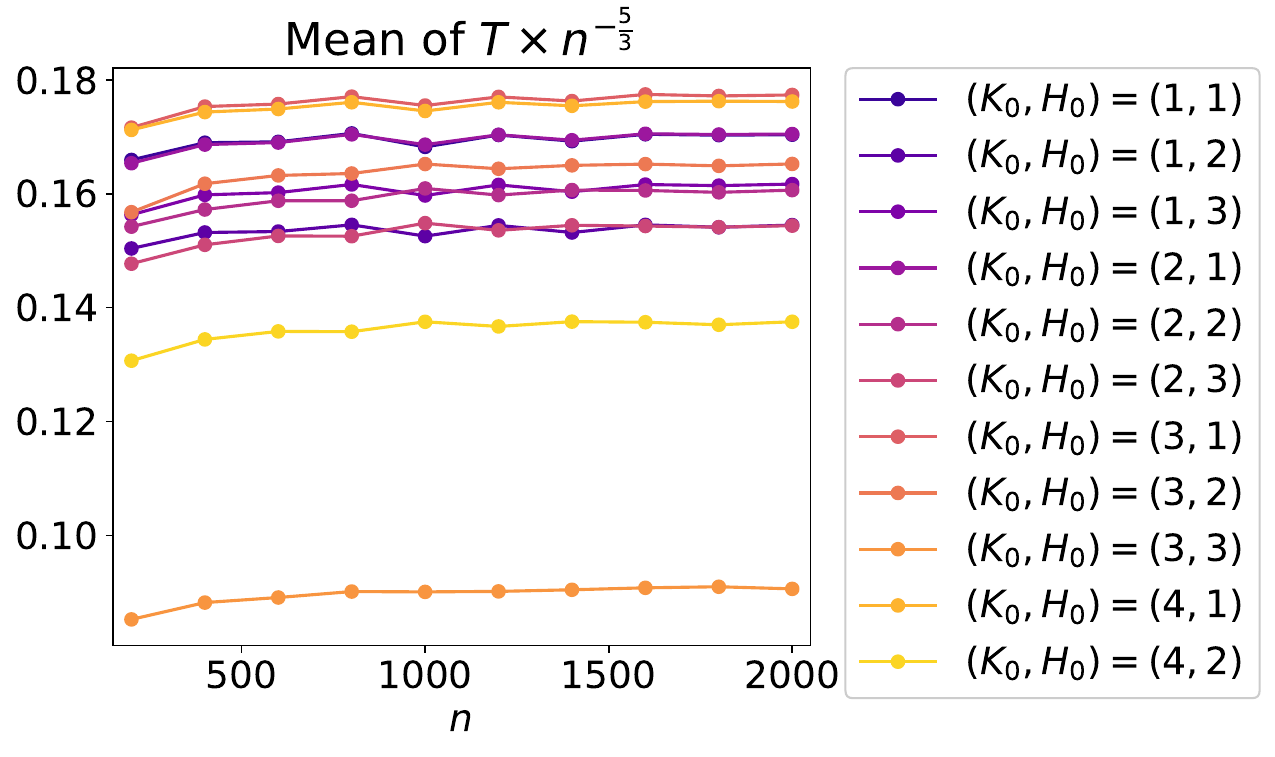}
  \includegraphics[width=0.32\hsize]{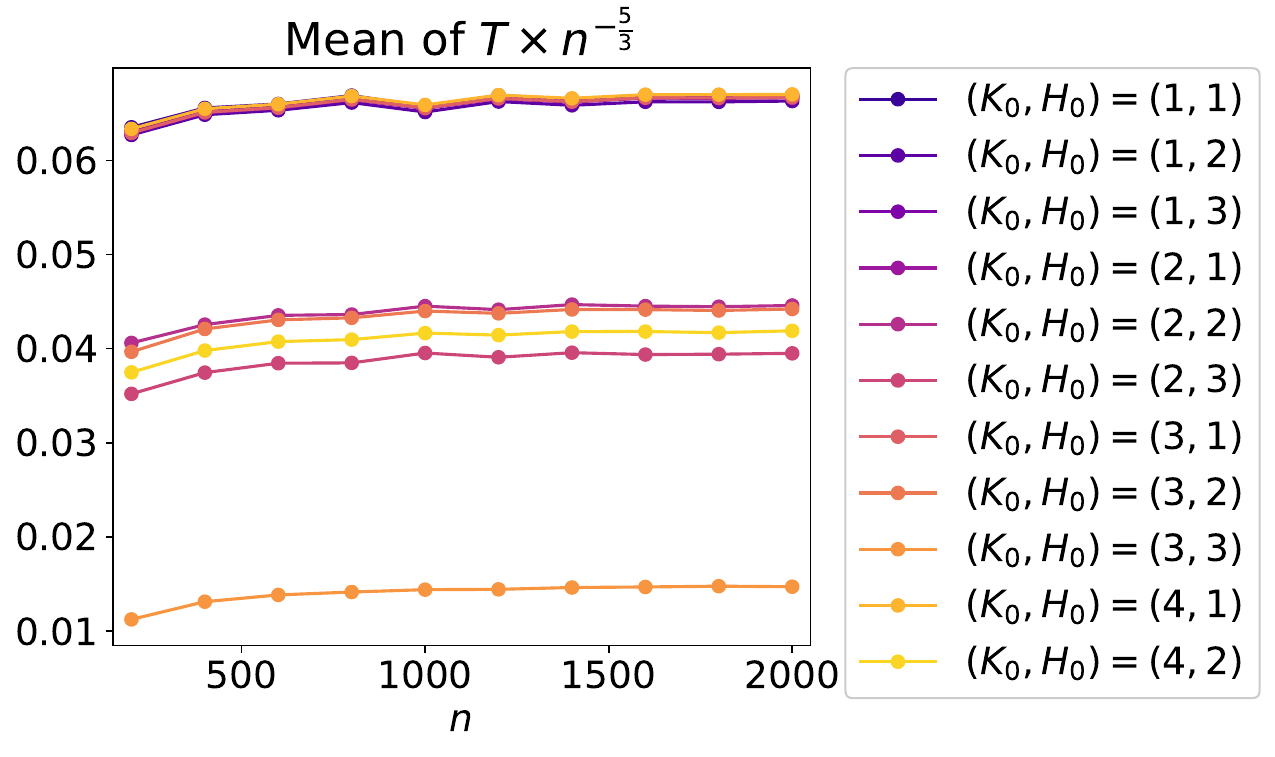}
  \includegraphics[width=0.32\hsize]{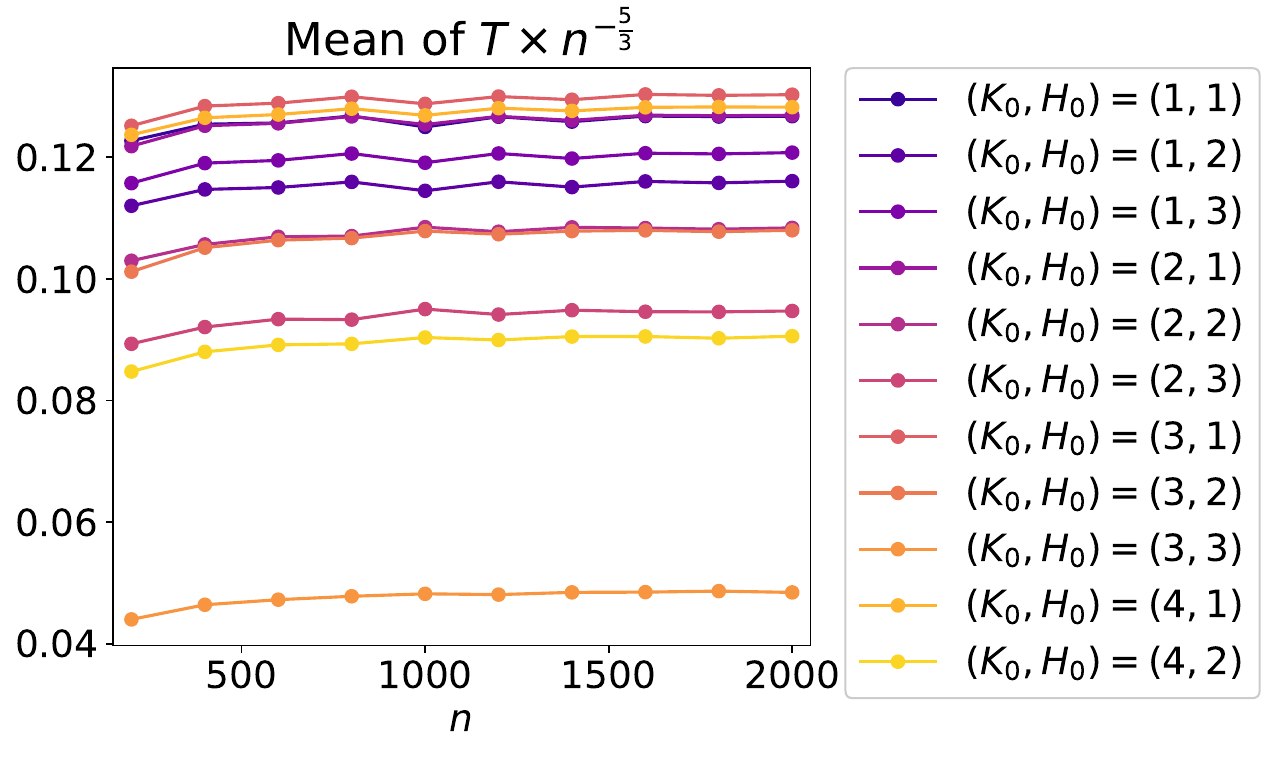}
  \caption{Mean test statistic $T$ divided by $n^{\frac{5}{3}}$ in the unrealizable case for $100$ trials. The left, center, and right figures, respectively, show the results for the settings of Gaussian, Bernoulli, and Poisson distributions. }
  \label{fig:unrealizableT2}
\end{figure}
%-----

\subsection{Accuracy of the proposed goodness-of-fit test}
\label{sec:accuracy}

Finally, we evaluated the proposed goodness-of-fit test in terms of its accuracy. By using synthetic data that were generated based on the same three types of distributions as in Section \ref{sec:exp_realizable}, we checked the ratio of trials where the selected set of cluster numbers $(K_0, H_0)$ is equal to the null one $(K, H)$. Here, we set the null set of cluster numbers at $(K, H) = (4, 3)$. For each distributional setting (i.e., Gaussian, Bernoulli, and Poisson LBMs), we tried $10$ settings with respect to the block-wise mean $B$. The concrete settings were as follows: 

\begin{itemize}
\item \textbf{Gaussian Latent Block Model}: We used the following parameters: 
\begin{align}
&B' = 
\begin{pmatrix}
0.9 & 0.1 & 0.4  \\
0.2 & 0.7 & 0.3  \\
0.3 & 0.2 & 0.8  \\
0.6 & 0.9 & 0.1  \\
\end{pmatrix}, \ \ \ 
S = 
\begin{pmatrix}
0.08 & 0.06 & 0.15  \\
0.14 & 0.12 & 0.07  \\
0.09 & 0.1 & 0.11  \\
0.16 & 0.13 & 0.05  \\
\end{pmatrix}, \nonumber \\
&\forall k, h,\ B_{kh} = \left( 1 - \frac{t}{10} \right) (B'_{kh} - 0.5) + 0.5, \ \ \ \mathrm{for}\ t = 0, \dots, 9, 
\end{align}
\item \textbf{Bernoulli Latent Block Model} We used the following parameters: 
\begin{align}
&B' = 
\begin{pmatrix}
0.9 & 0.1 & 0.4  \\
0.2 & 0.7 & 0.3  \\
0.3 & 0.2 & 0.8  \\
0.6 & 0.9 & 0.1  \\
\end{pmatrix}, \nonumber \\
&\forall k, h,\ B_{kh} = \left( 1 - \frac{t}{10} \right) (B'_{kh} - 0.5) + 0.5, \ \ \ \mathrm{for}\ t = 0, \dots, 9. 
\end{align}
\item \textbf{Poisson Latent Block Model} We used the following parameters: 
\begin{align}
&B' = 
\begin{pmatrix}
9.0 & 1.0 & 4.0  \\
2.0 & 7.0 & 3.0  \\
3.0 & 2.0 & 8.0  \\
6.0 & 9.0 & 1.0  \\
\end{pmatrix}, \nonumber \\
&\forall k, h,\ B_{kh} = \left( 1 - \frac{t}{10} \right) (B'_{kh} - 5) + 5, \ \ \ \mathrm{for}\ t = 0, \dots, 9. 
\end{align}
\end{itemize}

With respect to the matrix size, we tried the following $10$ settings for each distributional setting and for each setting of $B$: $(n, p) = (40 \times i, 30 \times i)$, $i = 1, \dots, 10$. When generating an observed matrix, the null cluster of each row was randomly chosen from the discrete uniform distribution on $\{1, 2, 3, 4\}$. Similarly, the null cluster of each column was randomly chosen from the discrete uniform distribution on $\{1, 2, 3\}$. 
In each of $3$ (Gaussian, Bernoulli, or Poisson LBM) $\times 10$ (for the setting of $B$) $\times 10$ (for the setting of matrix size) settings, we generated $1000$ observed matrices and applied the proposed sequential goodness-of-fit test, until the null hypothesis $(K, H) = (K_0, H_0)$ was accepted. 
For each observed matrix, we estimated their block structures based on the Ward's hierarchical clustering algorithm \cite{Ward1963} under each setting of a hypothetical set of cluster numbers $(K_0, H_0)$, computed the test statistic $T$, and performed the proposed test for the given cluster numbers $(K_0, H_0)$ using a significance level of $\alpha = 0.01$. 
Figures \ref{fig:As_normal}, \ref{fig:As_ber}, \ref{fig:As_pois}, respectively, show the examples of generated observed matrices of Gaussian, Bernoulli, and Poisson LBMs. 

%-----
\begin{figure}[t]
  \centering
  \includegraphics[height=0.19\hsize]{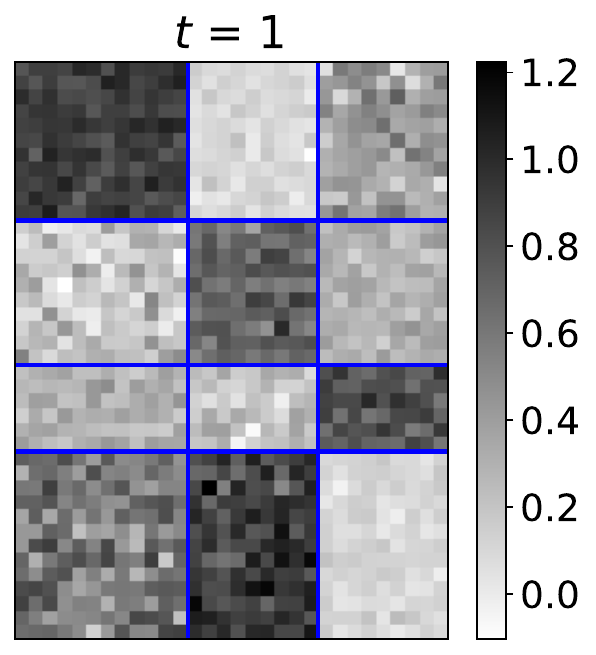}
  \includegraphics[height=0.19\hsize]{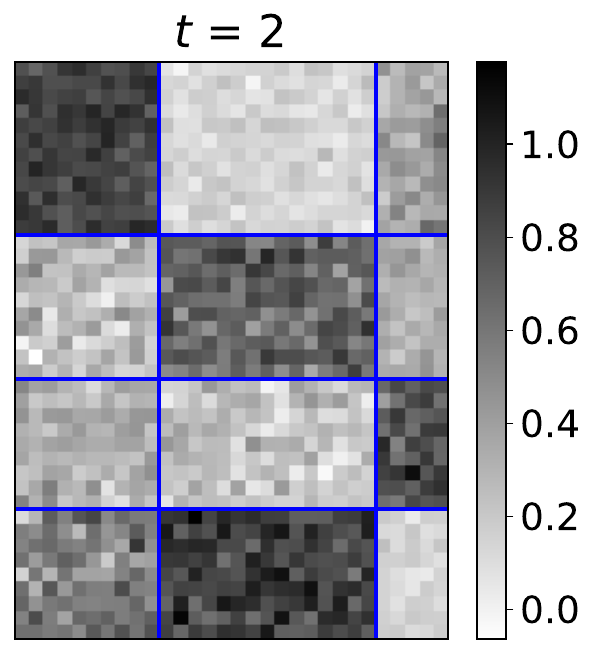}
  \includegraphics[height=0.19\hsize]{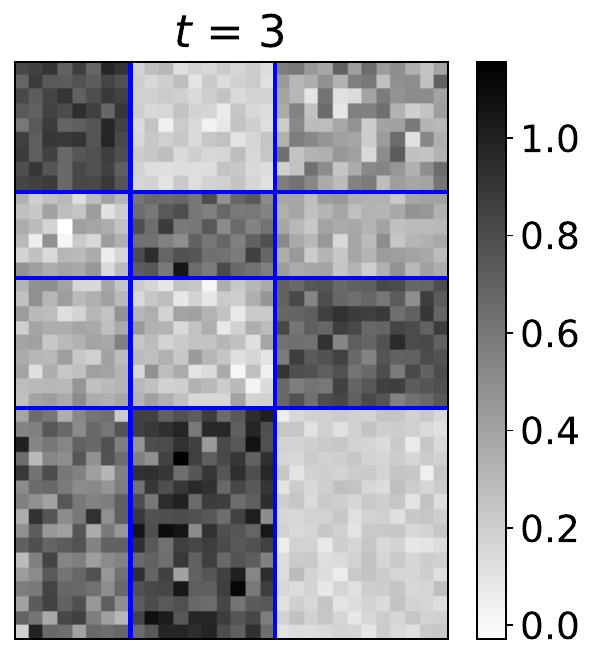}
  \includegraphics[height=0.19\hsize]{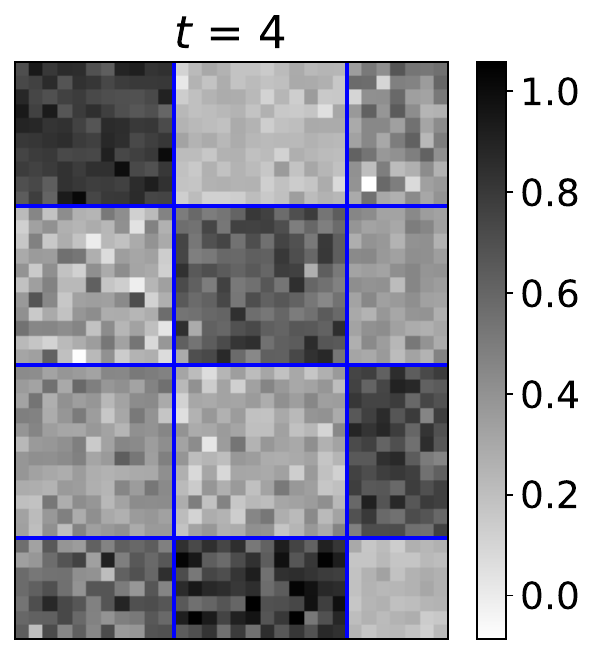}
  \includegraphics[height=0.19\hsize]{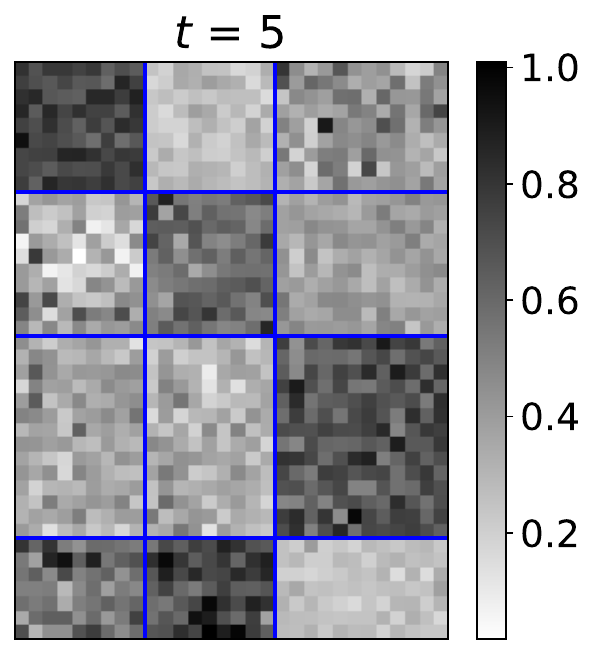}\\
  \includegraphics[height=0.19\hsize]{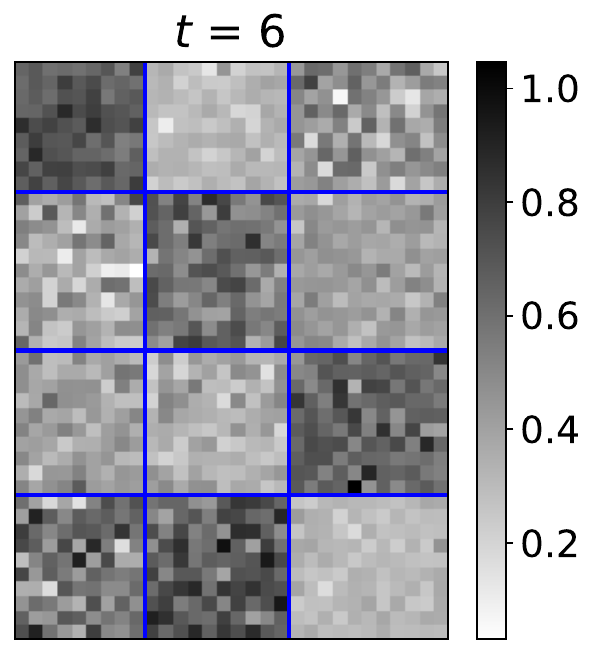}
  \includegraphics[height=0.19\hsize]{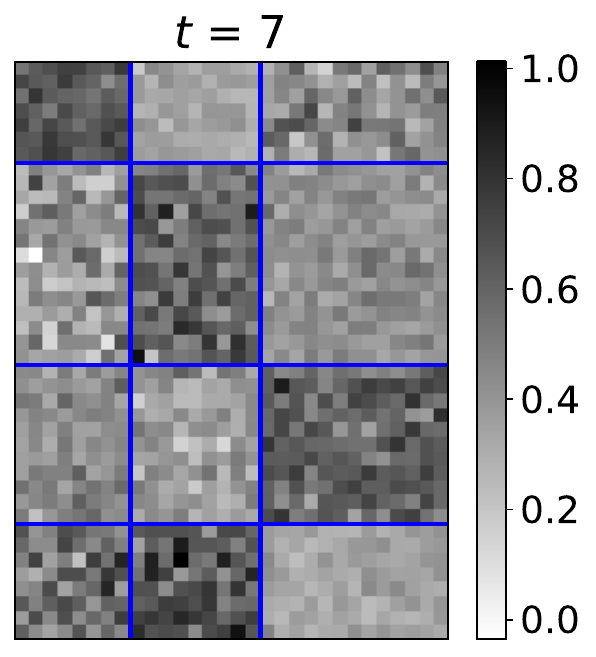}
  \includegraphics[height=0.19\hsize]{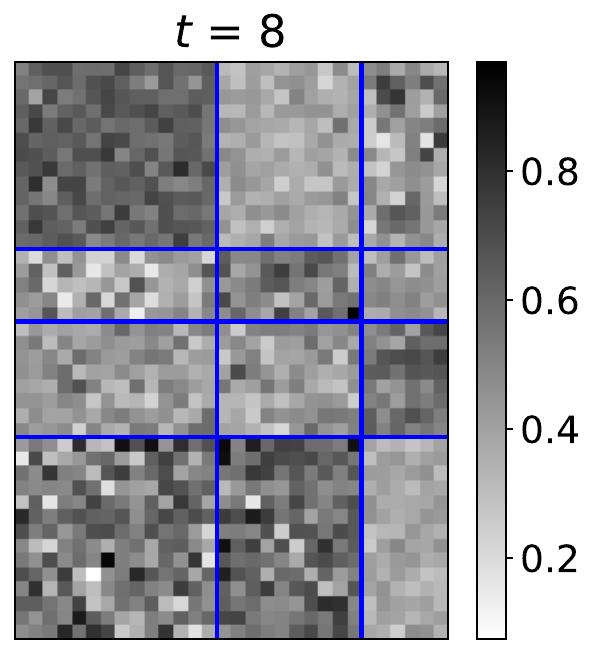}
  \includegraphics[height=0.19\hsize]{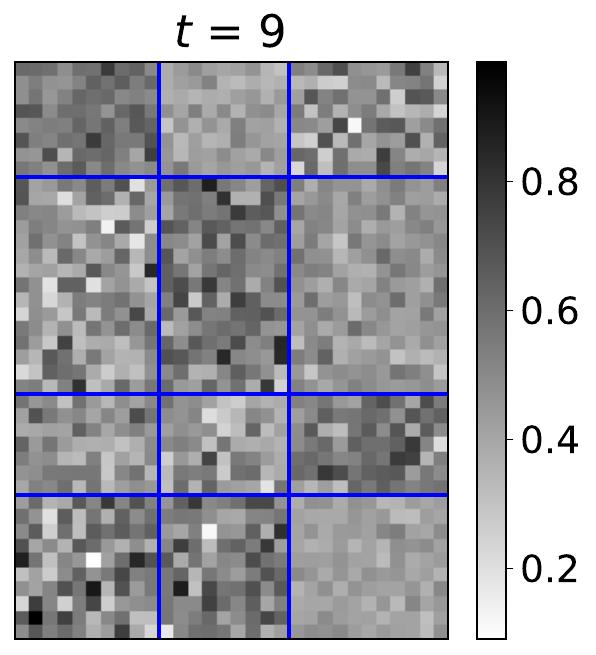}
  \includegraphics[height=0.19\hsize]{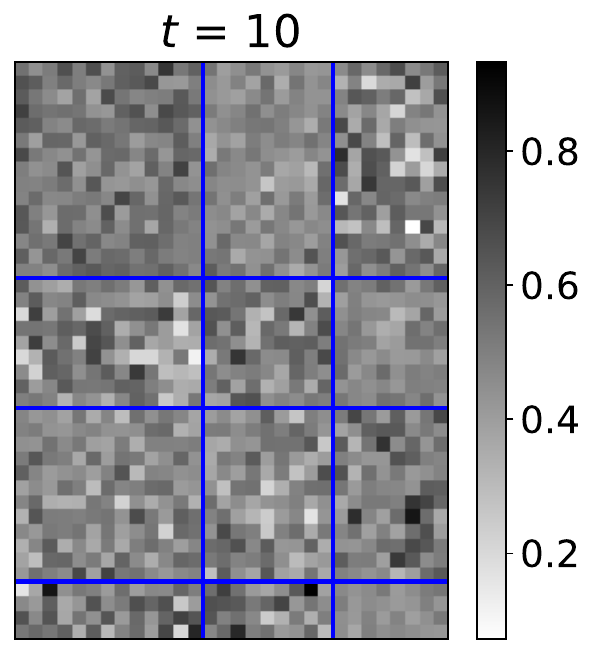}
  \caption{Examples of null block structures of the \textbf{Gaussian} LBM. $40 \times 30$ observed matrices are plotted for $10$ different settings of $B$ ($t = 1, \dots, 10$). The rows and columns of matrix $A$ were sorted according to the null clusters. }\vspace{3mm}
  \label{fig:As_normal}
  \includegraphics[height=0.19\hsize]{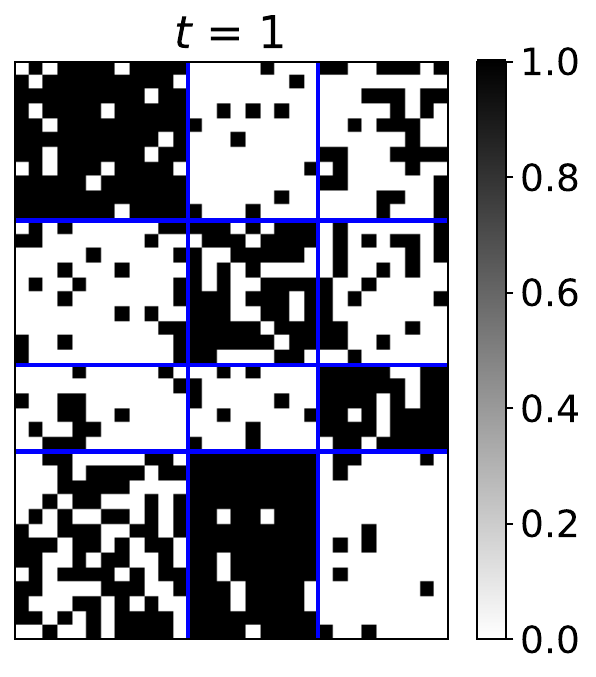}
  \includegraphics[height=0.19\hsize]{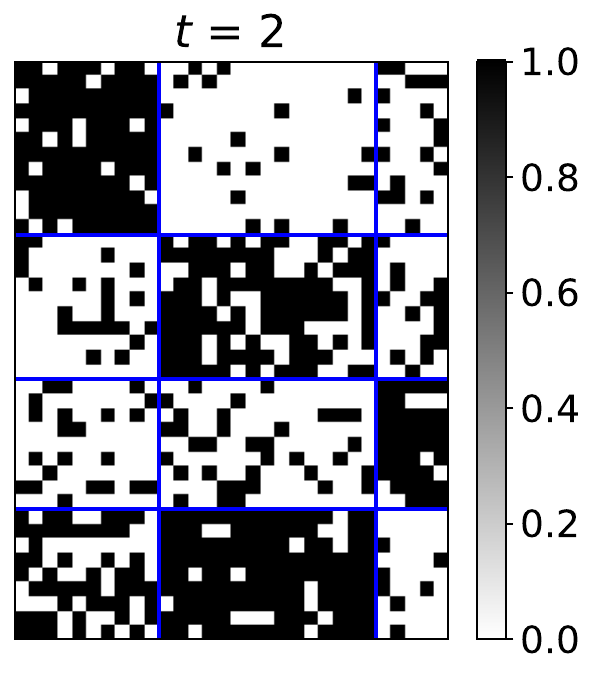}
  \includegraphics[height=0.19\hsize]{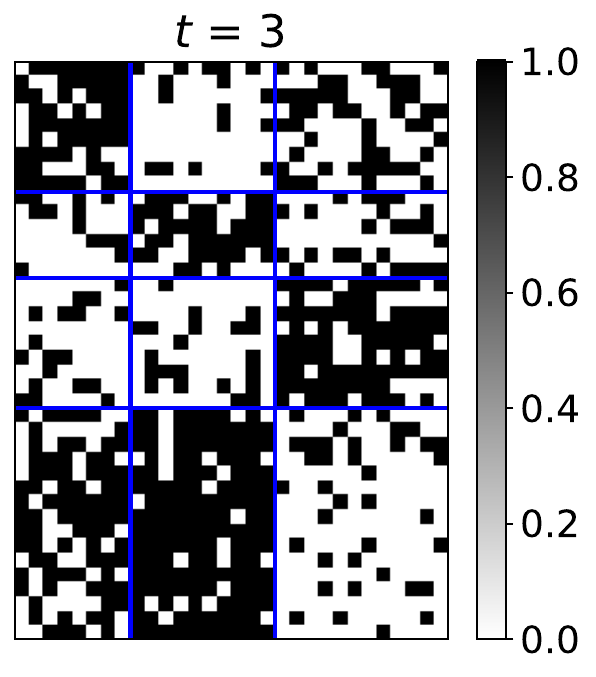}
  \includegraphics[height=0.19\hsize]{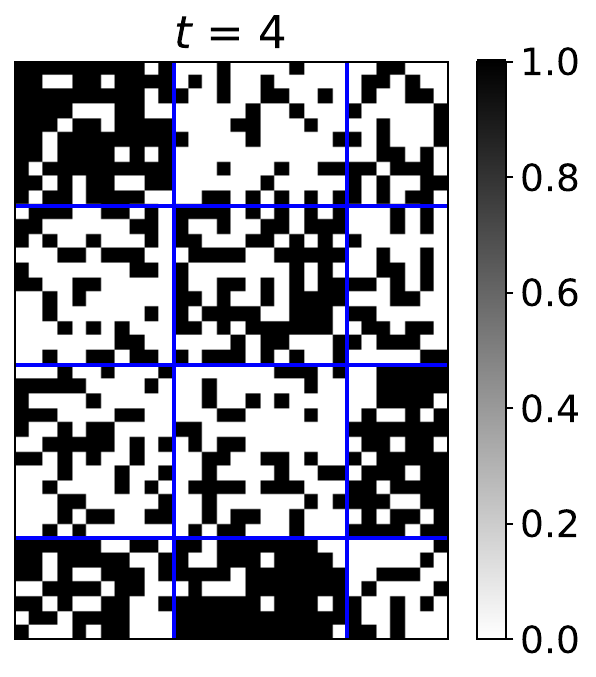}
  \includegraphics[height=0.19\hsize]{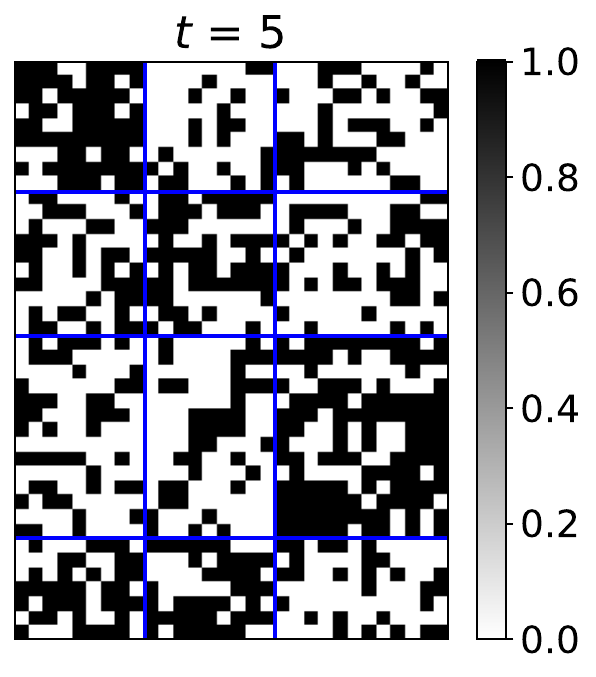}\\
  \includegraphics[height=0.19\hsize]{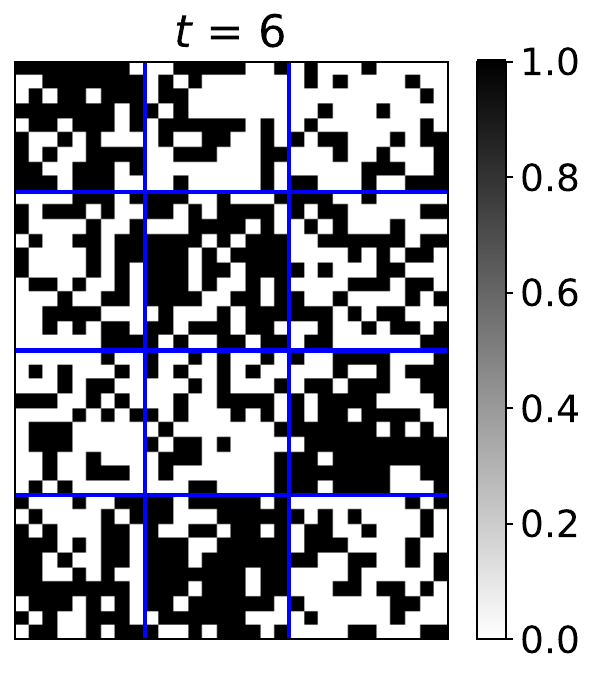}
  \includegraphics[height=0.19\hsize]{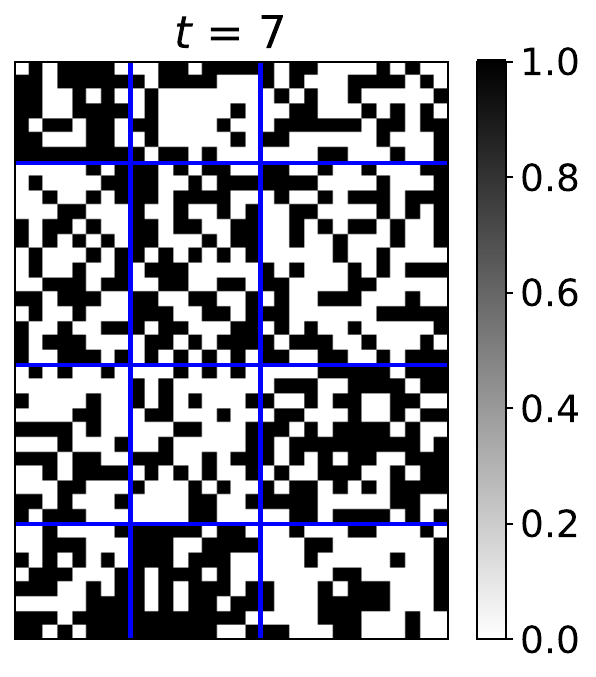}
  \includegraphics[height=0.19\hsize]{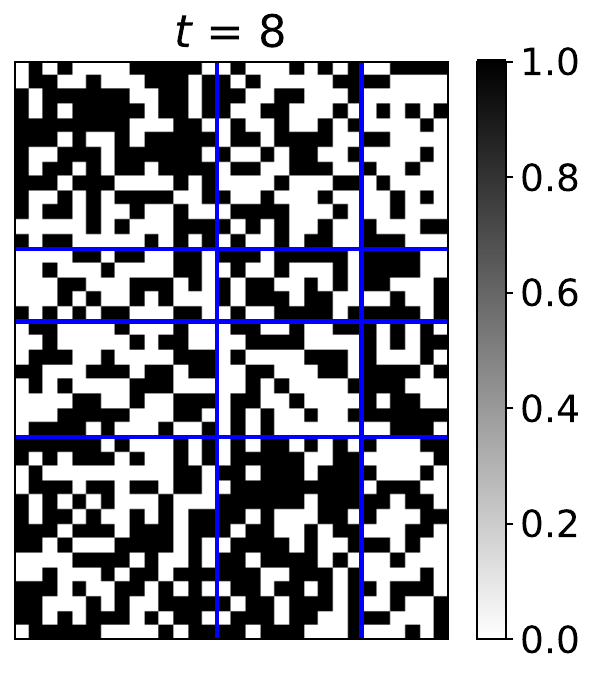}
  \includegraphics[height=0.19\hsize]{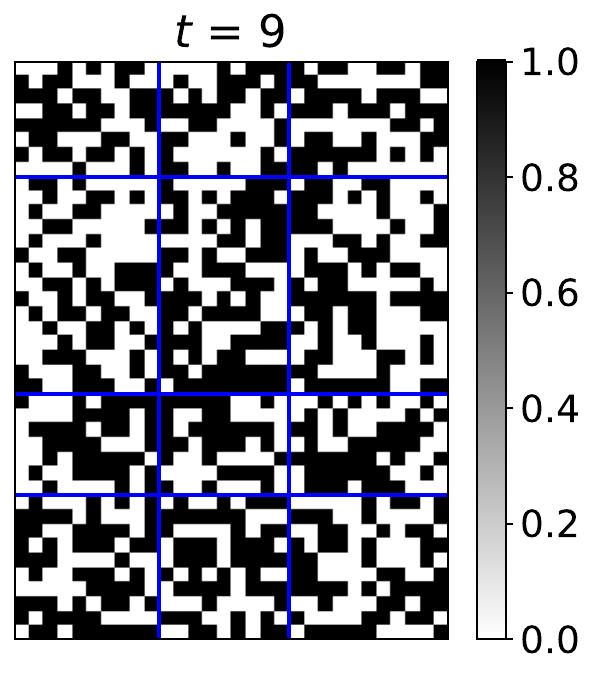}
  \includegraphics[height=0.19\hsize]{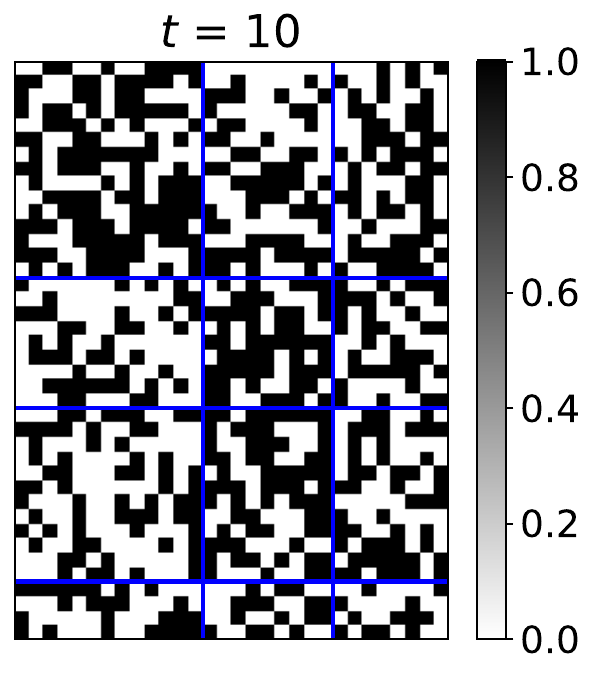}
  \caption{Examples of null block structures of the \textbf{Bernoulli} LBM. }\vspace{3mm}
  \label{fig:As_ber}
  \includegraphics[height=0.19\hsize]{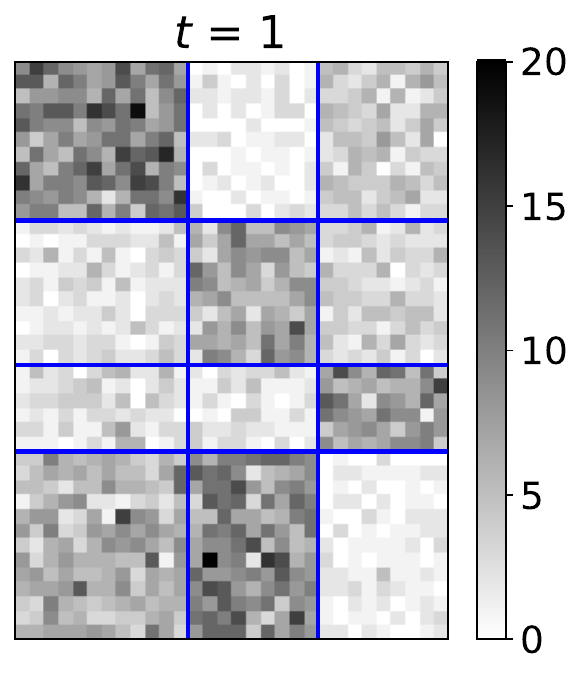}
  \includegraphics[height=0.19\hsize]{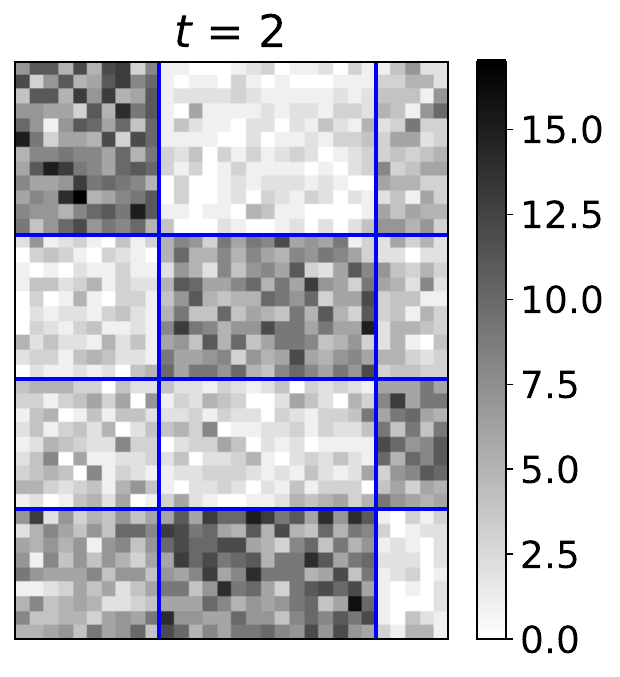}
  \includegraphics[height=0.19\hsize]{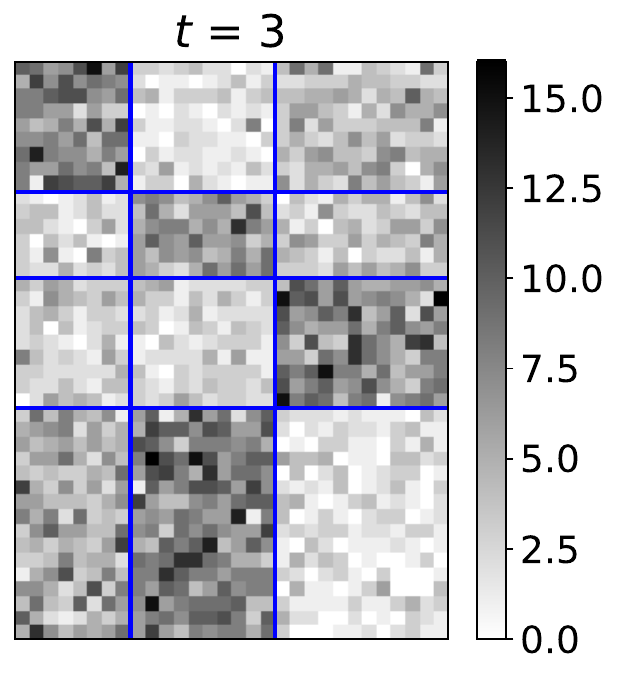}
  \includegraphics[height=0.19\hsize]{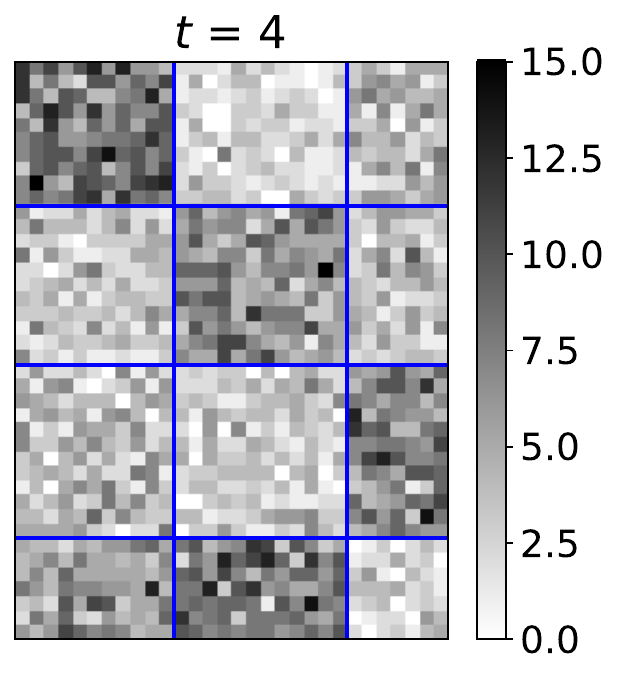}
  \includegraphics[height=0.19\hsize]{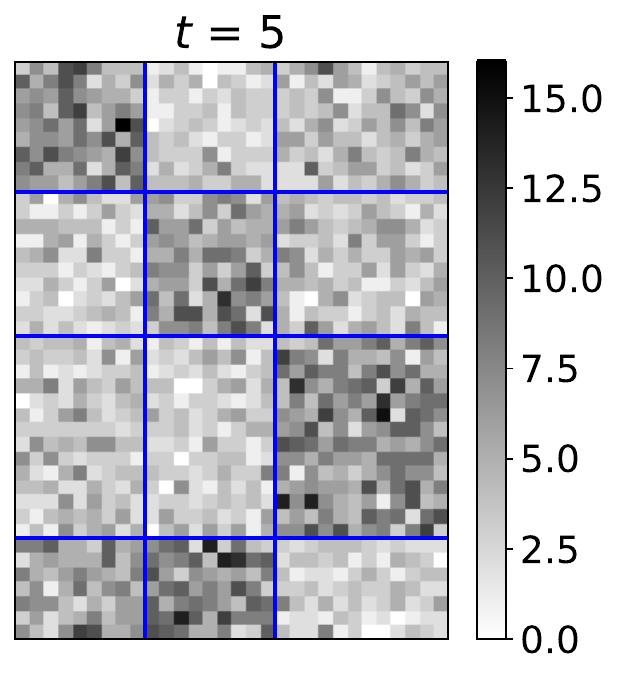}\\
  \includegraphics[height=0.19\hsize]{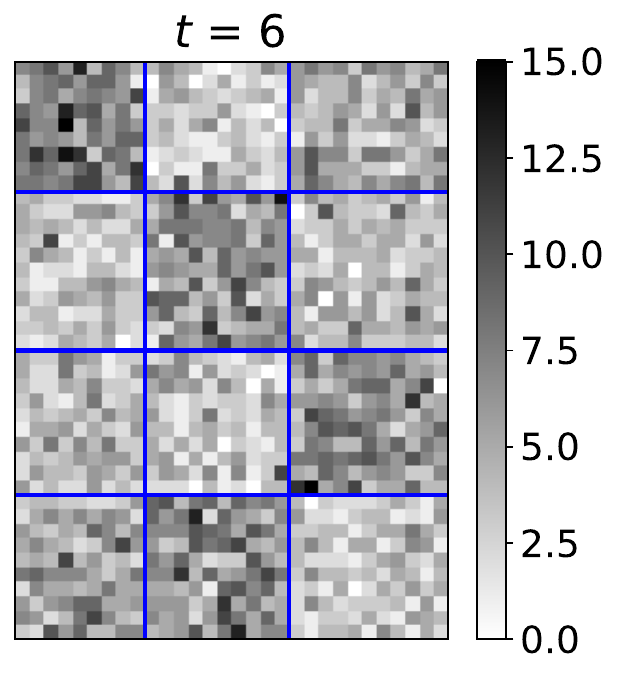}
  \includegraphics[height=0.19\hsize]{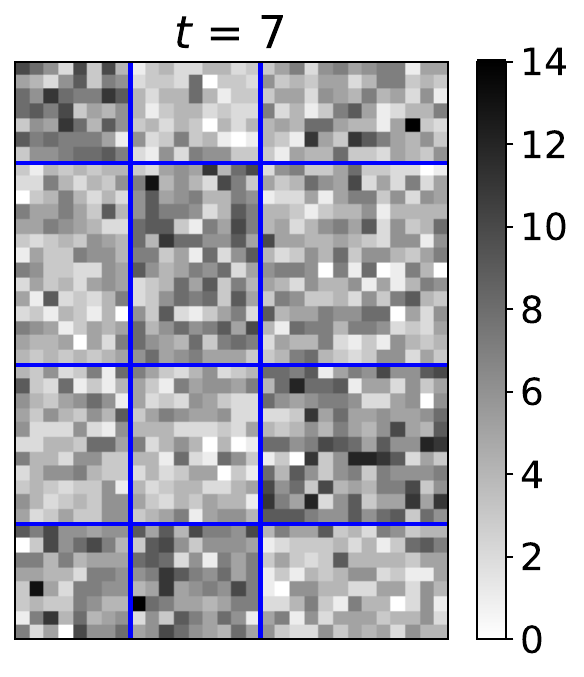}
  \includegraphics[height=0.19\hsize]{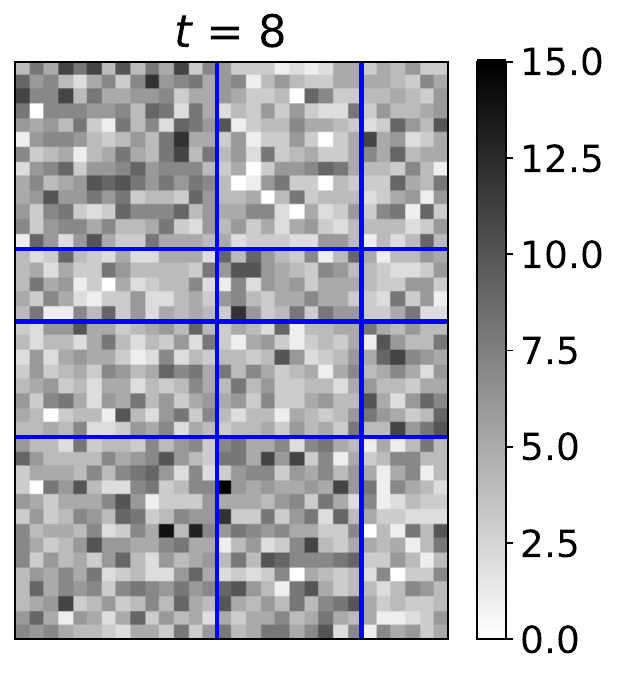}
  \includegraphics[height=0.19\hsize]{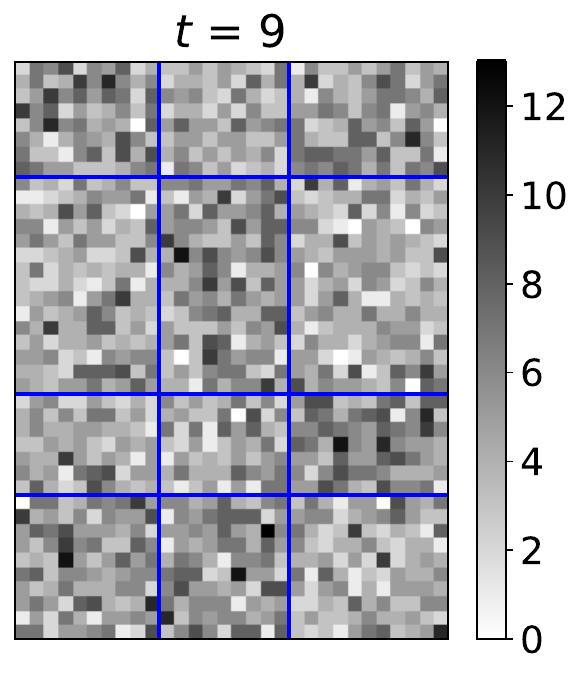}
  \includegraphics[height=0.19\hsize]{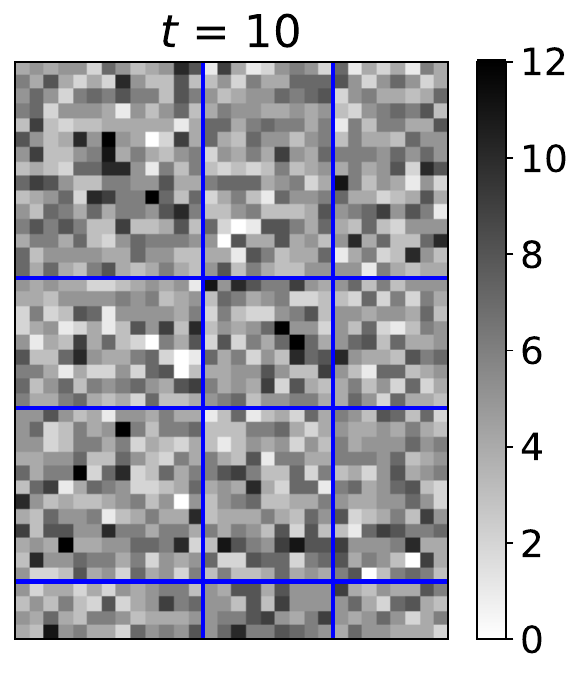}
  \caption{Examples of null block structures of the \textbf{Poisson} LBM. }
  \label{fig:As_pois}
\end{figure}
%-----

Figure \ref{fig:accuracy} shows the accuracy of the proposed test under $10$ different settings of block-wise mean $B$. From Figure \ref{fig:accuracy}, we see that the test accuracy increases with matrix size $n$ for a fixed block-wise mean $B$, and that it decreases with the smaller differences between the block-wise means for a fixed matrix size $n$. 
\begin{figure}[t]
  \centering
  \includegraphics[width=0.32\hsize]{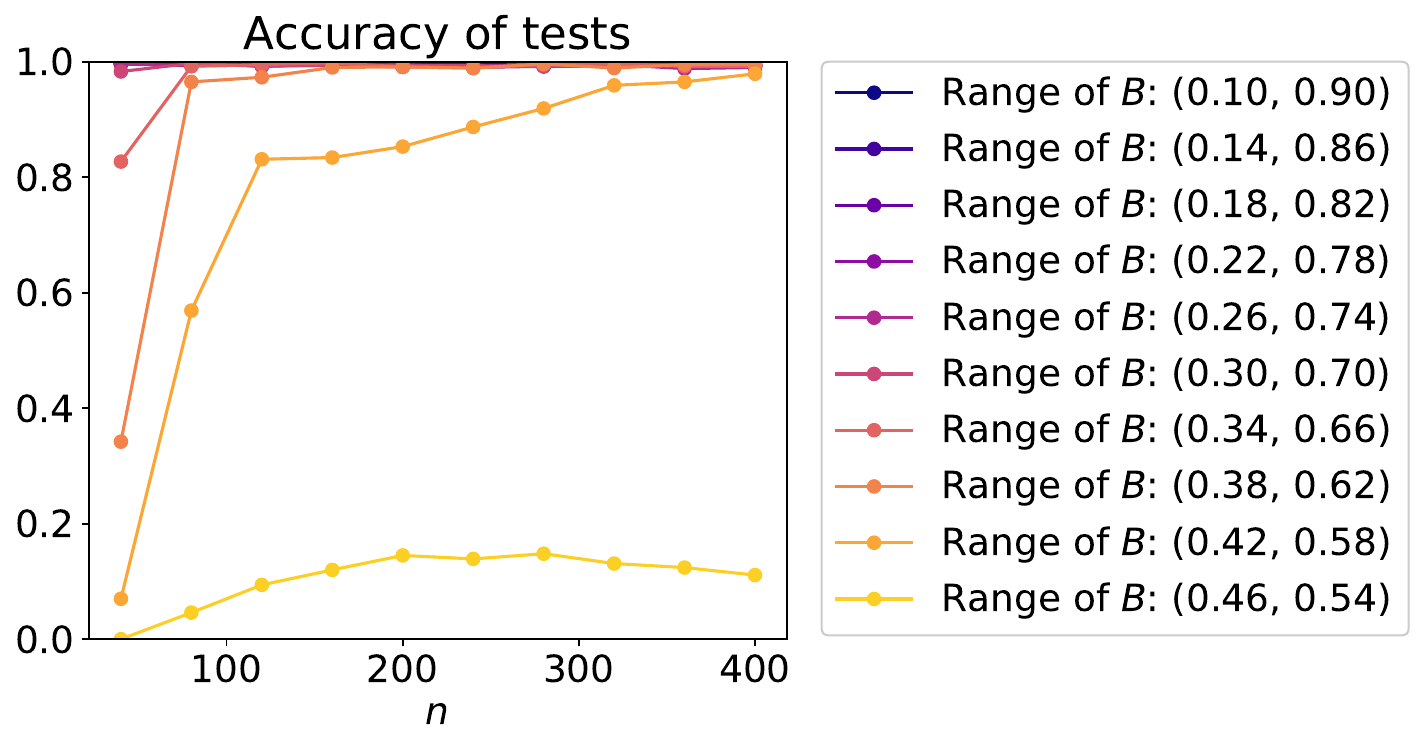}
  \includegraphics[width=0.32\hsize]{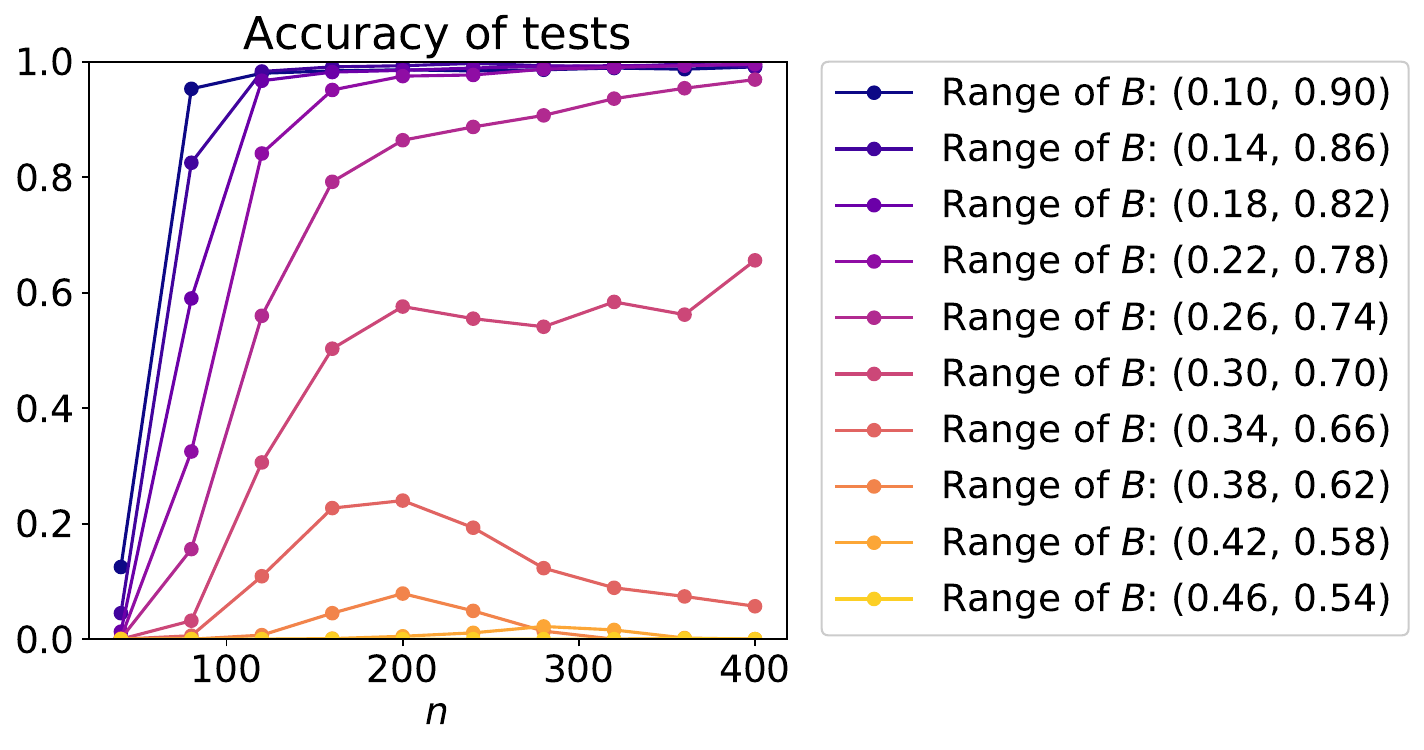}
  \includegraphics[width=0.32\hsize]{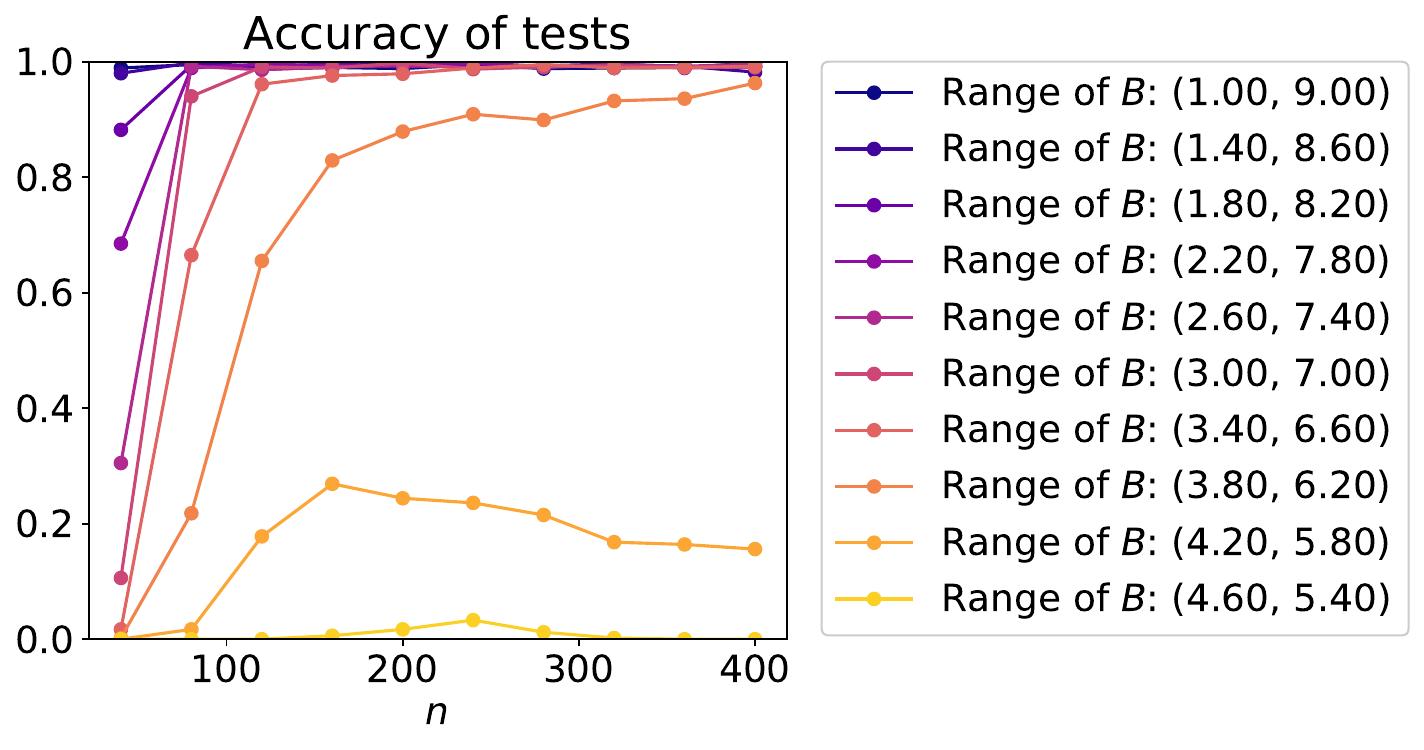}
  \caption{Accuracy of the proposed goodness-of-fit test under $10$ different settings of block-wise mean $B$. The left, center, and right figures, respectively, show the results for the settings of Gaussian, Bernoulli, and Poisson distributions. }
  \label{fig:accuracy}
\end{figure}
%-----

\paragraph{Comparison to the integrated completed likelihood (ICL)}

We also checked the difference in the behavior of the proposed test and the ICL. For the Bernoulli LBM, we can compute the asymptotic ICL \cite{Keribin2012} by assuming the following model: 
\begin{align}
\label{eq:icl_lh}
&p(A, g^{(1)}, g^{(2)} | \pi, \rho, B) \nonumber \\
&= \left( \prod_k \pi_k^{|I_k|} \right) \left( \prod_h \rho_h^{|J_h|} \right) \left[ \prod_{i, j} B_{g^{(1)}_i g^{(2)}_j}^{A_{ij}} \left( 1 - B_{g^{(1)}_i g^{(2)}_j} \right)^{1 - A_{ij}} \right], \nonumber \\
&p(\pi) = \frac{\Gamma (a_{\mathrm{D}} K_0)}{\Gamma (a_{\mathrm{D}})^{K_0}} \prod_{k=1}^{K_0} \pi_k^{a_{\mathrm{D}}-1}, \ \ \ 
p(\rho) = \frac{\Gamma (a_{\mathrm{D}} H_0)}{\Gamma (a_{\mathrm{D}})^{H_0}} \prod_{h=1}^{H_0} \rho_h^{a_{\mathrm{D}}-1}, \nonumber \\
&p(B) = \prod_{k, h} \frac{B_{kh}^{b_{\mathrm{B}}-1}(1 - B_{kh})^{b_{\mathrm{B}}-1}}{B(b_{\mathrm{B}}, b_{\mathrm{B}})}, 
\end{align}
where $p(\cdot)$ represents a probability density, and $a_{\mathrm{D}}$ and $b_{\mathrm{B}}$ are the hyperparameters. 

From Lemma 4.2 in \cite{Keribin2012}, for an estimated block structure $(\hat{g}^{(1)}, \hat{g}^{(2)})$, the resulting asymptotic ICL is given by
\begin{align}
\label{eq:icl_K0H0}
\mathrm{ICL} (K_0, H_0) &= \sum_k |I_k| \log \left( \frac{|I_k|}{n} \right) + \sum_h |J_h| \log \left( \frac{|J_h|}{p} \right) \nonumber \\
&+ \sum_{k, h} |I_{k}| |J_{h}| \left[ \hat{B}_{k h} \log \hat{B}_{k h} + \left( 1 - \hat{B}_{k h} \right) \log \left( 1 - \hat{B}_{k h} \right) \right] \nonumber \\
&- \frac{K_0 - 1}{2} \log n - \frac{H_0 - 1}{2} \log p - \frac{K_0 H_0}{2} \log (np). 
\end{align}
The proof of (\ref{eq:icl_K0H0}) is given in Appendix \ref{ap_icl}. 

To check the accuracy of the proposed test and the ICL, we generated synthetic binary data matrices based on the Bernoulli distribution as in Section \ref{sec:accuracy}, and checked the ratio of trials where the selected set of cluster numbers $(K_0, H_0)$ is equal to the null one $(K, H)$. We set the null set of cluster numbers at $(K, H) = (4, 3)$, and tried the following five settings with respect to the block-wise mean $B$. 
\begin{align}
&B' = 
\begin{pmatrix}
0.9 & 0.1 & 0.4  \\
0.2 & 0.7 & 0.3  \\
0.3 & 0.2 & 0.8  \\
0.6 & 0.9 & 0.1  \\
\end{pmatrix}, \nonumber \\
&\forall k, h,\ B_{kh} = \left( 1 - \frac{t}{5} \right) (B'_{kh} - 0.5) + 0.5, \ \ \ \mathrm{for}\ t = 0, \dots, 4. 
\end{align}
With respect to the matrix size, we tried the following five settings for each setting of $B$: $(n, p) = (40 \times i, 30 \times i)$, $i = 1, \dots, 5$. The null block of each element was chosen in the same way as in Section \ref{sec:accuracy}. In each of $5$ (for the setting of $B$) $\times 5$ (for the setting of matrix size) settings, we generated $100$ observed matrices, and applied the proposed test using a significance level of $\alpha = 0.01$ and the model selection based on the ICL. Unlike the proposed sequential test, which stopped if the null hypothesis was accepted, the ICL was computed for all the sets of cluster numbers from $(1, 1)$ to $(n, p)$ and then the optimal setting was selected that achieved the maximum ICL. For each setting, we estimated the block structure of an observed matrix based on the Ward's hierarchical clustering algorithm \cite{Ward1963}. 

Figure \ref{fig:accuracy_compare} shows the accuracy of the proposed test and the model selection based on the ICL. Although the purpose of the proposed test is not to achieve high accuracy in model selection, in some cases with small differences between the block-wise means $\{ B_{kh} \}$, it achieved better performance than the ICL. With larger difference between $\{ B_{kh} \}$, the ICL performed better than the proposed test in terms of model selection. Figures \ref{fig:ratio_compare_test} and \ref{fig:ratio_compare_ICL}, respectively, show the ratios of the trials where each set of cluster numbers was selected by the proposed test and the ICL. From Figures \ref{fig:ratio_compare_test} and \ref{fig:ratio_compare_ICL}, we see that in most cases (e.g., $B_{kh} \in [0.26, 0.74]$ for all $(k, h)$), the ICL tended to select smaller sets of cluster numbers than the proposed test. 

%-----
\begin{figure}[t]
  \centering
  \includegraphics[width=0.4\hsize]{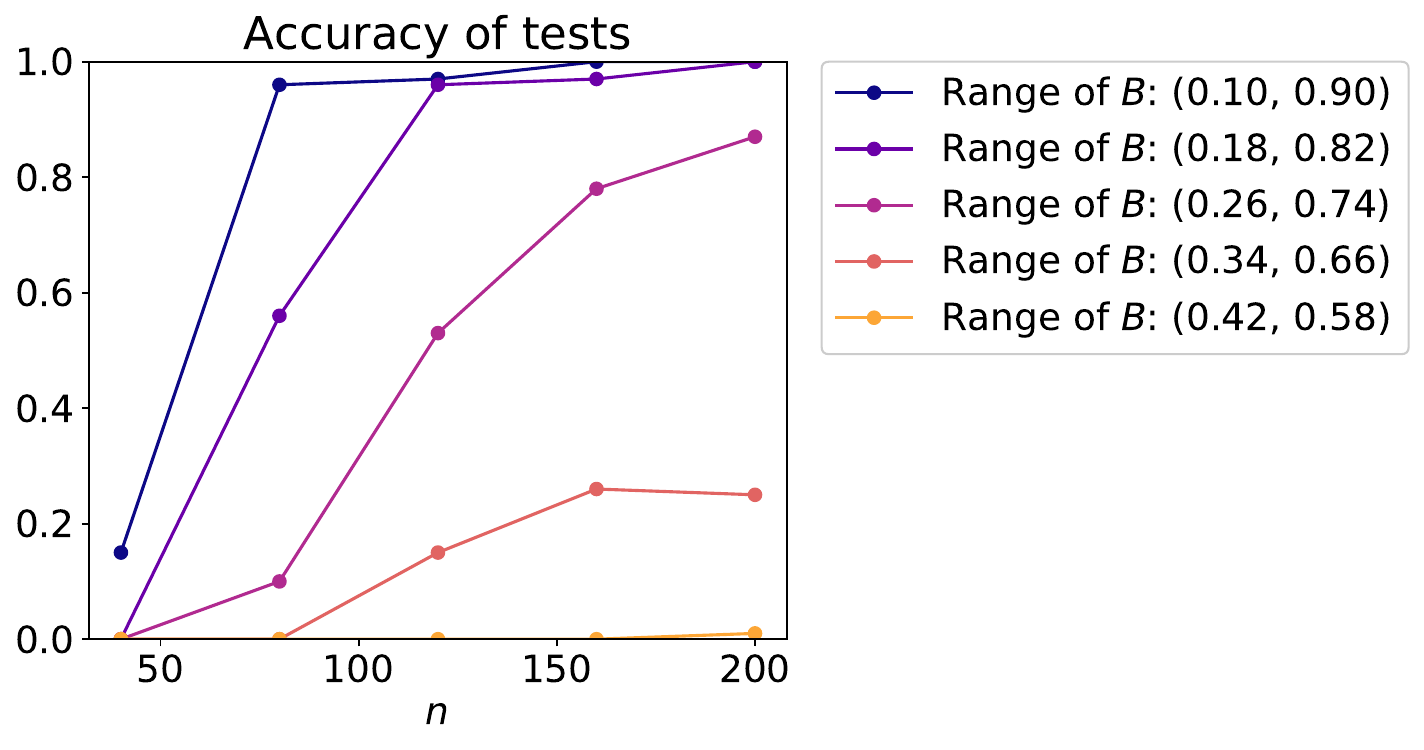}
  \includegraphics[width=0.4\hsize]{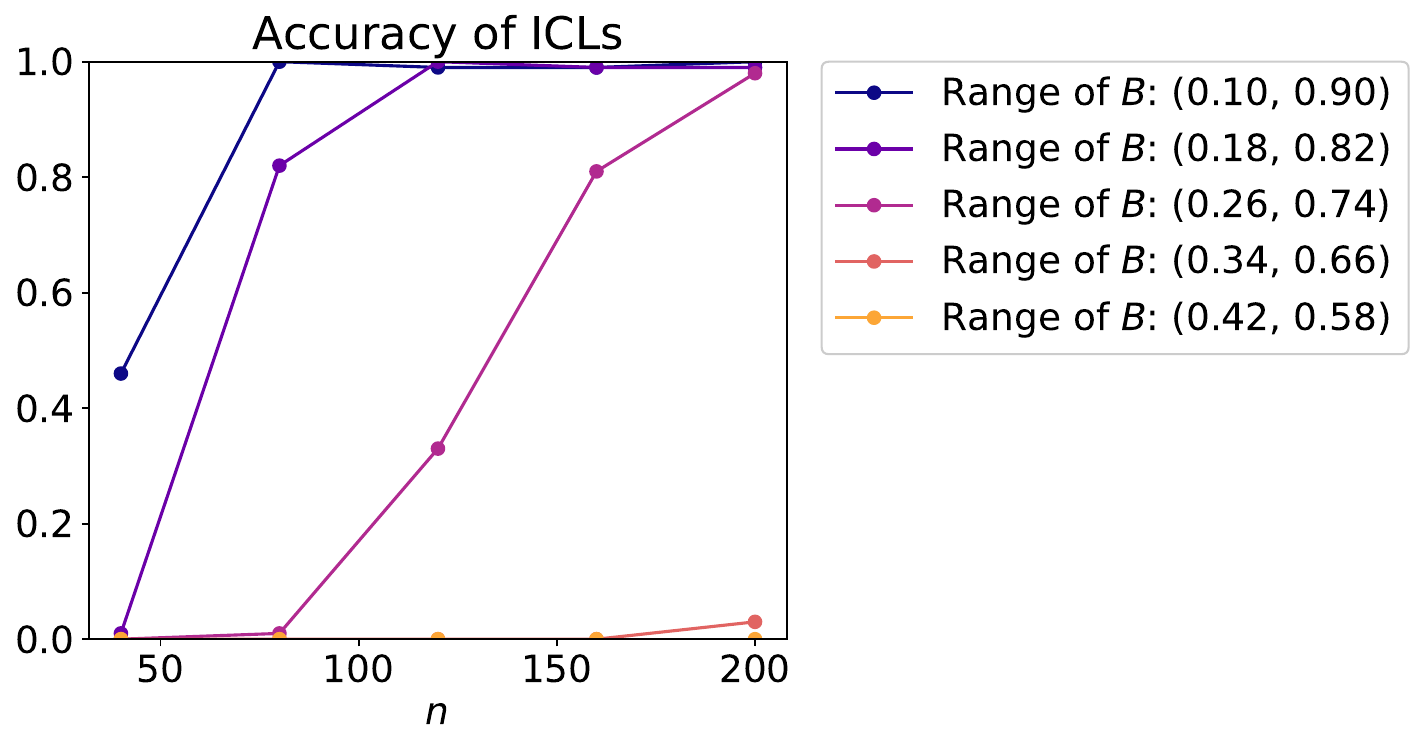}
  \caption{Accuracy of the proposed test and the model selection based on the ICL under five different settings of block-wise mean $B$. }
  \label{fig:accuracy_compare}
\end{figure}
%-----
\begin{figure}[t]
  \centering
  % [Revision 2020/9] <---
  \includegraphics[width=0.99\hsize]{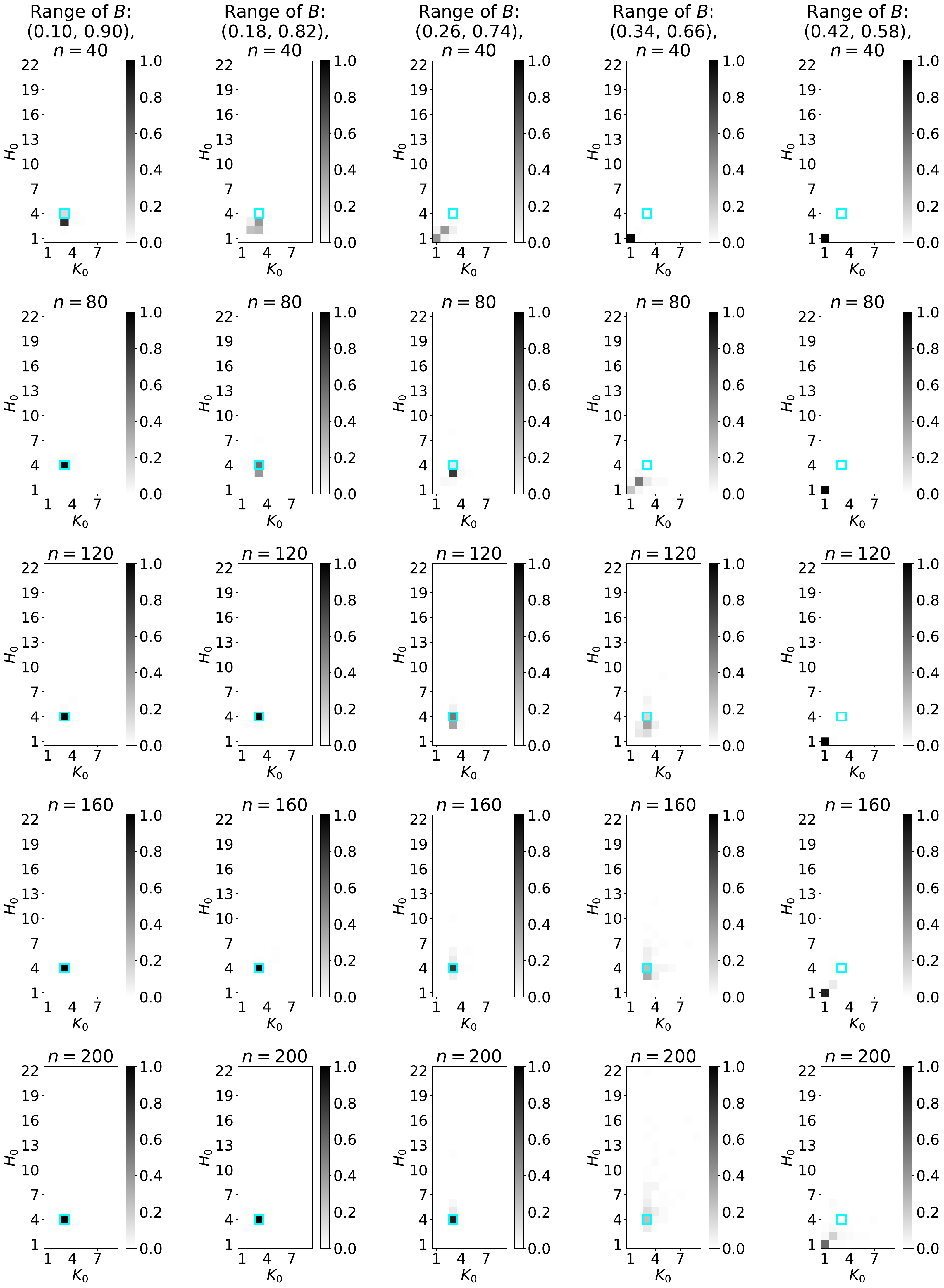}
  % [Revision 2020/9] --->
  \caption{Ratios of the trials where each set of cluster numbers $(K_0, H_0)$ was selected by the proposed test. The cyan rectangles show the null set of cluster numbers $(K, H)$. }
  \label{fig:ratio_compare_test}
\end{figure}
\begin{figure}[t]
  \centering
  % [Revision 2020/9] <---
  \includegraphics[width=0.99\hsize]{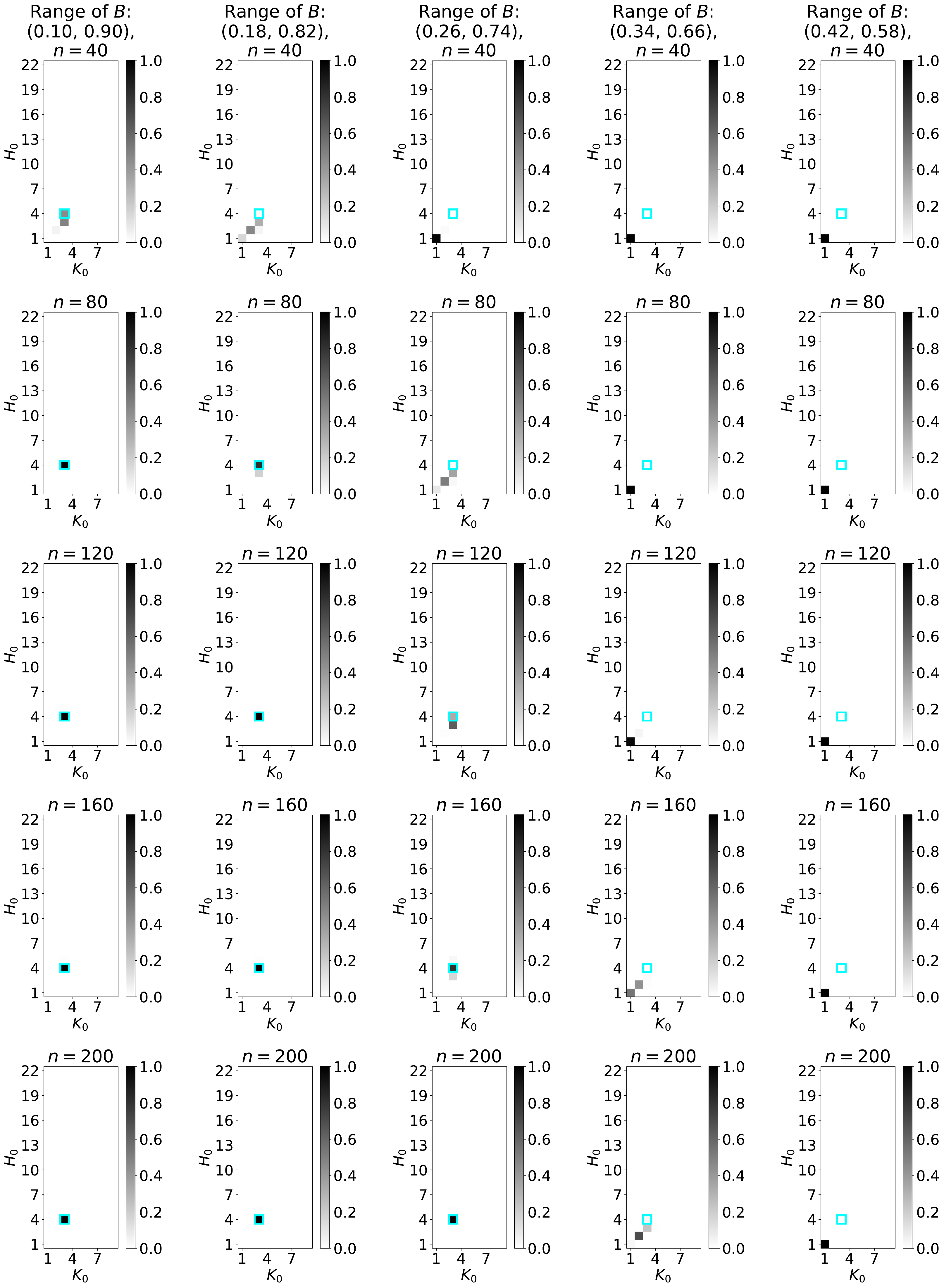}
  % [Revision 2020/9] --->
  \caption{Ratios of the trials where each set of cluster numbers $(K_0, H_0)$ was selected by the model selection based on the ICL. The cyan rectangles show the null set of cluster numbers $(K, H)$. }
  \label{fig:ratio_compare_ICL}
\end{figure}
%-----

\subsection{Real data analysis: Congressional Voting Records Data Set}

We also checked the result when we applied the proposed test to 1984 United States Congressional Voting Records Database from UCI Machine Learning Repository \cite{Dua2017}. The original data set contains three types of votes (``yea,'' ``nay,'' and unknown) for the pairs of a congressman and an attribute. We treated unknown as ``nay,'' as in \cite{Wyse2017}. The number of instances or congressmen and that of attributes are $435$ and $16$, respectively. Based on this data set, we defined a binary matrix $A \in \mathbb{R}^{435 \times 16}$, where the elements of one and zero, respectively, correspond to ``yea'' and ``nay.'' 

As in Section \ref{sec:accuracy}, we applied the proposed sequential tests using a significance level of $\alpha = 0.01$, until the null hypothesis was accepted. We also computed the ICL for each setting of a hypothetical set of cluster numbers $(K_0, H_0)$, and selected one with the largest ICL. For each setting of a hypothetical set of cluster numbers $(K_0, H_0)$, we estimated the block structure based on the Ward's hierarchical clustering algorithm \cite{Ward1963}. 

% [Revision 2020/9] <---
As a result, the sets of cluster numbers $(9, 14)$ and $(3, 13)$ were selected by the proposed test and the ICL, respectively. Figure \ref{fig:congress_A} shows the observed data matrix and its estimated block structures with the selected sets of cluster numbers. From Figure \ref{fig:congress_A}, we see that a finer block structure was accepted by the proposed test than the ICL, particularly for the row (i.e., congressman) cluster assignments. As for the column (i.e., attribute) cluster assignments, ``anti-satellite-test-ban,'' ``aid-to-nicaraguan-contras,'' and ``mx-missile'' were assigned into the same cluster in the selected block structure of the ICL, whereas the proposed test distinguished the first two attributes from the last one. Figure \ref{fig:congress_result} shows the $p$-value of the proposed test and the ICL for each setting of a hypothetical set of cluster numbers $(K_0, H_0)$ until the null hypothesis was accepted. 
% [Revision 2020/9] --->

%-----
\begin{figure}[t]
  \centering
  \includegraphics[width=0.265\hsize]{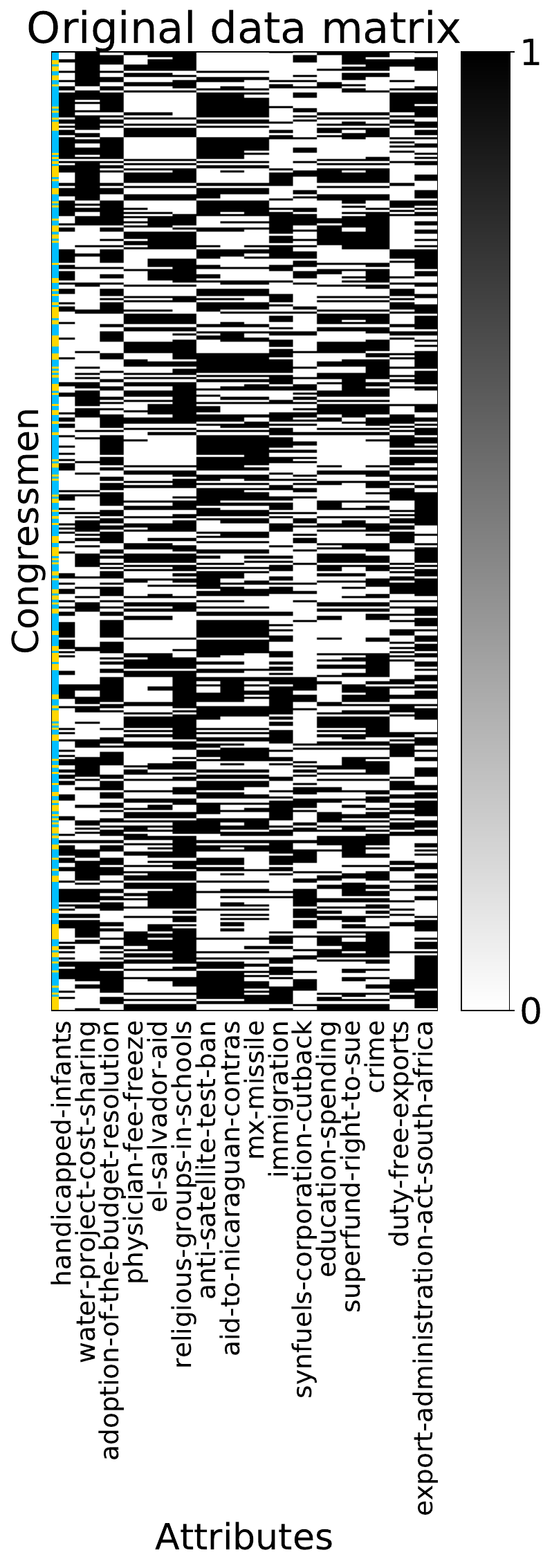}
  \includegraphics[width=0.265\hsize]{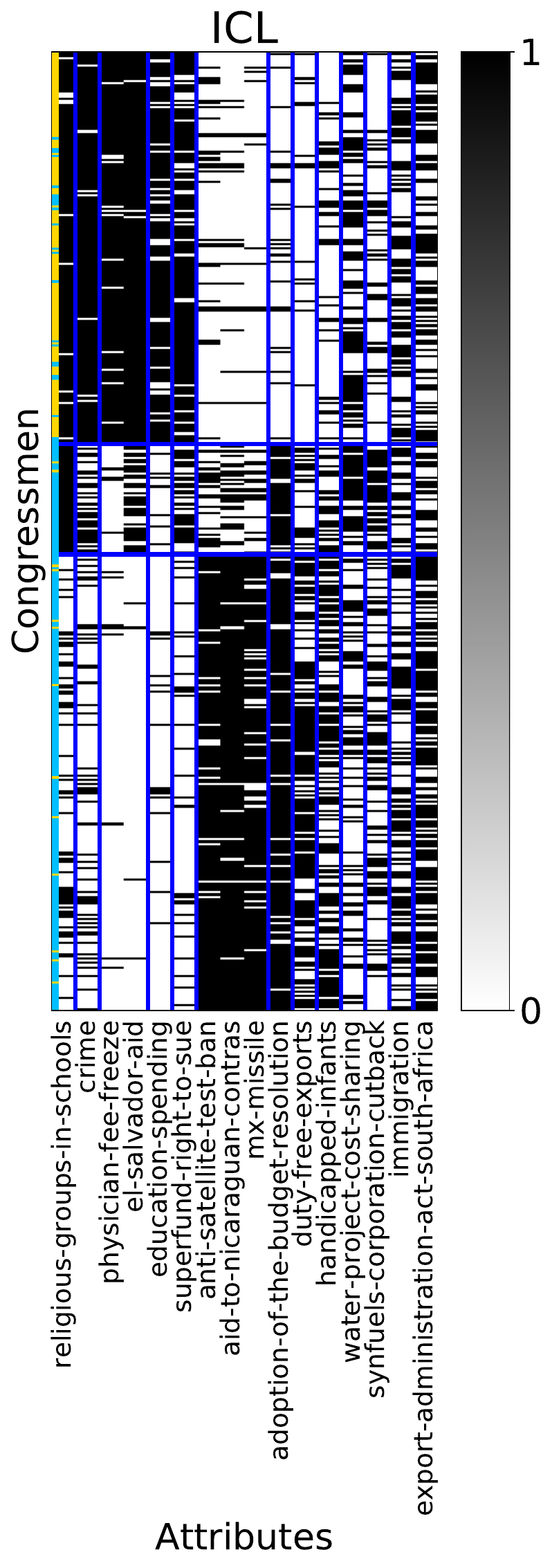}
  \includegraphics[width=0.265\hsize]{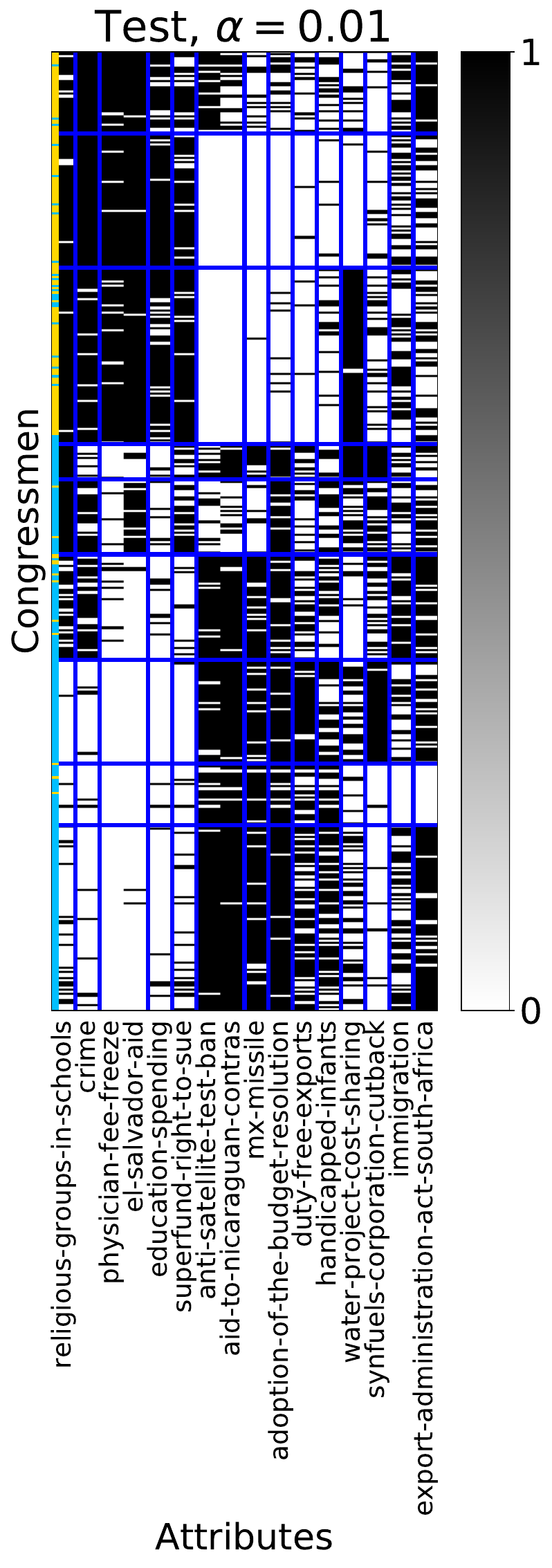}
  \includegraphics[width=0.17\hsize]{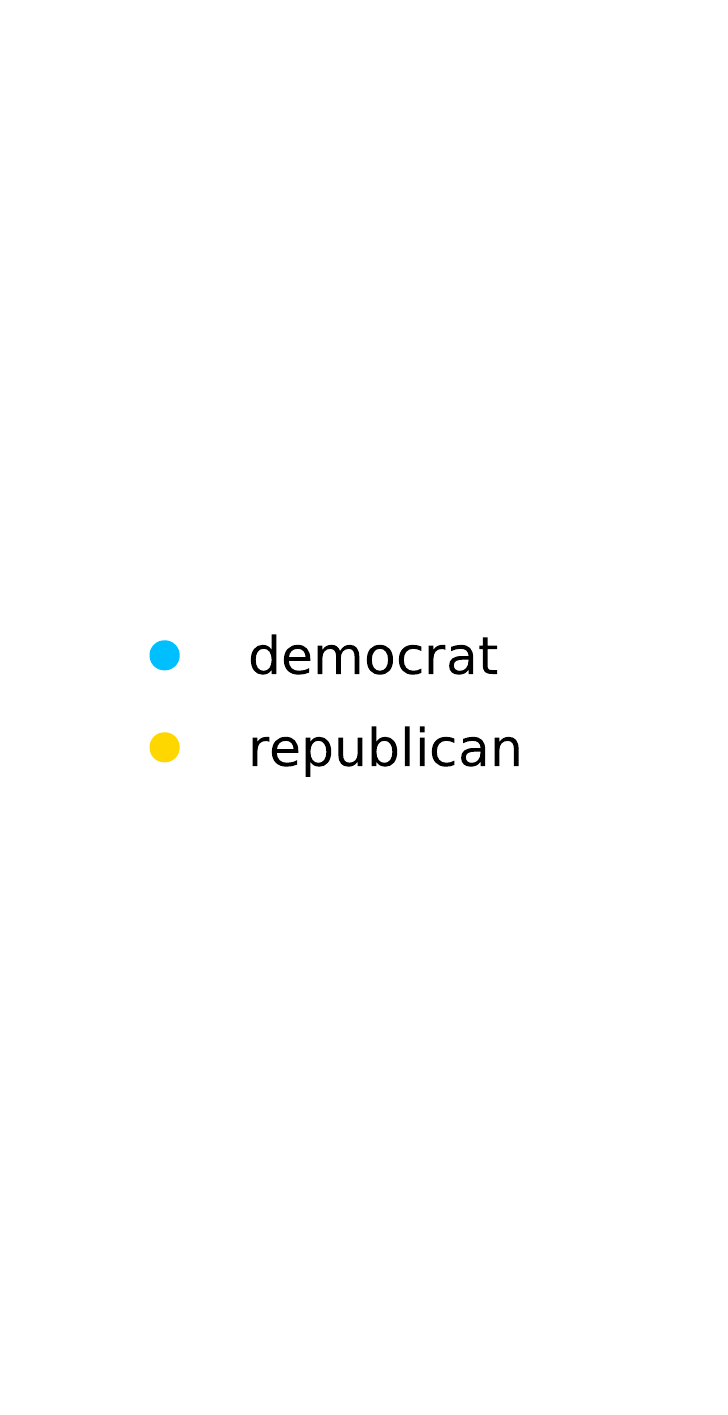}
  \caption{The observed data matrix of the Congressional Voting Records Data Set \cite{Dua2017} and its estimated block structures. The black and white elements, respectively, show ``yea'' and ``nay.'' }
  \label{fig:congress_A}
\end{figure}
%-----
\begin{figure}[t]
  \centering
  \includegraphics[width=0.3\hsize]{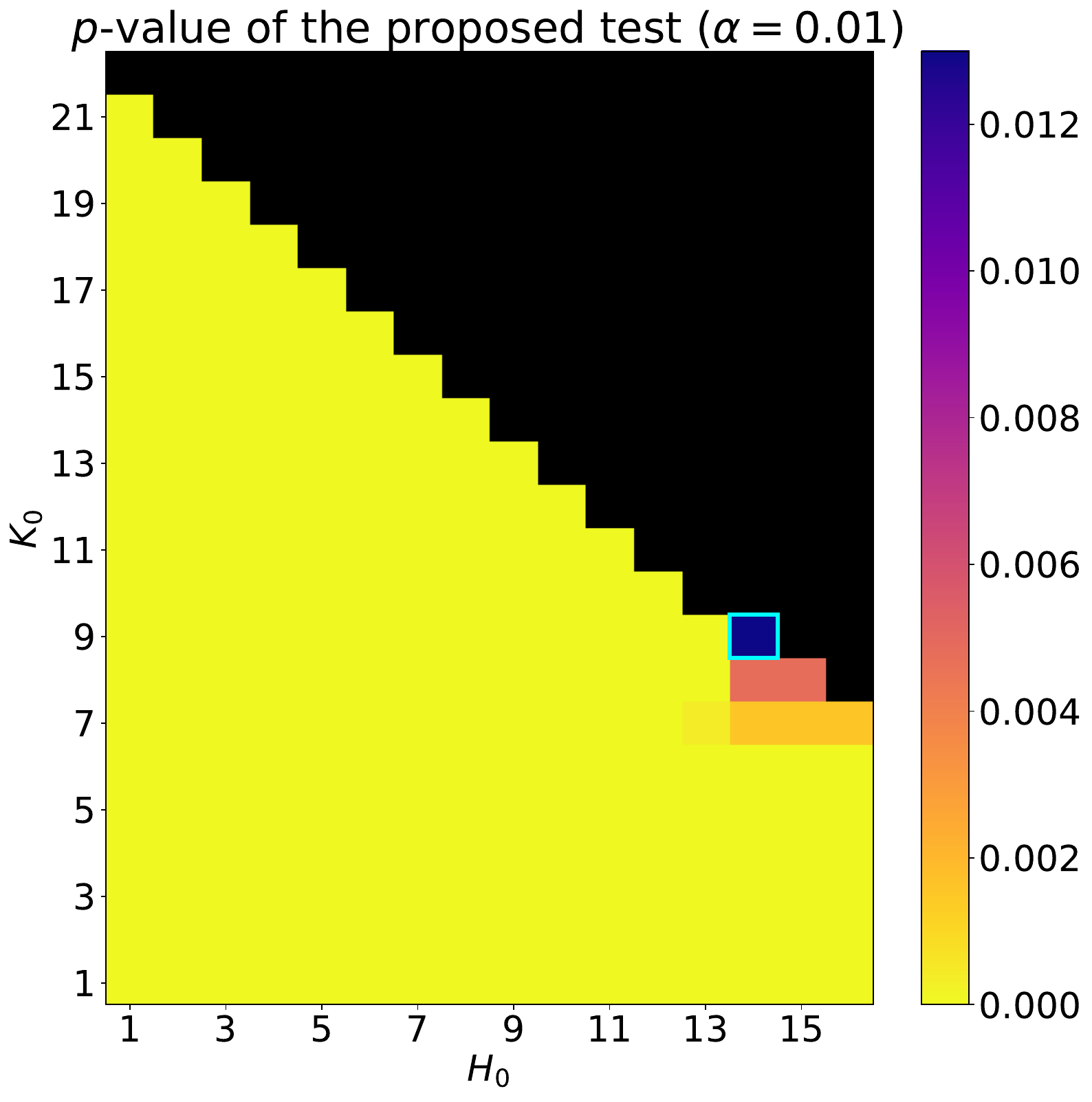}
  \includegraphics[width=0.3\hsize]{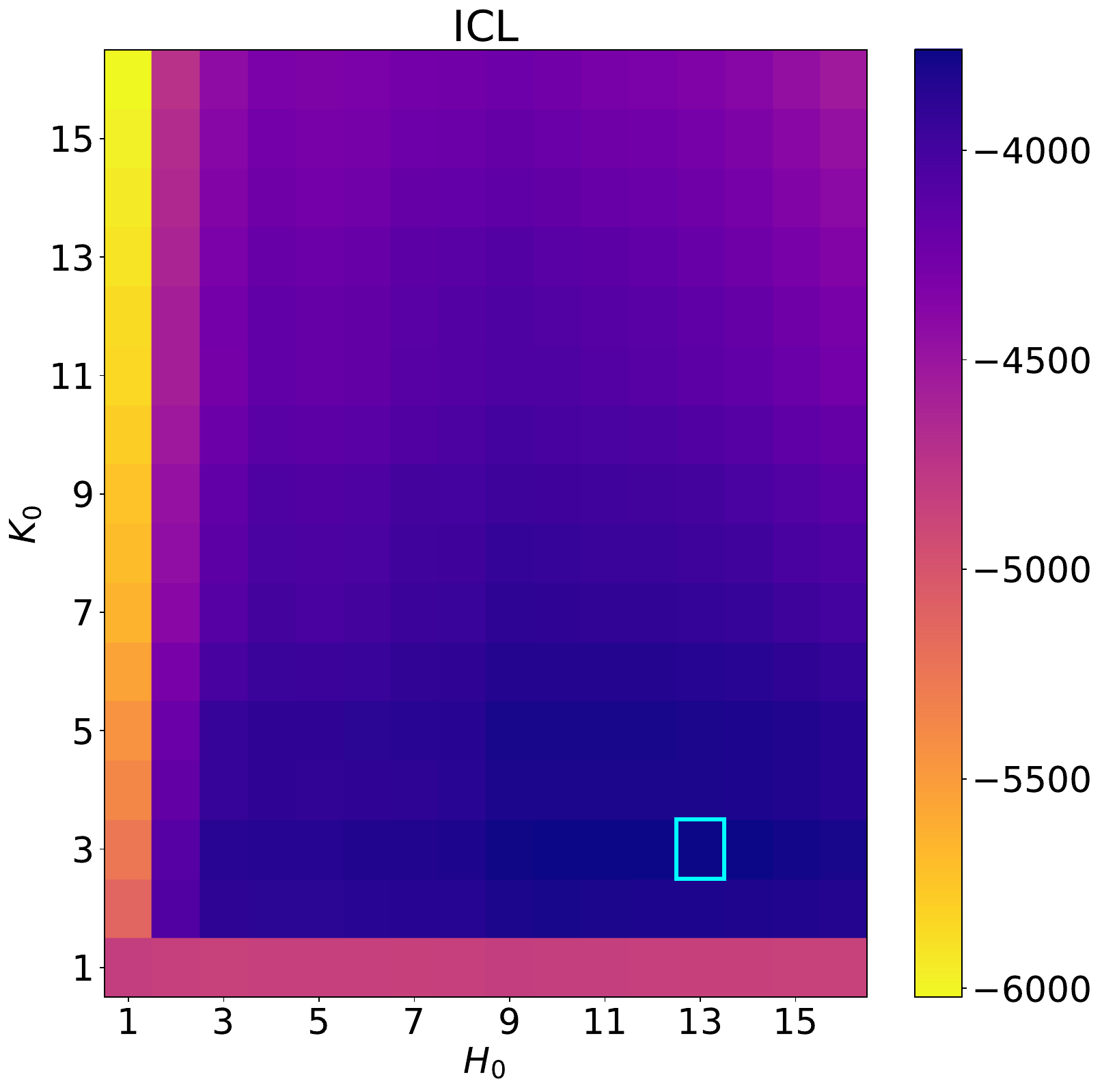}
  \caption{The $p$-value of the proposed test and the ICL for each setting of a hypothetical set of cluster numbers $(K_0, H_0)$. In the left figure, we plotted the $p$-values only for the tested settings (i.e., until the null hypothesis was accepted). As for the ICL, we plotted a part of results ($K_0 \leq 16$, $H_0 \leq 16$) for visibility. The cyan rectangles show the selected sets of cluster numbers. }
  \label{fig:congress_result}
\end{figure}
%-----

%%%%%%%%%%%%%%%%%%%%%%%%%%%%%%%%%%%%%%%%%%%%%%%%%%%%%%%

\section{Discussion}
\label{sec:discussion}

In this section, we discuss the proposed test method in terms of the test statistic and the conditions for the generative model.

With respect to the asymptotic behavior, the proposed test has a favorable property in terms of the power. From Theorem \ref{th:unrealizable_two_sided}, under the alternative hypothesis, the test statistic $T$ increases in proportion to $m^{\frac{5}{3}}$ with high probability, where $n, p \propto m$. In other words, the probability that the test makes a type II error (i.e., $T < t(\alpha)$) converges to zero in the limit of $p \to \infty$. Based on this fact, in the asymptotic sense, we do not need to consider the correction for the multiple comparison when applying the proposed sequential testing. However, it has not been shown what occurs in the non-asymptotic setting. In general, practical data matrices have finite sizes, where there has been shown no theoretical guarantee like Theorems \ref{th:realizable}, \ref{th:unrealizable_lower}, and \ref{th:unrealizable_upper}. 
On the other hand, for a Gaussian case (i.e., each entry of a matrix independently follows $\mathcal{N} (0, 1)$), the following statement holds \cite{Ma2012}: Suppose $n = n(p) > p$ and $n/p \to \gamma \in [1, \infty )$ in the limit of $p \to \infty$. Then, for any $s_0$, there exists $N_0 \in \mathbb{N}$ such that when $\max (n, p) \geq N_0$ and $\max (n, p)$ is even, for all $s \geq s_0$, 
\begin{eqnarray}
\label{eq:finite_sample_bound}
|\mathrm{Pr} (T^* \leq s) - F_1 (s)| \leq C(s_0) [\max (n, p)]^{-2/3} \exp \left( -\frac{s}{2} \right), 
\end{eqnarray}
where $T^*$ is defined as in (\ref{eq:T_true}) and $C(\cdot )$ is a continuous and non-increasing function. From the above inequality (\ref{eq:finite_sample_bound}), \textit{if the clustering algorithm outputs the correct block assignments,} the convergence rate of the normalized maximum eigenvalue $T^*$ of matrix $\tilde{Z}^\top \tilde{Z}$ (where $\tilde{Z}$ is defined as in (\ref{eq:tilde_Z})) to the Tracy-Widom distribution with index $1$ is $O(m^{-2/3})$. However, since the distribution of $T$ is unknown in the case where the correct block assignment is \textit{not} obtained, the convergence rate of $T$ is also unknown. Deriving the convergence rate of $T$ by considering the above discussion is a future research topic. 

In regard to the conditions for using the proposed test method, our proposed test is applicable to a wide range of practical distributional settings (e.g., Bernoulli distribution for binary data matrices and Poisson distribution for sparse ones). Nevertheless, it still requires some assumptions for the latent block structure of an observed matrix. For instance, the row and column cluster numbers $(K, H)$ should be constants that do not increase with the matrix sizes $n$ and $p$. Also, there should be no too small block (i.e., $n_{\mathrm{min}} = \Omega_p (m)$ and $p_{\mathrm{min}} = \Omega_p (m)$). In some practical cases, where it is more appropriate to assume that the number of blocks increases with the matrix size, it will be useful to construct a test which does not require the above conditions. As for the sub-exponential condition, Ding and Yang \cite{Ding2018} have shown more relaxed sufficient condition for the scaled maximum eigenvalue $T^*$ to converge in law to $TW_1$ distribution. However, the \textit{delocalization property} of an eigenvector of matrix $Z^{\top} Z$ \cite{Bloemendal2016}, which we used in Appendix \ref{sec:ap_lambda_diff} to prove our main result, has not been derived in the form as in the sub-exponential case \cite{Bloemendal2016}. If (\ref{eq:delocalization}) in Appendix \ref{sec:ap_lambda_diff} is shown in the above more general case, it would also be possible to extend our proposed test to such a case. Furthermore, there are proposed variants of latent block models with which we assume different block structures from a regular grid \cite{Roy2008, Nakano2014}. To construct test methods for the above settings is an important topic for future research. 

%%%%%%%%%%%%%%%%%%%%%%%%%%%%%%%%%%%%%%%%%%%%%%%%%%%%%%%

\section{Conclusion}
\label{sec:conclusion}

Latent block models are effective tools for biclustering, where rows and columns of an observed matrix are simultaneously decomposed into clusters. Such a bicluster structure appears in various types of relational data, such as the customer-product transaction data or and the document-word relationship data. One open problem in using latent block models is that there has been no statistical test method for determining the number of blocks. In this study, we developed a goodness-of-fit test for latent block models based on a result from the random matrix theory. By defining the test statistic $T$ based on the estimators of the block-wise means and standard deviations, we have derived its asymptotic behavior in both realizable (i.e., $(K, H) = (K_0, H_0)$) and unrealizable (i.e., $K > K_0$ or $H > H_0$) cases. Particularly, it has been shown that the test statistic $T$ converges in law to Tracy-Widom distribution with index $1$ in the realizable case. Based on these results, it was made possible to test whether the given observed matrix had $K_0 \times H_0$ latent blocks or more ones. In the experiments, we showed the validity of the proposed test method in terms of both the asymptotic behavior of the test statistic and the test accuracy by using synthetic data matrices with ground truth block structures. 

%%%%%%%%%%%%%%%%%%%%%%%%%%%%%%%%%%%%%%%%%%%%%%%%%%%%%%%

\section*{Acknowledgments}

% [Revision final] <<<
TS was partially supported by JSPS KAKENHI (18K19793, 18H03201, and 20H00576), Japan Digital Design, and JST CREST.
% [Revision final] >>>

%%%%%%%%%%%%%%%%%%%%%%%%%%%%%%%%%%%%%%%%%%%%%%%%%%%%%%%

\clearpage
\begin{appendices}
\section{Proof of $\left| \tilde{S}_{kh} - S_{kh} \right| = O_p \left( \frac{1}{m} \right)$. }
\label{sec:ap_sigma_tilde}

Let $n_k$ and $p_h$, respectively, be the row and column sizes of the $(k, h)$th \textbf{null} block, and $A^{(k, h)}$, $P^{(k, h)}$ and $\tilde{P}^{(k, h)}$, respectively, be the $(k, h)$th \textbf{null} blocks of matrices $A$, $P$ and $\tilde{P}$. Here, we prove the following lemma: 
\setcounter{lemma}{0}
\renewcommand{\thelemma}{\Alph{section}\arabic{lemma}}
\begin{lemma}
\label{lm:sigma_tilde}
Under the assumption that the fourth moment of the noise $Z_{ij}$ is bounded ($\mathbb{E} [Z_{ij}^4] < \infty$), 
\begin{eqnarray}
\label{eq:sigma_tilde_op}
\left| \tilde{S}_{kh} - S_{kh} \right| = O_p \left( \frac{1}{m} \right), 
\end{eqnarray}
% [Revision 2020/9] <---
where $\tilde{S}_{kh} = \sqrt{\frac{1}{n_k p_h} \sum_{i=1}^{n_k} \sum_{j=1}^{p_h} \left( A^{(k, h)}_{ij} - \tilde{B}_{kh} \right)^2}$. 
% [Revision 2020/9] --->
\end{lemma}
\begin{proof}
From the above definition of $\tilde{S}_{kh}$, we have 
\begin{align}
% [Revision 2020/9] <---
\tilde{S}_{kh}^2 
=&\ \frac{1}{n_k p_h} \sum_{i=1}^{n_k} \sum_{j=1}^{p_h} \left( A^{(k, h)}_{ij} - \tilde{B}_{kh} \right)^2 \nonumber \\
=&\ \frac{1}{n_k p_h} \sum_{i=1}^{n_k} \sum_{j=1}^{p_h} \left( A^{(k, h)}_{ij} - B_{kh} \right)^2 - \left( B_{kh} - \tilde{B}_{kh} \right)^2. 
% [Revision 2020/9] --->
\end{align}
% [Revision 2020/9] <---
To derive the second equation, we used the fact that $\tilde{B}_{kh} = \frac{1}{n_k p_h} \sum_{i=1}^{n_k} \sum_{j=1}^{p_h} A^{(k, h)}_{ij}$. Therefore, the following inequality holds: 
\begin{eqnarray}
\label{eq:sigma_tilde_diff}
\left| \tilde{S}_{kh}^2 - S_{kh}^2 \right| \leq \left| \frac{1}{n_k p_h} \sum_{i=1}^{n_k} \sum_{j=1}^{p_h} \left( A^{(k, h)}_{ij} - B_{kh} \right)^2 - S_{kh}^2 \right| + \left( B_{kh} - \tilde{B}_{kh} \right)^2. 
% [Revision 2020/9] --->
\end{eqnarray}

The first term in (\ref{eq:sigma_tilde_diff}) is given by 
\begin{align}
\label{eq:sigma_tilde_y}
% [Revision 2020/9] <---
&\frac{1}{n_k p_h} \sum_{i=1}^{n_k} \sum_{j=1}^{p_h} \left( A^{(k, h)}_{ij} - B_{kh} \right)^2 - S_{kh}^2 \nonumber \\
=&\ \frac{1}{n_k p_h} \sum_{i=1}^{n_k} \sum_{j=1}^{p_h} \left[ \left( A^{(k, h)}_{ij} - B_{kh} \right)^2 - S_{kh}^2 \right] 
= \frac{1}{n_k p_h} \sum_{i=1}^{n_k} \sum_{j=1}^{p_h} Y^{(k, h)}_{ij}, 
% [Revision 2020/9] --->
\end{align}
where we defined that $Y^{(k, h)}_{ij} \equiv \left( A^{(k, h)}_{ij} - B_{kh} \right)^2 - S_{kh}^2$. Note that $(Y^{(k, h)}_{ij})_{1 \leq i \leq n_k, 1 \leq j \leq p_h}$ is independent. The expectation and the variance of $Y^{(k, h)}_{ij}$ are given by
\begin{align}
\label{eq:sigma_tilde_ev}
\mathbb{E} \left[ Y^{(k, h)}_{ij} \right] =&\ \mathbb{E} \left[ \left( A^{(k, h)}_{ij} - B_{kh} \right)^2 \right] - S_{kh}^2 = 0, \nonumber \\
\mathbb{V} \left[ Y^{(k, h)}_{ij} \right] =&\ \mathbb{E} \left[ \left( Y^{(k, h)}_{ij} \right)^2 \right] 
= \mathbb{E} \left[ \left\{ \left( A^{(k, h)}_{ij} - B_{kh} \right)^2 - S_{kh}^2 \right\}^2 \right] \nonumber \\
=&\ S_{kh}^4 \left( \mathbb{E} \left[ \left( Z^{(k, h)}_{ij} \right)^4  \right] - 1 \right). 
\end{align}
From (\ref{eq:sigma_tilde_ev}), we have 
\begin{align}
\label{eq:sigma_tilde_ev_mean}
\mathbb{E} \left[ \frac{1}{n_k p_h} \sum_{i=1}^{n_k} \sum_{j=1}^{p_h} Y^{(k, h)}_{ij} \right] =&\ 0, \nonumber \\
\mathbb{V} \left[ \frac{1}{n_k p_h} \sum_{i=1}^{n_k} \sum_{j=1}^{p_h} Y^{(k, h)}_{ij} \right] =&\ \frac{1}{n_k p_h} S_{kh}^4 \left( \mathbb{E} \left[ \left( Z^{(k, h)}_{ij} \right)^4  \right] - 1 \right). 
\end{align}
From (\ref{eq:sigma_tilde_y}), (\ref{eq:sigma_tilde_ev_mean}), and Chebyshev's inequality, for all $t>0$, 
\begin{align}
\label{eq:chebyshev_abs}
% [Revision 2020/9] <---
\mathrm{Pr} &\left( \left| \frac{1}{n_k p_h} \sum_{i=1}^{n_k} \sum_{j=1}^{p_h} \left( A^{(k, h)}_{ij} - B_{kh} \right)^2 - S_{kh}^2 \right| \geq \right. \nonumber \\
&\left. t \frac{1}{\sqrt{n_k p_h}} \sqrt{ S_{kh}^4 \left( \mathbb{E} \left[ \left( Z^{(k, h)}_{ij} \right)^4  \right] - 1 \right) } \right) \leq \frac{1}{t^2}. 
% [Revision 2020/9] --->
\end{align}
Therefore, we have 
\begin{eqnarray}
\label{eq:sigma_tilde_term1}
% [Revision 2020/9] <---
\left| \frac{1}{n_k p_h} \sum_{i=1}^{n_k} \sum_{j=1}^{p_h} \left( A^{(k, h)}_{ij} - B_{kh} \right)^2 - S_{kh}^2 \right| = O_p \left( \frac{1}{m} \right). 
% [Revision 2020/9] --->
\end{eqnarray}

On the other hand, the second term in (\ref{eq:sigma_tilde_diff}) is given by 
\begin{align}
\label{eq:sigma_tilde_pp}
% [Revision 2020/9] <---
&\left( B_{kh} - \tilde{B}_{kh} \right)^2 \nonumber \\
%= \left( \mathbb{E} \left[ \tilde{B}_{kh} \right]  - \tilde{B}_{kh} \right)^2 \nonumber \\
=&\ \left[ \frac{1}{n_k p_h} \sum_{i=1}^{n_k} \sum_{j=1}^{p_h} \left( P^{(k, h)}_{ij} - A^{(k, h)}_{ij} \right) \right]^2  \ \ \ \left(\because \tilde{B}_{kh} = \frac{1}{n_k p_h} \sum_{i=1}^{n_k} \sum_{j=1}^{p_h} A^{(k, h)}_{ij}. \right) \nonumber \\
=&\ \frac{S_{kh}^2}{n_k^2 p_h^2} \left( \sum_{i=1}^{n_k} \sum_{j=1}^{p_h} Z^{(k, h)}_{ij} \right)^2. 
% [Revision 2020/9] --->
\end{align}

From (\ref{eq:sigma_tilde_pp}), we have 
\begin{align}
\label{eq:sigma_tilde_e_pp}
\mathbb{E} \left[ \left( B_{kh} - \tilde{B}_{kh} \right)^2 \right] 
% [Revision 2020/9] <---
&= \frac{S_{kh}^2}{n_k^2 p_h^2} \mathbb{E} \left[ \left( \sum_{i=1}^{n_k} \sum_{j=1}^{p_h} Z^{(k, h)}_{ij} \right)^2 \right] 
= \frac{S_{kh}^2}{n_k^2 p_h^2} \mathbb{E} \left[ X_{kh}^2 \right], 
% [Revision 2020/9] --->
\end{align}
% [Revision 2020/9] <---
where $X_{kh} \equiv \sum_{i=1}^{n_k} \sum_{j=1}^{p_h} Z^{(k, h)}_{ij}$. Here, $\mathbb{E} [X_{kh}] = 0$, and $\mathbb{E} \left[ X_{kh}^2 \right] = \mathbb{V} [X_{kh}] = n_k p_h$. By substituting this into (\ref{eq:sigma_tilde_e_pp}), we obtain 
% [Revision 2020/9] --->
\begin{eqnarray}
\label{eq:sigma_tilde_e_pp2}
\mathbb{E} \left[ \left( B_{kh} - \tilde{B}_{kh} \right)^2 \right] = \frac{S_{kh}^2}{n_k p_h}. 
\end{eqnarray}

From Markov's inequality and (\ref{eq:sigma_tilde_e_pp2}), 
\begin{align}
\label{eq:sigma_tilde_markov}
% [Revision 2020/9] <---
&\forall t>0, \ \mathrm{Pr} \left[ \left( B_{kh} - \tilde{B}_{kh} \right)^2 \geq t \right] \leq \frac{S_{kh}^2}{n_k p_h} \frac{1}{t} \nonumber \\
\iff \ &\forall t'>0, \ \mathrm{Pr} \left[ \left( B_{kh} - \tilde{B}_{kh} \right)^2 \geq \frac{S_{kh}^2}{n_k p_h} t' \right] \leq \frac{1}{t'}. 
% [Revision 2020/9] --->
\end{align}
Therefore, we have 
\begin{eqnarray}
\label{eq:sigma_tilde_term2}
\left( B_{kh} - \tilde{B}_{kh} \right)^2 = O_p \left( \frac{1}{m^2} \right). 
\end{eqnarray}

By combining (\ref{eq:sigma_tilde_diff}), (\ref{eq:sigma_tilde_term1}), and (\ref{eq:sigma_tilde_term2}), 
\begin{eqnarray}
\label{eq:sigma_tilde_diff2}
% [Revision 2020/9] <---
\left| \tilde{S}_{kh}^2 - S_{kh}^2 \right| = O_p \left( \frac{1}{m} \right). 
% [Revision 2020/9] --->
\end{eqnarray}
The difference between $\tilde{S}_{kh}$ and $S_{kh}$ is given by 
\begin{eqnarray}
\label{eq:sigma_tilde_diff3}
% [Revision 2020/9] <---
\left| \tilde{S}_{kh} - S_{kh} \right| = \frac{\left| \tilde{S}_{kh}^2 - S_{kh}^2 \right|}{\tilde{S}_{kh} + S_{kh}}. 
% [Revision 2020/9] --->
\end{eqnarray}
% [Revision 2020/9] <---
Here, from (\ref{eq:sigma_tilde_diff2}), $m \left| \tilde{S}_{kh}^2 - S_{kh}^2 \right|$ is bounded in probability. Therefore, $\tilde{S}_{kh}$ converges in probability to $S_{kh}$. By combining this fact with (\ref{eq:sigma_tilde_diff2}) and (\ref{eq:sigma_tilde_diff3}), we finally obtain 
% [Revision 2020/9] --->
\begin{eqnarray}
% [Revision 2020/9] <---
\left| \tilde{S}_{kh} - S_{kh} \right| = O_p \left( \frac{1}{m} \right), 
% [Revision 2020/9] --->
\end{eqnarray}
which concludes the proof. 
\end{proof}

%=======================================================================
\newpage
% [Revision 2020/9] <---
\section{Proof of $\| \hat{Z} \|_{\mathrm{op}}^2 = \| Z \|_{\mathrm{op}}^2 + O_p \left( m^{\frac{2}{7}} \right)$ in realizable case}
% [Revision 2020/9] --->
\label{sec:ap_lambda_diff}

We first derive the relationship between the maximum eigenvalues of matrices $Z^{\top} Z$ and $\tilde{Z}^{\top} \tilde{Z}$ in Lemma \ref{lm:zop2_eq1} and \ref{lm:zop2_eq2}. 

% - - - - - - - - - - - - - - - - - - - - - - - - - - - - - - - - - - - - - - - - - - - - - - - - - - - - - - - - - - - - - - - - - - - - - - - -

\begin{lemma}
\label{lm:zop2_eq1}

Let $\lambda_1$ and $\tilde{\lambda}_1$, respectively, be the maximum eigenvalues of matrices $Z^{\top} Z$ and $\tilde{Z}^{\top} \tilde{Z}$ (i.e., $\| Z \|_{\mathrm{op}}^2$ and $\| \tilde{Z} \|_{\mathrm{op}}^2$, respectively). Then, for all $\epsilon \in \left( 0, \frac{1}{2} \right)$, the following equation holds: 
\begin{align}
% [Revision 2020/9] <---
\lambda_1 \leq \tilde{\lambda}_1 + O_p \left( m^{\epsilon} \right). 
% [Revision 2020/9] --->
\end{align}
\end{lemma}

\begin{proof}
Let $\bm{v}$ and $\tilde{\bm{v}}$, respectively, be the normalized eigenvectors of $Z^{\top} Z$ and $\tilde{Z}^{\top} \tilde{Z}$, corresponding to the maximum eigenvalues $\lambda_1$ and $\tilde{\lambda}_1$: 
\begin{align}
Z^{\top} Z \bm{v} = \lambda_1 \bm{v}, \ \ \ \| \bm{v} \| = 1, \nonumber \\
\tilde{Z}^{\top} \tilde{Z} \tilde{\bm{v}} = \tilde{\lambda}_1 \tilde{\bm{v}}, \ \ \ \| \tilde{\bm{v}} \| = 1. 
\end{align}

Since $\sqrt{\tilde{\lambda}_1}$ is the largest singular value of matrix $\tilde{Z}$, we have
\begin{align}
\label{eq:tilde_sqrt_lmd_entire}
&\sqrt{\tilde{\lambda}_1} = \sup_{\bm{u}} \frac{\| \tilde{Z} \bm{u} \|}{\| \bm{u} \|} \geq \frac{\| \tilde{Z} \bm{v} \|}{\| \bm{v} \|} = \| \tilde{Z} \bm{v} \| 
\iff \tilde{\lambda}_1 \geq \| \tilde{Z} \bm{v} \|^2. 
\end{align}

We also define the following matrix $Q^{(k, h)}$ for each $(k, h)$th block: 
\begin{align}
\label{eq:q_kh}
Q^{(k, h)} &\equiv Z^{(k, h)} - \frac{\tilde{S}_{kh}}{S_{kh}} \tilde{Z}^{(k, h)} = \frac{1}{S_{kh}} \left( \tilde{P}^{(k, h)} - P^{(k, h)} \right) \nonumber \\
% [Revision 2020/9] <---
&= \frac{1}{n_k p_h} \left[ \sum_{i \in I_k, j \in J_h} \frac{A_{ij} - P_{ij}}{S_{kh}} \right] \begin{bmatrix}
1 & \cdots & 1 \\
\vdots & & \vdots \\
1 & \cdots & 1
\end{bmatrix} \nonumber \\
&= \left( \frac{1}{n_k p_h} \sum_{i = 1}^{n_k} \sum_{j = 1}^{p_h} Z^{(k, h)}_{ij} \right) \begin{bmatrix}
1 & \cdots & 1 \\
\vdots & & \vdots \\
1 & \cdots & 1
\end{bmatrix}, 
\end{align}
where $n_k$ and $p_h$, respectively, are the row and column sizes of the $(k, h)$th \textbf{null} block. 
% [Revision 2020/9] --->

% [Revision 2020/9] <---
Let $\underline{Z}^{(k, h)}$, $\underline{\tilde{Z}}^{(k, h)}$, and $\underline{Q}^{(k, h)}$, respectively be $n \times p$ matrices whose $(k, h)$th null blocks are $Z^{(k, h)}$, $\tilde{Z}^{(k, h)}$ and $Q^{(k, h)}$ and whose all the other entries are zero. As shown in Figure \ref{fig:q_block}, we define matrix $Q$ as $Q \equiv \sum_{k=1}^K \sum_{h=1}^H \underline{Q}^{(k, h)}$. 
% [Revision 2020/9] --->

From (\ref{eq:tilde_sqrt_lmd_entire}), we have
\begin{align}
\label{eq:tilde_lambda_entire0}
% [Revision 2020/9] <---
\tilde{\lambda}_1 &\geq \| \tilde{Z} \bm{v} \|^2 = \| \sum_{k=1}^K \sum_{h=1}^H \underline{\tilde{Z}}^{(k, h)} \bm{v} \|^2 
= \| \sum_{k=1}^K \sum_{h=1}^H \frac{S_{kh}}{\tilde{S}_{kh}} (\underline{Z}^{(k, h)} - \underline{Q}^{(k, h)}) \bm{v} \|^2 \nonumber \\
&\geq \left[ \| \sum_{k=1}^K \sum_{h=1}^H (\underline{Z}^{(k, h)} - \underline{Q}^{(k, h)}) \bm{v} \| - \| \sum_{k=1}^K \sum_{h=1}^H \left( 1 - \frac{S_{kh}}{\tilde{S}_{kh}} \right) (\underline{Z}^{(k, h)} - \underline{Q}^{(k, h)}) \bm{v} \| \right]^2 \nonumber \\
&= \left[ \| (Z - Q) \bm{v} \| - \| \sum_{k=1}^K \sum_{h=1}^H \left( 1 - \frac{S_{kh}}{\tilde{S}_{kh}} \right) (\underline{Z}^{(k, h)} - \underline{Q}^{(k, h)}) \bm{v} \| \right]^2 \nonumber \\
&\geq \| (Z - Q) \bm{v} \|^2 - 2 \| (Z - Q) \bm{v} \| \| \sum_{k=1}^K \sum_{h=1}^H \left( 1 - \frac{S_{kh}}{\tilde{S}_{kh}} \right) (\underline{Z}^{(k, h)} - \underline{Q}^{(k, h)}) \bm{v} \| \nonumber \\
&\geq \| (Z - Q) \bm{v} \|^2 - 2 \| (Z - Q) \bm{v} \| \left[ \sum_{k=1}^K \sum_{h=1}^H \left| 1 - \frac{S_{kh}}{\tilde{S}_{kh}} \right| \| (\underline{Z}^{(k, h)} - \underline{Q}^{(k, h)}) \bm{v} \| \right] \nonumber \\
&\geq \| (Z - Q) \bm{v} \|^2 - 2 \| (Z - Q) \bm{v} \| \left[ \sum_{k=1}^K \sum_{h=1}^H \left| 1 - \frac{S_{kh}}{\tilde{S}_{kh}} \right| \left( \| \underline{Z}^{(k, h)} \bm{v} \| + \| \underline{Q}^{(k, h)} \bm{v} \| \right) \right] \nonumber \\
&\geq \| (Z - Q) \bm{v} \|^2 - 2 \| (Z - Q) \bm{v} \| \left[ \sum_{k=1}^K \sum_{h=1}^H \left| 1 - \frac{S_{kh}}{\tilde{S}_{kh}} \right| \left( \sqrt{\lambda_1^{(k, h)}} + \| \underline{Q}^{(k, h)} \bm{v} \| \right) \right] \nonumber \\
&\geq \lambda_1 - 2 \sqrt{\lambda_1} \| Q \bm{v} \| - 2 (\sqrt{\lambda_1} + \| Q \bm{v} \|) \left[ \sum_{k=1}^K \sum_{h=1}^H \left| 1 - \frac{S_{kh}}{\tilde{S}_{kh}} \right| \left( \sqrt{\lambda_1^{(k, h)}} + \| \underline{Q}^{(k, h)} \bm{v} \| \right) \right], 
\end{align}
where $\lambda_1^{(k, h)}$ is the maximum eigenvalue of matrix $\left( Z^{(k, h)} \right)^{\top} Z^{(k, h)}$. 
% [Revision 2020/9] --->

\begin{figure}[t]
  \centering
  \includegraphics[width=0.8\hsize]{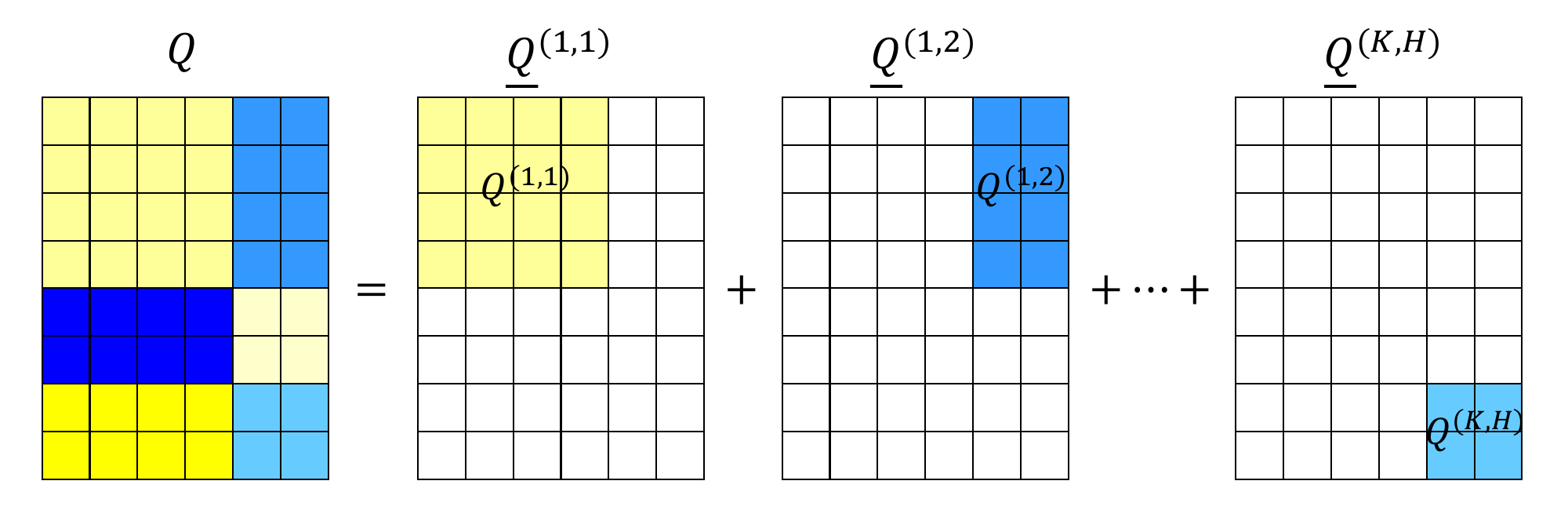}
  \caption{Definition of matrix $Q$. } 
  \label{fig:q_block}
\end{figure}

From now on, we prove that $\| Q \bm{v} \| = O_p \left( \frac{1}{\sqrt{m}} \right)$ and $\| \underline{Q}^{(k, h)} \bm{v} \| = O_p \left( \frac{1}{\sqrt{m}} \right)$. We use the following notations: 
\begin{align}
% [Revision 2020/9] <---
\nu_{kh} &\equiv \frac{1}{n_k p_h} \sum_{i=1}^{n_k} \sum_{j=1}^{p_h} Z^{(k, h)}_{ij}, \ \ \ 
\omega_{hh'} \equiv \sum_{k=1}^K n_k \nu_{kh} \nu_{kh'}, \ \ \ 
\zeta_h \equiv \sum_{h'=1}^H \omega_{hh'} \sum_{j \in J_{h'}} v_j, 
\end{align}
where $v_j$ is the $j$th entry of vector $\bm{v}$. 
% [Revision 2020/9] --->
Note that the $(k, h)$th block of matrix $Q$, the $(h, h')$th block of matrix $Q^{\top} Q$, and the $h$th block of vector $Q^{\top} Q \bm{v}$ are given by $\nu_{kh} \begin{bmatrix}
1 & \cdots & 1 \\
\vdots & & \vdots \\
1 & \cdots & 1
\end{bmatrix}$, 
$\omega_{hh'} \begin{bmatrix}
1 & \cdots & 1 \\
\vdots & & \vdots \\
1 & \cdots & 1
\end{bmatrix}$, and $\zeta_h \begin{bmatrix}
1\\
\vdots \\
1
\end{bmatrix}$, respectively. 
% [Revision 2020/9] <---
Let $\bm{u}^{(h)} \in \mathbb{R}^p$ be a vector whose entries in the $h$th column cluster are $\frac{1}{\sqrt{p_h}}$ and whose all the other entries are zero. 
Here, from Theorem 2.17 in \cite{Bloemendal2016}, each $j$th eigenvector $\bm{v}_j$ of matrix $Z^{\top} Z$ has a \textit{delocalization property}, that is, for any constant vector $\bm{w}$ that satisfies $\| \bm{w} \| = 1$, 
\begin{align}
\label{eq:delocalization}
| \bm{v}_j^{\top} \bm{w} | = O_p \left( m^{-\frac{1}{2} + \epsilon} \right), \ \ \ \mathrm{for\ all}\ \epsilon > 0. 
\end{align}
From Theorem 2.20 in \cite{Bloemendal2016}, (\ref{eq:delocalization}) holds uniformly in $j$ and $\bm{w}$. 

% [Revision 2020/9] --->
Note that this theorem holds in our case, where we assume that $n, p \propto m$ and that each entry of $Z$ is independently generated from a distribution with zero mean and unit variance that satisfies the sub-exponential condition. 
% [Revision 2020/9] <---
Therefore, we have $| \bm{v}^{\top} \bm{u}^{(h)} | = O_p \left( m^{-\frac{1}{2} + \epsilon} \right)$ for all $\epsilon > 0$. Since $Q^{\top} Q \bm{v} = \sum_{h=1}^H \zeta_h \sqrt{p_h} \bm{u}^{(h)}$ and $\nu_{kh} = O_p \left( \frac{1}{m} \right)$, $\omega_{hh'} = O_p \left( \frac{1}{m} \right)$, $\zeta_h = \sum_{h'=1}^H \omega_{hh'} \sqrt{p_{h'}} \bm{v}^{\top} \bm{u}^{(h')} = O_p \left( m^{-1+\epsilon} \right)$ for all $\epsilon > 0$ by definition, the following equation holds: 
\begin{align}
\label{eq:qv_norm_entire}
&\| Q \bm{v} \|^2 = \sum_{h=1}^H \zeta_h \sqrt{p_h} \bm{v}^{\top} \bm{u}^{(h)} = O_p \left( m^{-1+2\epsilon} \right) \nonumber \\
&\iff \| Q \bm{v} \| = O_p \left( m^{-\frac{1}{2} + \epsilon} \right), \ \ \ \mathrm{for\ all}\ \epsilon > 0. 
\end{align}

Similarly, $(h, h)$th block of matrix $(\underline{Q}^{(k, h)})^{\top} \underline{Q}^{(k, h)}$ is $n_k \nu_{kh}^2 \begin{bmatrix}
1 & \cdots & 1 \\
\vdots & & \vdots \\
1 & \cdots & 1
\end{bmatrix}$, and its all the other entries are zero, which results in that $(\underline{Q}^{(k, h)})^{\top} \underline{Q}^{(k, h)} \bm{v} = n_k \nu_{kh}^2 \left( \sum_{j \in J_h} v_j \right) \sqrt{p_h} \bm{u}^{(h)}$. Therefore, we have
\begin{align}
\label{eq:qkhv_norm_entire}
\| \underline{Q}^{(k, h)} \bm{v} \|^2 &= n_k \nu_{kh}^2 \left( \sum_{j \in J_h} v_j \right) \sqrt{p_h} \bm{v}^{\top} \bm{u}^{(h)} = n_k \nu_{kh}^2 \left( \sqrt{p_h} \bm{v}^{\top} \bm{u}^{(h)} \right)^2 \nonumber \\
&= O_p \left( m^{-1+2\epsilon} \right) \nonumber \\
\iff& \| \underline{Q}^{(k, h)} \bm{v} \| = O_p \left( m^{-\frac{1}{2} + \epsilon} \right), \ \ \ \mathrm{for\ all}\ \epsilon > 0. 
\end{align}

Moreover, from Lemma \ref{lm:sigma_tilde}, we have
% [Revision 2020/9] --->
\begin{align}
\label{eq:1_s_tildes}
\left| 1 - \frac{S_{kh}}{\tilde{S}_{kh}} \right| = O_p \left( \frac{1}{m} \right). 
\end{align}

By substituting (\ref{eq:qv_norm_entire}), (\ref{eq:qkhv_norm_entire}), and (\ref{eq:1_s_tildes}) into (\ref{eq:tilde_lambda_entire0}) and by setting $\epsilon < \frac{1}{2}$, we have 
\begin{align}
\tilde{\lambda}_1 &\geq \lambda_1 - O_p \left( m^{\epsilon} \right) - O_p \left( m^{\frac{1}{2}} \right) \left[ \sum_{k=1}^K \sum_{h=1}^H O_p \left( m^{-1} \right) O_p \left( m^{\frac{1}{2}} \right) \right] \nonumber \\
&= \lambda_1 - O_p \left( m^{\epsilon} \right), 
\end{align}
which concludes the proof. 
\end{proof}

% - - - - - - - - - - - - - - - - - - - - - - - - - - - - - - - - - - - - - - - - - - - - - - - - - - - - - - - - - - - - - - - - - - - - - - - -

\begin{lemma}
\label{lm:zop2_eq2}
Let $\lambda_1$ and $\tilde{\lambda}_1$, respectively, be the maximum eigenvalues of matrices $Z^{\top} Z$ and $\tilde{Z}^{\top} \tilde{Z}$. Then, the following equation holds: 
\begin{align}
\label{eq:lemma_zop2_1_4}
% [Revision 2020/9] <---
\tilde{\lambda}_1 \leq \lambda_1 + O_p \left( m^{\frac{2}{7} + \epsilon} \right),\ \ \ \mathrm{for\ all}\ \epsilon \in \left( 0, \frac{2}{7} \right). 
% [Revision 2020/9] --->
\end{align}
\end{lemma}

\begin{proof}
% [Revision 2020/9] <---
We use the same notations as in Lemmas \ref{lm:zop2_eq1}. Let $\{ \lambda_j \}$ and $\{ \bm{v}_j \}$, respectively, be the sets of the eigenvalues and the corresponding normalized eigenvectors (i.e., $\| \bm{v}_j \| = 1$ for all $j$) of matrix $Z^{\top} Z$, where $j = 1, \dots, p$ and $\lambda_1 \geq \lambda_2 \geq \dots \geq \lambda_p$. We also define that $\tau_{kh} \equiv \frac{S_{kh}}{\tilde{S}_{kh}}$. Note that $|\tau_{kh} - 1| = O_p \left( \frac{1}{m} \right)$ from (\ref{eq:1_s_tildes}). Let $\tilde{\bm{v}}^{(h)} \in \mathbb{R}^{p_h}$ be a subvector of $\tilde{\bm{v}}$ in the $h$th column cluster. 
% [Revision 2020/9] --->

Since $Z^{\top} Z$ is a symmetric matrix, its eigenvectors $\{ \bm{v}_j \}$ form an orthonormal system, and thus there exists a unique set of coefficients $\{ c_j \}$ that satisfies 
\begin{align}
\tilde{\bm{v}} = \sum_{j = 1}^p c_j \bm{v}_j = \tilde{\bm{v}}_1 + \tilde{\bm{v}}_2, 
\end{align}
where
\begin{align}
\label{eq:tilde_v1v2}
&\tilde{\bm{v}}_1 \equiv \sum_{j = 1}^t c_j \bm{v}_j, \ \ \ 
\tilde{\bm{v}}_2 \equiv \sum_{j = t + 1}^p c_j \bm{v}_j, \nonumber \\
&\lambda_t \geq \lambda_1 - n^d, \ \ \ 
\lambda_{t + 1} < \lambda_1 - n^d, \ \ \ 
d = \frac{5}{7}. 
\end{align}

Therefore, the following equation holds: 
\begin{align}
\label{eq:tilde_lambda_u1u2}
% [Revision 2020/9] <---
&\tilde{\lambda}_1 = \tilde{\bm{v}}^{\top} \tilde{Z}^{\top} \tilde{Z} \tilde{\bm{v}} 
= \| \sum_{k=1}^K \sum_{h=1}^H \left[ \underline{Z}^{(k, h)} + (\tau_{kh} - 1) \underline{Z}^{(k, h)} - \tau_{kh} \underline{Q}^{(k, h)} \right] \tilde{\bm{v}} \|^2 \nonumber \\
&= \| \left\{ Z + \sum_{k=1}^K \sum_{h=1}^H \left[ (\tau_{kh} - 1) \underline{Z}^{(k, h)} - \tau_{kh} \underline{Q}^{(k, h)} \right] \right\} \tilde{\bm{v}} \|^2 \nonumber \\
&= \| Z \tilde{\bm{v}} \|^2 + 2 \tilde{\bm{v}}^{\top} Z^{\top} \sum_{k=1}^K \sum_{h=1}^H \left[ (\tau_{kh} - 1) \underline{Z}^{(k, h)} - \tau_{kh} \underline{Q}^{(k, h)} \right] \tilde{\bm{v}} \nonumber \\
&+ \| \sum_{k=1}^K \sum_{h=1}^H \left[ (\tau_{kh} - 1) \underline{Z}^{(k, h)} - \tau_{kh} \underline{Q}^{(k, h)} \right] \tilde{\bm{v}} \|^2 \nonumber \\
&\leq \| Z \tilde{\bm{v}} \|^2 + 2 \sqrt{\lambda_1} \sum_{k=1}^K \sum_{h=1}^H |\tau_{kh} - 1| \| \underline{Z}^{(k, h)} \tilde{\bm{v}} \| - 2 \tilde{\bm{v}}^{\top} Z^{\top} \sum_{k=1}^K \sum_{h=1}^H \tau_{kh} \underline{Q}^{(k, h)} \tilde{\bm{v}} \nonumber \\
&+ \| \sum_{k=1}^K \sum_{h=1}^H \left[ (\tau_{kh} - 1) \underline{Z}^{(k, h)} - \tau_{kh} \underline{Q}^{(k, h)} \right] \tilde{\bm{v}} \|^2 \nonumber \\
&= \| Z \tilde{\bm{v}} \|^2 + 2 \sqrt{\lambda_1} \sum_{k=1}^K \sum_{h=1}^H |\tau_{kh} - 1| \| Z^{(k, h)} \tilde{\bm{v}}^{(h)} \| - 2 \tilde{\bm{v}}^{\top} Z^{\top} \sum_{k=1}^K \sum_{h=1}^H \tau_{kh} \underline{Q}^{(k, h)} \tilde{\bm{v}} \nonumber \\
&+ \| \sum_{k=1}^K \sum_{h=1}^H \left[ (\tau_{kh} - 1) \underline{Z}^{(k, h)} - \tau_{kh} \underline{Q}^{(k, h)} \right] \tilde{\bm{v}} \|^2 \nonumber \\
&= \| Z \tilde{\bm{v}} \|^2 + O_p (\sqrt{m}) \sum_{k=1}^K \sum_{h=1}^H O_p \left( \frac{1}{m} \right) O_p (\sqrt{m}) - 2 \tilde{\bm{v}}^{\top} Z^{\top} \sum_{k=1}^K \sum_{h=1}^H \tau_{kh} \underline{Q}^{(k, h)} \tilde{\bm{v}} \nonumber \\
&+ \| \sum_{k=1}^K \sum_{h=1}^H \left[ (\tau_{kh} - 1) \underline{Z}^{(k, h)} - \tau_{kh} \underline{Q}^{(k, h)} \right] \tilde{\bm{v}} \|^2 \nonumber \\
&= \| Z \tilde{\bm{v}} \|^2 + O_p (1) - 2 \tilde{\bm{v}}^{\top} Z^{\top} \sum_{k=1}^K \sum_{h=1}^H \tau_{kh} \underline{Q}^{(k, h)} \tilde{\bm{v}} \nonumber \\
&+ \| \sum_{k=1}^K \sum_{h=1}^H \left[ (\tau_{kh} - 1) \underline{Z}^{(k, h)} - \tau_{kh} \underline{Q}^{(k, h)} \right] \tilde{\bm{v}} \|^2\ \ \ (\because K\ \mathrm{and}\ H\ \mathrm{are\ fixed\ constants}) \nonumber \\
&\leq \| Z \tilde{\bm{v}} \|^2 + O_p (1) - 2 \tilde{\bm{v}}^{\top} Z^{\top} \sum_{k=1}^K \sum_{h=1}^H \tau_{kh} \underline{Q}^{(k, h)} \tilde{\bm{v}} \nonumber \\
&+ \left( \| \sum_{k=1}^K \sum_{h=1}^H (\tau_{kh} - 1) \underline{Z}^{(k, h)} \tilde{\bm{v}} \| + \| \sum_{k=1}^K \sum_{h=1}^H \tau_{kh} \underline{Q}^{(k, h)} \tilde{\bm{v}} \| \right)^2 \nonumber \\
&\leq \| Z \tilde{\bm{v}} \|^2 + O_p (1) - 2 \tilde{\bm{v}}^{\top} Z^{\top} \sum_{k=1}^K \sum_{h=1}^H \tau_{kh} \underline{Q}^{(k, h)} \tilde{\bm{v}} \nonumber \\
&+ \left( \sum_{k=1}^K \sum_{h=1}^H |\tau_{kh} - 1| \| \underline{Z}^{(k, h)} \tilde{\bm{v}} \| + \sum_{k=1}^K \sum_{h=1}^H \tau_{kh} \| \underline{Q}^{(k, h)} \tilde{\bm{v}} \| \right)^2 \nonumber \\
&= \| Z \tilde{\bm{v}} \|^2 + O_p (1) - 2 \sum_{k=1}^K \sum_{h=1}^H \tau_{kh} \tilde{\bm{v}}^{\top} Z^{\top} \underline{Q}^{(k, h)} \tilde{\bm{v}} \nonumber \\
&+ \left[ O_p \left( \frac{1}{\sqrt{m}} \right) + \sum_{k=1}^K \sum_{h=1}^H \tau_{kh} \| \underline{Q}^{(k, h)}\tilde{\bm{v}} \| \right]^2\ \ \ (\because K\ \mathrm{and}\ H\ \mathrm{are\ fixed\ constants}). 
\end{align}
% [Revision 2020/9] --->

As for the last term in (\ref{eq:tilde_lambda_u1u2}), the following equation holds: 
\begin{align}
\label{eq:Qv_norm2}
% [Revision 2020/9] <---
\| \underline{Q}^{(k, h)} \tilde{\bm{v}} \|^2 &= \tilde{\bm{v}}^{\top} (\underline{Q}^{(k, h)})^{\top} \underline{Q}^{(k, h)} \tilde{\bm{v}} 
= \sum_{j=1}^p \sum_{j'=1}^p c_j c_{j'} \bm{v}_j^{\top} (\underline{Q}^{(k, h)})^{\top} \underline{Q}^{(k, h)} \bm{v}_{j'} \nonumber \\
&= O_p \left( \frac{1}{m} \right) \sum_{j=1}^p \sum_{j'=1}^p c_j c_{j'} \bm{v}_j^{\top} p_h (\bm{v}_{j'}^{\top} \bm{u}^{(h)}) \bm{u}^{(h)} \nonumber \\
&= O_p (1) \left[ \sum_{j=1}^p c_j (\bm{v}_j^{\top} \bm{u}^{(h)}) \right]^2 
\leq O_p (1) \left[ \sqrt{\sum_{j=1}^p c_j^2} \sqrt{\sum_{j=1}^p (\bm{v}_j^{\top} \bm{u}^{(h)})^2} \right]^2 \nonumber \\
&= O_p (1) \left( \sum_{j=1}^p c_j^2 \right) \left[ \sum_{j=1}^p (\bm{v}_j^{\top} \bm{u}^{(h)})^2 \right] 
= O_p (1) \| \tilde{\bm{v}} \|^2 \left[ \sum_{j=1}^p (\bm{v}_j^{\top} \bm{u}^{(h)})^2 \right] \nonumber \\
&= O_p (1) \left[ \sum_{j=1}^p (\bm{v}_j^{\top} \bm{u}^{(h)})^2 \right] \nonumber \\
&= O_p \left( m^{2\epsilon} \right)\ \ \ (\because (\ref{eq:delocalization})\ \mathrm{holds\ uniformly\ in}\ j\ \mathrm{and}\ \bm{w}) \\
\label{eq:Qv_norm}
\iff& \| \underline{Q}^{(k, h)} \tilde{\bm{v}} \| = O_p \left( m^{\epsilon} \right),\ \ \ \mathrm{for\ all}\ \epsilon > 0. 
\end{align}
where $\bm{u}^{(h)} \in \mathbb{R}^p$ is a vector whose elements in the $h$th column cluster is $\frac{1}{\sqrt{p_h}}$ and whose all the other elements are zero. 
In the last equation in (\ref{eq:Qv_norm2}), we used the delocalization property of $\{ \bm{v}_j \}$, which are eigenvectors of matrix $Z^{\top} Z$. 
% [Revision 2020/9] --->
By substituting (\ref{eq:Qv_norm}) into (\ref{eq:tilde_lambda_u1u2}) and using the fact that $\tau_{kh} = 1 + O_p \left( \frac{1}{m} \right)$ and the assumption that $K$ and $H$ are fixed constants, we have 
\begin{align}
\label{eq:tilde_lambda_u1u2_2}
% [Revision 2020/9] <---
&\tilde{\lambda}_1 \leq \| Z \tilde{\bm{v}} \|^2 - 2 \sum_{k=1}^K \sum_{h=1}^H \tau_{kh} \tilde{\bm{v}}^{\top} Z^{\top} \underline{Q}^{(k, h)} \tilde{\bm{v}} + O_p \left( m^{2\epsilon} \right),\ \ \ \mathrm{for\ all}\ \epsilon > 0. 
% [Revision 2020/9] --->
\end{align}

Here, by definition in (\ref{eq:tilde_v1v2}), the following equation holds: 
\begin{align}
\label{eq:Zv_vZQv}
% [Revision 2020/9] <---
&\| Z \tilde{\bm{v}} \|^2 - 2 \sum_{k=1}^K \sum_{h=1}^H \tau_{kh} \tilde{\bm{v}}^{\top} Z^{\top} \underline{Q}^{(k, h)} \tilde{\bm{v}} \nonumber \\
&= (\tilde{\bm{v}}_1 + \tilde{\bm{v}}_2)^{\top} Z^{\top} Z (\tilde{\bm{v}}_1 + \tilde{\bm{v}}_2) - 2 \sum_{k=1}^K \sum_{h=1}^H \tau_{kh} \tilde{\bm{v}}^{\top} Z^{\top} \underline{Q}^{(k, h)} \tilde{\bm{v}} \nonumber \\
&= \tilde{\bm{v}}_1^{\top} Z^{\top} Z \tilde{\bm{v}}_1 + \tilde{\bm{v}}_2^{\top} Z^{\top} Z \tilde{\bm{v}}_2 - 2 \sum_{k=1}^K \sum_{h=1}^H \tau_{kh} \tilde{\bm{v}}^{\top} Z^{\top} \underline{Q}^{(k, h)} \tilde{\bm{v}} \nonumber \\
&= \left( \sum_{j = 1}^t c_j \bm{v}_j \right)^{\top} \sum_{j = 1}^t c_j \lambda_j \bm{v}_j + \left( \sum_{j = t + 1}^p c_j \bm{v}_j \right)^{\top} \sum_{j = t + 1}^p c_j \lambda_j \bm{v}_j - 2 \sum_{k=1}^K \sum_{h=1}^H \tau_{kh} \tilde{\bm{v}}^{\top} Z^{\top} \underline{Q}^{(k, h)} \tilde{\bm{v}} \nonumber \\
&= \sum_{j = 1}^t c_j^2 \lambda_j + \sum_{j = t + 1}^p c_j^2 \lambda_j - 2 \sum_{k=1}^K \sum_{h=1}^H \tau_{kh} \tilde{\bm{v}}^{\top} Z^{\top} \underline{Q}^{(k, h)} \tilde{\bm{v}}\ \ \ (\because \| \bm{v}_j \| = 1\ \mathrm{for\ all}\ j) \nonumber \\
&\leq \lambda_1 \sum_{j = 1}^t c_j^2 + (\lambda_1 - n^d) \sum_{j = t + 1}^p c_j^2 - 2 \sum_{k=1}^K \sum_{h=1}^H \tau_{kh} \tilde{\bm{v}}^{\top} Z^{\top} \underline{Q}^{(k, h)} \tilde{\bm{v}} \nonumber \\
&= \lambda_1 \| \tilde{\bm{v}} \|^2 - n^d \| \tilde{\bm{v}}_2 \|^2 - 2 \sum_{k=1}^K \sum_{h=1}^H \tau_{kh} \tilde{\bm{v}}^{\top} Z^{\top} \underline{Q}^{(k, h)} (\tilde{\bm{v}}_1 + \tilde{\bm{v}}_2) \nonumber \\
&= \lambda_1 - n^d \| \tilde{\bm{v}}_2 \|^2 - 2 \sum_{k=1}^K \sum_{h=1}^H \tau_{kh} \tilde{\bm{v}}^{\top} Z^{\top} \underline{Q}^{(k, h)} \tilde{\bm{v}}_1 - 2 \sum_{k=1}^K \sum_{h=1}^H \tau_{kh} \tilde{\bm{v}}^{\top} Z^{\top} \underline{Q}^{(k, h)} \tilde{\bm{v}}_2. 
% [Revision 2020/9] --->
\end{align}

The third term in (\ref{eq:Zv_vZQv}) can be upper bounded as follows: 
\begin{align}
\label{vZQv1_upper}
&- \tilde{\bm{v}}^{\top} Z^{\top} \underline{Q}^{(k, h)} \tilde{\bm{v}}_1 \leq | \tilde{\bm{v}}^{\top} Z^{\top} \underline{Q}^{(k, h)} \tilde{\bm{v}}_1 | 
% [Revision 2020/9] <---
= \left| \tilde{\bm{v}}^{\top} Z^{\top} \underline{Q}^{(k, h)} \left( \sum_{j = 1}^t c_j \bm{v}_j \right) \right| \nonumber \\
&= | \sum_{j = 1}^t c_j \tilde{\bm{v}}^{\top} Z^{\top} \underline{Q}^{(k, h)} \bm{v}_j | \nonumber \\
&= | \sum_{j = 1}^t c_j O_p \left( \frac{1}{m} \right) \left[ \sqrt{n_k} \left( \bm{\tilde{w}}^{(k)} \right)^{\top} Z \tilde{\bm{v}} \right]^{\top} \sqrt{p_h} \left( \bm{\tilde{u}}^{(h)} \right)^{\top} \bm{v}_j | \nonumber \\
&= | \sum_{j = 1}^t c_j O_p (1) \left( \tilde{\bm{v}}^{\top} Z^{\top} \bm{\tilde{w}}^{(k)} \right) \left( \bm{\tilde{u}}^{(h)} \right)^{\top} \bm{v}_j | \nonumber \\
&= | \sum_{j = 1}^t c_j O_p \left( m^{-\frac{1}{2} + \epsilon} \right) (\tilde{\bm{v}}^{\top} Z^{\top} \bm{\tilde{w}}^{(k)}) |
= | \sum_{j = 1}^t c_j O_p \left( m^{-\frac{1}{2} + \epsilon} \right) O_p (\sqrt{m}) | \nonumber \\
&= | \sum_{j = 1}^t c_j |\ O_p \left( m^{\epsilon} \right)\ \ \ (\because (\ref{eq:delocalization})\ \mathrm{holds\ uniformly\ in}\ j\ \mathrm{and}\ \bm{w}) \nonumber \\
&\leq \sqrt{\sum_{j = 1}^t c_j^2} \sqrt{\sum_{j = 1}^t 1^2}\ O_p \left( m^{\epsilon} \right) 
\leq \| \tilde{\bm{v}} \| \sqrt{t}\ O_p \left( m^{\epsilon} \right) 
= \sqrt{t}\ O_p \left( m^{\epsilon} \right),\ \ \ \mathrm{for\ all}\ \epsilon > 0, 
\end{align}
where $\bm{\tilde{w}}^{(k)} \in \mathbb{R}^n$ is a vector whose elements in the $k$th row cluster is $\frac{1}{\sqrt{n_k}}$ and whose all the other elements are zero. Here we used the delocalization property of $\{ \bm{v}_j \}$, which are eigenvectors of matrix $Z^{\top} Z$. 
% [Revision 2020/9] --->

The fourth term in (\ref{eq:Zv_vZQv}) can also be upper bounded as follows: 
\begin{align}
\label{vZQv2_upper}
- \tilde{\bm{v}}^{\top} Z^{\top} \underline{Q}^{(k, h)} \tilde{\bm{v}}_2 &\leq | \tilde{\bm{v}}^{\top} Z^{\top} \underline{Q}^{(k, h)} \tilde{\bm{v}}_2 | 
\leq \| \tilde{\bm{v}} \| \| Z^{\top} \underline{Q}^{(k, h)} \tilde{\bm{v}}_2 \| 
= \| Z^{\top} \underline{Q}^{(k, h)} \tilde{\bm{v}}_2 \| \nonumber \\
&\leq \| Z^{\top} \underline{Q}^{(k, h)} \|_{\mathrm{op}} \| \tilde{\bm{v}}_2 \| 
\leq \| Z \|_{\mathrm{op}} \| \underline{Q}^{(k, h)} \|_{\mathrm{F}} \| \tilde{\bm{v}}_2 \| \nonumber \\
% [Revision 2020/9] <---
&= O_p (\sqrt{m}) O_p \left( \frac{1}{m} \right) \sqrt{n_k p_h} \| \tilde{\bm{v}}_2 \| 
% [Revision 2020/9] --->
= \| \tilde{\bm{v}}_2 \| O_p (\sqrt{m}). 
\end{align}

By substituting (\ref{vZQv1_upper}) and (\ref{vZQv2_upper}) into (\ref{eq:Zv_vZQv}), we have 
\begin{align}
\label{eq:Zv_vZQv2}
% [Revision 2020/9] <---
&\| Z \tilde{\bm{v}} \|^2 - 2 \sum_{k=1}^K \sum_{h=1}^H \tau_{kh} \tilde{\bm{v}}^{\top} Z^{\top} \underline{Q}^{(k, h)} \tilde{\bm{v}} \nonumber \\
&\leq \lambda_1 - n^d \| \tilde{\bm{v}}_2 \|^2 + 2 \sum_{k=1}^K \sum_{h=1}^H \tau_{kh} \sqrt{t}\ O_p \left( m^{\epsilon} \right) + 2 \sum_{k=1}^K \sum_{h=1}^H \tau_{kh} \| \tilde{\bm{v}}_2 \| O_p (\sqrt{m}) \nonumber \\
&= \lambda_1 - n^d \| \tilde{\bm{v}}_2 \|^2 + 2 \sum_{k=1}^K \sum_{h=1}^H \sqrt{t}\ O_p \left( m^{\epsilon} \right) + 2 \sum_{k=1}^K \sum_{h=1}^H \| \tilde{\bm{v}}_2 \| O_p (\sqrt{m}) \ \ \ (\because \tau_{kh} = 1 + O_p \left( \frac{1}{m} \right)) \nonumber \\
&= \lambda_1 - n^d \| \tilde{\bm{v}}_2 \|^2 + \sqrt{t}\ O_p \left( m^{\epsilon} \right) + \| \tilde{\bm{v}}_2 \| O_p (\sqrt{m}). 
% [Revision 2020/9] --->
\end{align}
In the last equation of (\ref{eq:Zv_vZQv2}), we used the assumption that $K$ and $H$ are fixed constants. 

Let $\nu_j \equiv \frac{1}{n} \lambda_j$ be a normalized eigenvalue of matrix $Z^{\top} Z$. Note that $t$ in (\ref{eq:tilde_v1v2}) is the number of normalized eigenvalues $\{ \nu_j \}$ that satisfy $\nu_j \geq \nu_1 - n^{d-1}$. We also define the following variables: 
\begin{align}
\nu_+ \equiv \left( 1 + \sqrt{\frac{p}{n}} \right)^2, \ \ \ 
\nu_- \equiv \left( 1 - \sqrt{\frac{p}{n}} \right)^2, \ \ \ 
\epsilon_1 \equiv \nu_+ - \nu_1. 
\end{align}
From (4.1) of \cite{Pillai2014}, $|\epsilon_1| = O_p \left( \phi^C m^{-\frac{2}{3}} \right)$ holds for some constant $C>0$, where $\phi \equiv (\log p)^{\log \log p}$. Since $\phi = o(m^{\tilde{\epsilon}_0})$ holds for any $\tilde{\epsilon}_0 > 0$, we have $|\epsilon_1| = O_p \left( m^{-\frac{2}{3} + \epsilon_0} \right)$ for any $\epsilon_0 > 0$. 

Since $\nu_j$ follows the Marcenko–Pastur distribution, whose probability density function is given by 
\begin{align}
q(x) = \frac{1}{2\pi} \frac{n}{p} \frac{\sqrt{\max \{ (\nu_+ - x)(x - \nu_-), 0 \}}}{x}, 
\end{align}
by setting $\epsilon_0 < d - \frac{1}{3}$, we have 
\begin{align}
\label{eq:q_nu_nd}
&q(\nu_1 - n^{d-1}) = q(\nu_+ - n^{d-1} - \epsilon_1) \nonumber \\
&= \frac{\sqrt{\nu_+ - \nu_-}}{\nu_+} \left[ n^{\frac{d-1}{2}} + O_p \left( m^{-\frac{1}{3} + \frac{\epsilon_0}{2}} \right) \right] \left[ 1 + O \left( m^{\frac{d-1}{2}} \right) + O_p \left( m^{-\frac{1}{3} + \frac{\epsilon_0}{2}} \right) \right] \nonumber \\
&= \frac{\sqrt{\nu_+ - \nu_-}}{\nu_+} n^{\frac{d-1}{2}} + O_p \left( m^{\max \{ \frac{d-1}{2}, -\frac{1}{3} + \frac{\epsilon_0}{2} \}} \right) \nonumber \\
&= \frac{\sqrt{\nu_+ - \nu_-}}{\nu_+} n^{\frac{d-1}{2}} + O_p \left( m^{\frac{d-1}{2}} \right) \ \ \ \left( \because \epsilon_0 < d - \frac{1}{3} \right). 
\end{align}

% [Revision 2020/9] <---
From (\ref{eq:q_nu_nd}) and the fact that $|\epsilon_1| = O_p \left( m^{-\frac{2}{3} + \epsilon_0} \right)$ for any $\epsilon_0 > 0$, by setting $\epsilon_0 < d - \frac{1}{3}$, the following equation holds: 
\begin{align}
\bar{n} &\equiv \int_{\nu_1 - n^{d-1}}^{\infty} q(x) \mathrm{d}x 
\leq \left| \int_{\nu_1 - n^{d-1}}^{\nu_+} q(x) \mathrm{d}x \right| + \left| \int_{\nu_+}^{\infty} q(x) \mathrm{d}x \right| \nonumber \\
&= \left| \int_{\nu_1 - n^{d-1}}^{\nu_+} q(x) \mathrm{d}x \right| 
\leq |\epsilon_1 + n^{d-1}|\ q(\nu_1 - n^{d-1}) \nonumber \\
&= O_p \left( m^{d-1} \right) O_p \left( m^{\frac{d-1}{2}} \right) 
= O_p \left( m^{\frac{3(d-1)}{2}} \right). 
\end{align}
% [Revision 2020/9] <---

From (3.7) of \cite{Pillai2014}, the difference between $\bar{n}$ and $\frac{t}{p}$ is given by
\begin{align}
\left| \bar{n} - \frac{t}{p} \right| = O_p \left( m^{-1 + \epsilon_2} \right), \ \ \ \forall \epsilon_2 > 0. 
\end{align}
Therefore, by setting $\epsilon_2 < \frac{3}{2} d - \frac{1}{2}$, we have
\begin{align}
\label{eq:t_op_3_2}
t = O_p \left( m^{\frac{3}{2} d - \frac{1}{2}} \right). 
\end{align}

% [Revision 2020/9] <---
By assumption in (\ref{eq:tilde_v1v2}) that $d = \frac{5}{7}$, the following equation holds for all $\epsilon > 0$: 
\begin{align}
\label{eq:Zv_vZQv2_2}
&\| Z \tilde{\bm{v}} \|^2 - 2 \sum_{k=1}^K \sum_{h=1}^H \tau_{kh} \tilde{\bm{v}}^{\top} Z^{\top} \underline{Q}^{(k, h)} \tilde{\bm{v}} 
\leq \lambda_1 + \| \tilde{\bm{v}}_2 \| \left[ n^{\frac{1}{2}} \varpi - n^d \| \tilde{\bm{v}}_2 \|\right] + O_p \left( m^{\frac{2}{7} + \epsilon} \right), \nonumber \\
&\varpi \equiv n^{-\frac{1}{2}} \| Z \|_{\mathrm{op}} \| \underline{Q}^{(k, h)} \|_{\mathrm{F}} = O_p (1)\ \ \ (\because (\ref{vZQv2_upper})). 
% [Revision 2020/9] --->
\end{align}
Here, we consider the following two patterns: (a) If $n^{\frac{1}{2}} \varpi - n^d \| \tilde{\bm{v}}_2 \| \leq 0$, from (\ref{eq:Zv_vZQv2_2}), we have 
\begin{align}
\label{eq:Zv_op1_3}
% [Revision 2020/9] <---
\| Z \tilde{\bm{v}} \|^2 - 2 \sum_{k=1}^K \sum_{h=1}^H \tau_{kh} \tilde{\bm{v}}^{\top} Z^{\top} \underline{Q}^{(k, h)} \tilde{\bm{v}} 
= \lambda_1 + O_p \left( m^{\frac{2}{7} + \epsilon} \right). 
% [Revision 2020/9] --->
\end{align}
(b) If $n^{\frac{1}{2}} \varpi - n^d \| \tilde{\bm{v}}_2 \| > 0$, we have $\| \tilde{\bm{v}}_2 \| < n^{\frac{1}{2} -d} \varpi$ and thus 
\begin{align}
\label{eq:v_np_nv}
\| \tilde{\bm{v}}_2 \| \left[ n^{\frac{1}{2}} \varpi - n^d \| \tilde{\bm{v}}_2 \|\right] \leq n^{1-d} \varpi^2. 
\end{align}
% [Revision 2020/9] <---
By assumption in (\ref{eq:tilde_v1v2}) that $d = \frac{5}{7}$, we have $n^{1-d} \varpi^2 = O_p \left( m^{\frac{2}{7}} \right)$ and thus (\ref{eq:Zv_op1_3}) holds. 
% [Revision 2020/9] --->

% [Revision 2020/9] <---
In summary, (\ref{eq:Zv_op1_3}) always holds. By combining this fact and (\ref{eq:tilde_lambda_u1u2_2}), we have
\begin{align}
&\tilde{\lambda}_1 \leq \lambda_1 + O_p \left( m^{\frac{2}{7} + \epsilon} \right) + O_p \left( m^{2\epsilon} \right),\ \ \ \mathrm{for\ all}\ \epsilon > 0. 
\end{align}
By setting $\epsilon < \frac{2}{7}$, we finally obtain
\begin{align}
&\tilde{\lambda}_1 \leq \lambda_1 + O_p \left( m^{\frac{2}{7} + \epsilon} \right),\ \ \ \mathrm{for\ all}\ \epsilon \in \left( 0, \frac{2}{7} \right), 
\end{align}
which concludes the proof. 
% [Revision 2020/9] --->
\end{proof}

% - - - - - - - - - - - - - - - - - - - - - - - - - - - - - - - - - - - - - - - - - - - - - - - - - - - - - - - - - - - - - - - - - - - - - - - -

\begin{lemma}
\label{lm:zop2_hat}

Let $\lambda_1$ and $\hat{\lambda}_1$, respectively, be the maximum eigenvalues of matrices $Z^{\top} Z$ and $\hat{Z}^{\top} \hat{Z}$ (i.e., $\| Z \|_{\mathrm{op}}^2$ and $\| \hat{Z} \|_{\mathrm{op}}^2$, respectively). Then, for all $\epsilon \in \left( 0, \frac{2}{7} \right)$, 
\begin{align}
\label{eq:lmd_hatlmd_diff}
% [Revision 2020/9] <---
\frac{| \lambda_1 - \hat{\lambda}_1 |}{b} = O_p \left( m^{-\frac{1}{21} + \epsilon} \right), 
% [Revision 2020/9] --->
\end{align}
where $b$ is defined as in (\ref{eq:ab}). 
\end{lemma}

\begin{proof}
% [Revision 2020/9] <---
From Lemma \ref{lm:zop2_eq1} and \ref{lm:zop2_eq2}, we have already shown that the following equation holds for all $\epsilon \in \left( 0, \frac{2}{7} \right)$: 
\begin{align}
\label{eq:lmd_tldlmd_diff}
| \lambda_1 - \tilde{\lambda}_1 | = O_p \left( m^{\frac{2}{7} + \epsilon} \right) 
\iff \frac{| \lambda_1 - \tilde{\lambda}_1 |}{b} = O_p \left( m^{-\frac{1}{21} + \epsilon} \right). 
\end{align}

We consider the joint probability of the event $\mathcal{F}_m$ that the clustering algorithm outputs the correct block structure (i.e., $\tilde{Z} = \hat{Z}$) and the event $\mathcal{G}_{m, C}$ that $\frac{| \lambda_1 - \tilde{\lambda}_1 |}{b} \leq C m^{-\frac{1}{21} + \epsilon}$ holds. Such a joint probability satisfies the following inequality: 
% [Revision 2020/9] --->
\begin{eqnarray}
\mathrm{Pr} \left( \mathcal{F}_m \cap \mathcal{G}_{m, C} \right) \geq 1 - \mathrm{Pr} \left( \mathcal{F}^{\mathrm{C}}_m \right) - \mathrm{Pr} \left( \mathcal{G}^{\mathrm{C}}_{m, C} \right), 
\end{eqnarray}
where $\mathcal{A}^{\mathrm{C}}$ is the complement of event $\mathcal{A}$. 
The consistency assumption \ref{asmp:consistency} guarantees that if $(K_0, H_0) = (K, H)$, $\mathrm{Pr} \left( \mathcal{F}^{\mathrm{C}}_m \right)$ converges to $0$ in the limit of $m \to \infty$. 
By combining this fact with (\ref{eq:lmd_tldlmd_diff}), we obtain 
\begin{eqnarray}
% [Revision 2020/9] <---
\forall \tilde{\epsilon}>0, \ \exists C>0, M>0, \ \forall m \geq M, \ 
\mathrm{Pr} \left( \mathcal{F}_m \cap \mathcal{G}_{m, C} \right) \geq 1 - \tilde{\epsilon}, 
% [Revision 2020/9] --->
\end{eqnarray}
which results in (\ref{eq:lmd_hatlmd_diff}). 
\end{proof}

%=======================================================================
\newpage
\section{Proof of $\hat{\sigma}^*  = O_p (KH)$ in unrealizable case}
\label{ap_sigma_star}

\begin{proof}
Throughout the proof, we use the following notations: 
\begin{itemize}
% [Revision 2020/9] <---
\item $A^{(k, h)}$, $P^{(k, h)}$, and $Z^{(k, h)}$, respectively, are the $(k, h)$th \textbf{null} blocks of matrices $A$, $P$, and $Z$. 
\item $\underline{A}^{(k, h)}$, $\underline{P}^{(k, h)}$, and $\underline{\hat{P}}^{(k, h)}$, respectively, are the $(k, h)$th \textbf{estimated} blocks of matrices $A$, $P$, and $\hat{P}$. 
\item We denote the row and column sizes of the $(k, h)$th \textbf{estimated} block as $\underline{n}_k$ and $\underline{p}_h$, respectively. 
% [Revision 2020/9] --->
\item $(k_1, h_1)$ is the set of row and column cluster indices of submatrix $\bar{X}$ in the \textbf{estimated} block structure. 
\end{itemize}

% [Revision 2020/9] <---
As for the order of the estimated standard deviation $\hat{\sigma}^*$, we have $\hat{\sigma}^* = \hat{S}_{k_1 h_1}$. Note that the block size $(\bar{n}_1, \bar{p}_1)$ of submatrix $\bar{X}$ is at least $(n_{\mathrm{min}}/K_0) \times (p_{\mathrm{min}}/H_0)$. Therefore, we have 
\begin{align}
\label{eq:sigma_hat}
\hat{\sigma}^* &= \hat{S}_{k_1 h_1} 
= \frac{1}{\sqrt{\underline{n}_{k_1} \underline{p}_{h_1}}} \| \underline{A}^{(k_1, h_1)} - \underline{\hat{P}}^{(k_1, h_1)} \|_{\mathrm{F}} \nonumber \\
&\leq \frac{1}{\sqrt{\bar{n}_1 \bar{p}_1}} \| \underline{A}^{(k_1, h_1)} - \underline{\hat{P}}^{(k_1, h_1)} \|_{\mathrm{F}} 
% [Revision 2020/9] --->
\leq \sqrt{\frac{K_0 H_0}{n_{\mathrm{min}} p_{\mathrm{min}}}} \| \underline{A}^{(k_1, h_1)} - \underline{\hat{P}}^{(k_1, h_1)} \|_{\mathrm{F}} \nonumber \\
&\leq \sqrt{\frac{K_0 H_0}{n_{\mathrm{min}} p_{\mathrm{min}}}} \| A - \hat{P} \|_{\mathrm{F}} 
= \sqrt{\frac{K_0 H_0}{n_{\mathrm{min}} p_{\mathrm{min}}}} \| A - P + P - \hat{P} \|_{\mathrm{F}} \nonumber \\
&\leq \sqrt{\frac{K_0 H_0}{n_{\mathrm{min}} p_{\mathrm{min}}}} \left( \| A - P \|_{\mathrm{F}} + \| P - \hat{P} \|_{\mathrm{F}} \right) \nonumber \\
% [Revision 2020/9] <---
&= \sqrt{\frac{K_0 H_0}{n_{\mathrm{min}} p_{\mathrm{min}}}} \left( \sqrt{\sum_{k=1}^K \sum_{h=1}^H \| A^{(k, h)} - P^{(k, h)} \|_{\mathrm{F}}^2} + \| P - \hat{P} \|_{\mathrm{F}} \right) \nonumber \\
&= \sqrt{\frac{K_0 H_0}{n_{\mathrm{min}} p_{\mathrm{min}}}} \left( \sqrt{\sum_{k=1}^K \sum_{h=1}^H S_{kh}^2 \| Z^{(k, h)} \|_{\mathrm{F}}^2} + \| P - \hat{P} \|_{\mathrm{F}} \right) \nonumber \\
&\leq \sqrt{\frac{K_0 H_0}{n_{\mathrm{min}} p_{\mathrm{min}}}} \left[ \sqrt{KH} \left( \max_{k = 1, \dots, K, h = 1, \dots, H} S_{kh} \right) \| Z \|_{\mathrm{F}}+ \| P - \hat{P} \|_{\mathrm{F}} \right]. 
% [Revision 2020/9] --->
\end{align}

% [Revision 2020/9] <---
Here, for all $(i, j)$, $\left( Z^{(k, h)}_{ij} \right)^2$ independently follows the same distribution, and $\mathbb{E} \left[ \left( Z^{(k, h)}_{ij} \right)^2 \right] = 1$. 
We also have $\mathbb{V} \left[ \left( Z^{(k, h)}_{ij} \right)^2 \right] = \mathbb{E} \left[ \left( Z^{(k, h)}_{ij} \right)^4 \right] - 1 < \infty$, since we have assumed that $\mathbb{E} \left[ \left( Z^{(k, h)}_{ij} \right)^4 \right] < \infty$ from the sub-exponential assumption. Therefore, from the central limit theorem and Prokhorov's theorem \cite{Vaart1998}, we have $\frac{1}{\sqrt{n_k p_h}} \sum_{i=1}^{n_k} \sum_{j=1}^{p_k} \left[ \left( Z^{(k, h)}_{ij} \right)^2 - 1 \right] = O_p (1)$. 
In other words, the following equation holds: $\sum_{i=1}^{n_k} \sum_{j=1}^{p_k} \left( Z^{(k, h)}_{ij} \right)^2 = n_k p_h + O_p (m) = O_p (m^2)$.
Based on this result, we obtain 
\begin{align}
\label{eq:Z_F}
\| Z \|_{\mathrm{F}} 
&= \sqrt{\sum_{k=1}^K \sum_{h=1}^H \| Z^{(k, h)} \|_{\mathrm{F}}^2} 
= \sqrt{\sum_{k=1}^K \sum_{h=1}^H \sum_{i=1}^{n_k} \sum_{j=1}^{p_k} \left( Z^{(k, h)}_{ij} \right)^2} \nonumber \\
&= \sqrt{\sum_{k=1}^K \sum_{h=1}^H O_p (m^2)} 
= O_p (\sqrt{KH} m). 
\end{align}

Furthermore, we have
\begin{align}
\label{eq:pp_F}
&\| P - \hat{P} \|_{\mathrm{F}} = \sqrt{\sum_{i=1}^n \sum_{j=1}^p \left(P_{ij} - \hat{P}_{ij} \right)^2} 
= \sqrt{\sum_{i=1}^n \sum_{j=1}^p \left(P_{ij} - \bar{P}_{ij} + \bar{P}_{ij} - \hat{P}_{ij} \right)^2} \nonumber \\
&\leq \sqrt{\sum_{i=1}^n \sum_{j=1}^p \left( \left|P_{ij} - \bar{P}_{ij} \right| + \left|\bar{P}_{ij} - \hat{P}_{ij} \right| \right)^2} \nonumber \\
&\leq \sqrt{\sum_{i=1}^n \sum_{j=1}^p \left[ \left( \max_{i' = 1, \dots, n, j' = 1, \dots, p}  \left|P_{i' j'} - \bar{P}_{i' j'} \right| \right) +  \left|\bar{P}_{ij} - \hat{P}_{ij} \right| \right]^2} \nonumber \\
&\leq \sqrt{\sum_{i=1}^n \sum_{j=1}^p \left[ \left(\max_{\substack{k = 1, \dots, K, h = 1, \dots, H,\\k' = 1, \dots, K, h' = 1, \dots, H}} \left| B_{kh} -B_{k' h'} \right| \right) + \left| \bar{P}_{ij} - \hat{P}_{ij} \right| \right]^2} \nonumber \\
&\leq \sqrt{\sum_{i=1}^n \sum_{j=1}^p \left(\max_{\substack{k = 1, \dots, K, h = 1, \dots, H,\\k' = 1, \dots, K, h' = 1, \dots, H}} \left| B_{kh} -B_{k' h'} \right| + \max_{i = 1, \dots, n, j = 1, \dots, p} \left| \bar{P}_{ij} - \hat{P}_{ij} \right| \right)^2} \nonumber \\
&= \sqrt{np} \left(\max_{\substack{k = 1, \dots, K, h = 1, \dots, H,\\k' = 1, \dots, K, h' = 1, \dots, H}} \left| B_{kh} -B_{k' h'} \right| + \max_{i = 1, \dots, n, j = 1, \dots, p} \left| \bar{P}_{ij} - \hat{P}_{ij} \right| \right) \nonumber \\
&\leq \sqrt{np} \left[ \max_{\substack{k = 1, \dots, K, h = 1, \dots, H,\\k' = 1, \dots, K, h' = 1, \dots, H}} \left| B_{kh} -B_{k' h'} \right| + O_p \left( \frac{KH}{\sqrt{m}} \right) \right] \nonumber \\
&= O_p (m +KH\sqrt{m}). 
\end{align}
Here, to derive the last inequality in (\ref{eq:pp_F}), we used the assumption that (\ref{eq:o_block_un}) holds for the block with the maximum difference between $\bar{P}$ and $\hat{P}$. In the final equation, we used the fact that $\max_{\substack{k = 1, \dots, K, h = 1, \dots, H,\\k' = 1, \dots, K, h' = 1, \dots, H}} |B_{kh} -B_{k' h'}|$ is bounded by a finite constant. 

By combining (\ref{eq:sigma_hat}), (\ref{eq:Z_F}), and (\ref{eq:pp_F}), we obtain $\hat{\sigma}^*  = O_p (KH)$. 
% [Revision 2020/9] --->
\end{proof}

%=======================================================================
%%% Uncheck <==
\newpage
\section{Proof of the asymptotic ICL in the Bernoulli case}
\label{ap_icl}

\begin{proof}
From Lemma 4.2 in \cite{Keribin2012}, the resulting asymptotic ICL is given by
\begin{align}
\label{eq:icl}
\mathrm{ICL} (K_0, H_0) &= \max_{\pi, \rho, B} \log p(A, \hat{g}^{(1)}, \hat{g}^{(2)} | \pi, \rho, B) \nonumber \\
&- \frac{K_0 - 1}{2} \log n - \frac{H_0 - 1}{2} \log p - \frac{K_0 H_0}{2} \log (np). 
\end{align}
In regard to the first term in (\ref{eq:icl}), we consider the following optimization problem: 
\begin{align}
\label{eq:max_logp}
&\max_{\pi, \rho, B} \log p(A, \hat{g}^{(1)}, \hat{g}^{(2)} | \pi, \rho, B), \nonumber \\
% [Revision 2020/9] <---
\mathrm{s.t.}\ &\sum_{k=1}^{K_0} \pi_k = 1,\ \pi_k \geq 0\ \mathrm{for\ all}\ k,\ \sum_{h=1}^{H_0} \rho_h = 1,\ \rho_h \geq 0\ \mathrm{for\ all}\ h, \nonumber \\
&0 \leq B_{kh} \leq 1\ \mathrm{for\ all}\ (k, h). 
\end{align}
The above problem is solved with the Lagrangian undetermined multiplier method, which employs
\begin{align}
\label{eq:lagrange1}
&f \equiv \log p(A, \hat{g}^{(1)}, \hat{g}^{(2)} | \pi, \rho, B) - \xi_1 \sum_{k=1}^{K_0} \pi_k - \xi_2 \sum_{h=1}^{H_0} \rho_h, \nonumber \\
&= \sum_{k=1}^{K_0} |I_k| \log \pi_k + \sum_{h=1}^{H_0} |J_h| \log \rho_h - \xi_1 \sum_{k=1}^{K_0} \pi_k - \xi_2 \sum_{h=1}^{H_0} \rho_h \nonumber \\
&+ \sum_{k=1}^{K_0} \sum_{h=1}^{H_0} \sum_{i \in I_k, j \in J_h} \left[ A_{ij} \log B_{kh} + (1 - A_{ij}) \log \left( 1 - B_{kh} \right) \right], \\
\label{eq:lagrange2}
&\frac{\partial f}{\partial \pi_k} = \frac{\partial f}{\partial \rho_h} = \frac{\partial f}{\partial B_{kh}} = 0\ \mathrm{for\ all}\ k,\ h. 
\end{align}
By substituting (\ref{eq:lagrange1}) into (\ref{eq:lagrange2}), we have
\begin{align}
\label{eq:solution_B}
&\frac{|I_k|}{\pi_k} = \xi_1, \ \ \ 
\frac{|J_h|}{\rho_h} = \xi_2, \ \ \ 
\sum_{i \in I_k, j \in J_h} \left[ \frac{A_{ij}}{B_{kh}} - \frac{1 - A_{ij}}{1 - B_{kh}} \right] = 0 \nonumber \\
\iff &\pi_k = \frac{|I_k|}{\xi_1}, \ \ \ 
\rho_h = \frac{|J_h|}{\xi_2}, \ \ \ 
B_{kh} = \frac{\sum_{i \in I_k, j \in J_h} A_{ij}}{|I_k| |J_h|}, 
\end{align}
for all $(k, h)$. In regard to $\{ \pi_k \}$ and $\{ \rho_h \}$, from the conditions in (\ref{eq:max_logp}), $\sum_k |I_k| = \xi_1$ and $\sum_h |J_h| = \xi_2$ hold and thus we finally have
\begin{align}
\label{eq:solution_pr}
\pi_k = \frac{|I_k|}{\sum_{k=1}^{K_0} |I_k|}, \ \ \ 
\rho_h = \frac{|J_h|}{\sum_{h=1}^{H_0} |J_h|}. 
% [Revision 2020/9] --->
\end{align}
We can easily check that the solutions of (\ref{eq:solution_B}) and (\ref{eq:solution_pr}) satisfy all the conditions in (\ref{eq:max_logp}). 

Finally, by substituting the above results into (\ref{eq:icl}), we have
\begin{align}
\mathrm{ICL} (K_0, H_0) &= \sum_{k=1}^{K_0} |I_k| \log \left( \frac{|I_k|}{\sum_{k=1}^{K_0} |I_k|} \right) + \sum_{h=1}^{H_0} |J_h| \log \left( \frac{|J_h|}{\sum_{h=1}^{H_0} |J_h|} \right) \nonumber \\
&+ \sum_{k=1}^{K_0} \sum_{h=1}^{H_0} \left( \sum_{i \in I_k, j \in J_h} A_{ij} \right) \log \left( \frac{\sum_{i \in I_k, j \in J_h} A_{ij}}{|I_k| |J_h|} \right) \nonumber \\
&+ \sum_{k=1}^{K_0} \sum_{h=1}^{H_0} \left( |I_k| |J_h| - \sum_{i \in I_k, j \in J_h} A_{ij} \right) \log \left( 1 - \frac{\sum_{i \in I_k, j \in J_h} A_{ij}}{|I_k| |J_h|} \right) \nonumber \\
&- \frac{K_0 - 1}{2} \log n - \frac{H_0 - 1}{2} \log p - \frac{K_0 H_0}{2} \log (np) \nonumber \\
&= \sum_{k=1}^{K_0} |I_k| \log \left( \frac{|I_k|}{n} \right) + \sum_{h=1}^{H_0} |J_h| \log \left( \frac{|J_h|}{p} \right) \nonumber \\
&+ \sum_{k=1}^{K_0} \sum_{h=1}^{H_0} |I_{k}| |J_{h}| \left[ \hat{B}_{k h} \log \hat{B}_{k h} + \left( 1 - \hat{B}_{k h} \right) \log \left( 1 - \hat{B}_{k h} \right) \right] \nonumber \\
&- \frac{K_0 - 1}{2} \log n - \frac{H_0 - 1}{2} \log p - \frac{K_0 H_0}{2} \log (np). 
\end{align}
Note that we have defined $\hat{B}_{k h}$ as in (\ref{eq:BPS_hat}). 
\end{proof}
%%% Uncheck <==

\end{appendices}

%===========================================================================

\clearpage
\bibliographystyle{abbrv}
\bibliography{paper}

\begin{thebibliography}{10}

\bibitem{Ames2014}
B.~P.~W. Ames.
\newblock Guaranteed clustering and biclustering via semidefinite programming.
\newblock {\em Mathematical Programming}, 147(1):429--465, 2014.

\bibitem{Arabie1978}
P.~Arabie, S.~A. Boorman, and P.~R. Levitt.
\newblock Constructing blockmodels: How and why.
\newblock {\em Journal of Mathematical Psychology}, 17(1):21--63, 1978.

\bibitem{Bai1993}
Z.~D. Bai and Y.~Q. Yin.
\newblock Limit of the smallest eigenvalue of a large dimensional sample
  covariance matrix.
\newblock {\em The Annals of Probability}, 21(3):1275--1294, 1993.

\bibitem{Bao2015}
Z.~Bao, G.~Pan, and W.~Zhou.
\newblock Universality for the largest eigenvalue of sample covariance matrices
  with general population.
\newblock {\em The Annals of Statistics}, 43(1):382--421, 2015.

\bibitem{Bickel2016}
P.~J. Bickel and P.~Sarkar.
\newblock Hypothesis testing for automated community detection in networks.
\newblock {\em Journal of the Royal Statistical Society: Series B (Statistical
  Methodology)}, 78(1):253--273, 2016.

\bibitem{Bloemendal2016}
A.~Bloemendal, A.~Knowles, H.-T. Yau, and J.~Yin.
\newblock On the principal components of sample covariance matrices.
\newblock {\em Probability Theory and Related Fields}, 164:459--552, 2016.

\bibitem{Brault2016}
V.~Brault and A.~Channarond.
\newblock Fast and consistent algorithm for the latent block model.
\newblock arXiv:1610.09005, 2016.

\bibitem{Chen2018}
K.~Chen and J.~Lei.
\newblock Network cross-validation for determining the number of communities in
  network data.
\newblock {\em Journal of the American Statistical Association},
  113(521):241--251, 2018.

\bibitem{Conover1999}
W.~J. Conover.
\newblock {\em Practical Nonparametric Statistics}.
\newblock John Wiley \& Sons, New York, 1999.

\bibitem{Corneli2015}
M.~Corneli, P.~Latouche, and F.~Rossi.
\newblock Exact {ICL} maximization in a non-stationary time extension of the
  latent block model for dynamic networks.
\newblock In {\em Proceedings of the 23-th European Symposium on Artificial
  Neural Networks, Computational Intelligence and Machine Learning}, pages
  225--230, 2015.

\bibitem{Dabbs2016}
B.~Dabbs and B.~Junker.
\newblock Comparison of cross-validation methods for stochastic block models.
\newblock arXiv:1605.03000, 2016.

\bibitem{Dhillon2001}
I.~S. Dhillon.
\newblock Co-clustering documents and words using bipartite spectral graph
  partitioning.
\newblock In {\em Proceedings of the 7th ACM SIGKDD International Conference on
  Knowledge Discovery and Data Mining}, pages 269--274, 2001.

\bibitem{Ding2018}
X.~Ding and F.~Yang.
\newblock A necessary and sufficient condition for edge universality at the
  largest singular values of covariance matrices.
\newblock {\em The Annals of Probability}, 28(3):1679--1738, 2018.

\bibitem{Dua2017}
D.~Dua and C.~Graff.
\newblock {UCI} machine learning repository.
\newblock \url{http://archive.ics.uci.edu/ml}, 2017.
\newblock {U}niversity of California, Irvine, School of Information and
  Computer Sciences.

\bibitem{Flynn2020}
C.~J. Flynn and P.~O. Perry.
\newblock Profile likelihood biclustering.
\newblock {\em Electronic Journal of Statistics}, 14(1):731--768, 2020.

\bibitem{Geman1980}
S.~Geman.
\newblock A limit theorem for the norm of random matrices.
\newblock {\em The Annals of Probability}, 8(2):252--261, 1980.

\bibitem{Govaert2003}
G.~Govaert and M.~Nadif.
\newblock Clustering with block mixture models.
\newblock {\em Pattern Recognition}, 36:463--473, 2003.

\bibitem{Hartigan1972}
J.~A. Hartigan.
\newblock Direct clustering of a data matrix.
\newblock {\em Journal of the American Statistical Association},
  67(337):123--129, 1972.

\bibitem{Holland1983}
P.~W. Holland, K.~B. Laskey, and S.~Leinhardt.
\newblock Stochastic blockmodels: First steps.
\newblock {\em Social Networks}, 5:109--137, 1983.

\bibitem{Hu2019}
J.~Hu, H.~Qin, T.~Yan, and Y.~Zhao.
\newblock Corrected {B}ayesian information criterion for stochastic block
  models.
\newblock {\em Journal of the American Statistical Association}, 0(0):1--13,
  2019.

\bibitem{Hu2020}
J.~Hu, J.~Zhang, H.~Qin, T.~Yan, and J.~Zhu.
\newblock Using maximum entry-wise deviation to test the goodness of fit for
  stochastic block models.
\newblock {\em Journal of the American Statistical Association}, 0(0):1--10,
  2020.

\bibitem{Johansson2000}
K.~Johansson.
\newblock Shape fluctuations and random matrices.
\newblock {\em Communications in Mathematical Physics}, 209:437--476, 2000.

\bibitem{Johnstone2001}
I.~M. Johnstone.
\newblock On the distribution of the largest eigenvalue in principal components
  analysis.
\newblock {\em The Annals of Statistics}, 29(2):295--327, 2001.

\bibitem{Karwa2016}
V.~Karwa, D.~Pati, S.~Petrovi\'c, L.~Solus, N.~Alexeev, M.~Rai\v{c},
  D.~Wilburne, R.~Williams, and B.~Yan.
\newblock Exact tests for stochastic block models.
\newblock arXiv:1612.06040, 2016.

\bibitem{Kawamoto2017}
T.~Kawamoto and Y.~Kabashima.
\newblock Cross-validation estimate of the number of clusters in a network.
\newblock {\em Scientific Reports}, 7(3327), 2017.

\bibitem{Keribin2012}
C.~Keribin, V.~Brault, G.~Celeux, and G.~Govaert.
\newblock Model selection for the binary latent block model.
\newblock In {\em Proceedings of 20th International Conference on Computational
  Statistics}, pages 379--390, 2012.

\bibitem{Keribin2015}
C.~Keribin, V.~Brault, G.~Celeux, and G.~Govaert.
\newblock Estimation and selection for the latent block model on categorical
  data.
\newblock {\em Statistics and Computing}, 25:1201--1216, 2015.

\bibitem{Labiod2011}
L.~Labiod and M.~Nadif.
\newblock Modularity and spectral co-clustering for categorical data.
\newblock In {\em Proceedings of the International Conference on Data Mining},
  pages 386--392, 2011.

\bibitem{Lei2016}
J.~Lei.
\newblock A goodness-of-fit test for stochastic block models.
\newblock {\em The Annals of Statistics}, 44(1):401--424, 2016.

\bibitem{Li2020}
T.~Li, E.~Levina, and J.~Zhu.
\newblock Network cross-validation by edge sampling.
\newblock {\em Biometrika}, 107(2):257--276, 2020.

\bibitem{Lomet2012}
A.~Lomet, G.~Govaert, and Y.~Grandvalet.
\newblock Model selection in block clustering by the integrated classification
  likelihood.
\newblock In {\em Proceedings of 20th International Conference on Computational
  Statistics}, pages 519--530, 2012.

\bibitem{Ma2012}
Z.~Ma.
\newblock Accuracy of the {T}racy-{W}idom limits for the extreme eigenvalues in
  white {W}ishart matrices.
\newblock {\em Bernoulli}, 18(1):322--359, 2012.

\bibitem{Mariadassou2015}
M.~Mariadassou and C.~Matias.
\newblock Convergence of the groups posterior distribution in latent or
  stochastic block models.
\newblock {\em Bernoulli}, 21(1):537--573, 2015.

\bibitem{Saldana2017}
D.~F.~S. na, Y.~Yu, and Y.~Feng.
\newblock How many communities are there?
\newblock {\em Journal of Computational and Graphical Statistics},
  26(1):171--181, 2017.

\bibitem{Nakano2014}
M.~Nakano, K.~Ishiguro, A.~Kimura, T.~Yamada, and N.~Ueda.
\newblock Rectangular tiling process.
\newblock In {\em Proceedings of the 31st International Conference on Machine
  Learning}, pages 361--369, 2014.

\bibitem{Passino2020}
F.~S. Passino and N.~A. Heard.
\newblock {B}ayesian estimation of the latent dimension and communities in
  stochastic blockmodels.
\newblock {\em Statistics and Computing}, 30:1291--1307, 2020.

\bibitem{Peche2009}
S.~P\'ech\'e.
\newblock Universality results for the largest eigenvalues of some sample
  covariance matrix ensembles.
\newblock {\em Probability Theory and Related Fields}, 143:481--516, 2009.

\bibitem{Peixoto2013}
T.~P. Peixoto.
\newblock Parsimonious module inference in large networks.
\newblock {\em Physical Review Letters}, 110:148701, 2013.

\bibitem{Pillai2014}
N.~S. Pillai and J.~Yin.
\newblock Universality of covariance matrices.
\newblock {\em Annals of Applied Probability}, 24(3):935--1001, 2014.

\bibitem{Pontes2015}
B.~Pontes, R.~Gir\'aldez, and J.~S. Aguilar-Ruiz.
\newblock Biclustering on expression data: A review.
\newblock {\em Journal of Biomedical Informatics}, 57:163--180, 2015.

\bibitem{Rastelli2018}
R.~Rastelli and N.~Friel.
\newblock Optimal {B}ayesian estimators for latent variable cluster models.
\newblock {\em Statistics and Computing}, 28:1169--1186, 2018.

\bibitem{Robert2017}
V.~Robert and Y.~Vasseur.
\newblock Comparing high dimensional partitions, with the coclustering
  {A}djusted {R}and {I}ndex.
\newblock arXiv:1705.06760, 2017.

\bibitem{Roy2008}
D.~M. Roy and Y.~W. Teh.
\newblock The {M}ondrian process.
\newblock In {\em Advances in Neural Information Processing Systems 21}, pages
  1377--1384, 2008.

\bibitem{Shan2008}
H.~Shan and A.~Banerjee.
\newblock {B}ayesian co-clustering.
\newblock In {\em Proceedings of the 8th IEEE International Conference on Data
  Mining}, pages 530--539, 2008.

\bibitem{Silverstein1985}
J.~W. Silverstein.
\newblock The smallest eigenvalue of a large dimensional {W}ishart matrix.
\newblock {\em The Annals of Probability}, 13(4):1364--1368, 1985.

\bibitem{Soshnikov2002}
A.~Soshnikov.
\newblock A note on universality of the distribution of the largest eigenvalues
  in certain sample covariance matrices.
\newblock {\em Journal of Statistical Physics}, 108:1033--1056, 2002.

\bibitem{Tracy2009}
C.~A. Tracy and H.~Widom.
\newblock The distributions of random matrix theory and their applications.
\newblock In {\em New Trends in Mathematical Physics}, pages 753--765.
  Springer, 2009.

\bibitem{Vaart1998}
A.~W. van~der Vaart.
\newblock {\em Asymptotic Statistics}.
\newblock Cambridge University Press, 1998.

\bibitem{Ward1963}
J.~H. Ward, Jr.
\newblock Hierarchical grouping to optimize an objective function.
\newblock {\em Journal of the American Statistical Association},
  58(301):236--244, 1963.

\bibitem{Wyse2012}
J.~Wyse and N.~Friel.
\newblock Block clustering with collapsed latent block models.
\newblock {\em Statistics and Computing}, 22:415--428, 2012.

\bibitem{Wyse2017}
J.~Wyse, N.~Friel, and P.~Latouche.
\newblock Inferring structure in bipartite networks using the latent blockmodel
  and exact {ICL}.
\newblock {\em Network Science}, 5(1):45--69, 2017.

\bibitem{Yin1988}
Y.~Q. Yin, Z.~D. Bai, and P.~R. Krishnaiah.
\newblock On the limit of the largest eigenvalue of the large dimensional
  sample covariance matrix.
\newblock {\em Probability Theory and Related Fields}, 78:509--521, 1988.

\bibitem{Yuan2018}
M.~Yuan, Y.~Feng, and Z.~Shang.
\newblock A likelihood-ratio type test for stochastic block models with bounded
  degrees.
\newblock arXiv:1807.04426, 2018.

\end{thebibliography}

\end{document}